\documentclass[twoside]{article}

\usepackage[accepted]{aistats2025_arxiv}
\usepackage[round]{natbib}

\usepackage{float}
\usepackage{amsmath,amsthm,amssymb}
\usepackage{mathtools}
\usepackage{todonotes}
\usepackage{times}
\usepackage{bm} 
\usepackage{enumerate}
\usepackage{thmtools} 
\usepackage{thm-restate}
\usepackage{catchfilebetweentags}
\usepackage{multirow}
\usepackage{float}
\usepackage{booktabs}
\usepackage{subcaption}
\usepackage{pgfplots}
\usepackage{graphicx}
\usepackage[font=small, labelfont=bf]{caption}
\usepackage{booktabs}
\usepackage{tabularx}
\usepackage{array}
\usepackage{dsfont}
\usepackage{amsthm}
\usepackage{amsmath}
\usepackage{amssymb}
\usepackage{mathtools}
\usepackage{todonotes}
\usepackage[ruled, vlined]{algorithm2e}
\usepackage{algpseudocode}
\usepackage{thmtools}
\usepackage{stmaryrd}

\usepackage{hyperref}
\usepackage{xcolor}
\definecolor{Bleu}{RGB}{30,144,245}
\definecolor{Red}{HTML}{FF617B}
\hypersetup{colorlinks,citecolor=Bleu,linkcolor=Red}

\newcommand{\elise}[0]{\text{TnS}}

\theoremstyle{plain}

\newtheorem{theorem}{Theorem}
\newtheorem{lemma}{Lemma}

\newtheorem{corollary}{Corollary}

\SetKw{Pull}{pull}
\SetKw{Sample}{sample}
\SetKw{Update}{update}
\SetKw{Observe}{observe}
\SetKwInOut{Output}{return}
\SetKwInOut{Input}{initialize}

\newcommand{\C}{\subset}
\newcommand{\T}{\intercal}
\DeclareMathOperator{\m}{m}
\DeclareMathOperator{\M}{M}

\newcommand{\bP}{\mathbb{P}}

\newcommand{\bB}{\mathbb{B}}
\newcommand{\bE}{\mathbb{E}}

\newcommand{\bR}{\mathbb{R}}

\newcommand{\bN}{\mathbb{N}}

\newcommand{\cN}{\mathcal{N}}
\newcommand{\cM}{\mathcal{M}}

\newcommand{\cU}{\mathcal{U}}
\newcommand{\cC}{\mathcal{C}}

\newcommand{\cA}{\mathcal{A}}
\newcommand{\cZ}{\mathcal{Z}}

\newcommand{\cL}{\mathcal{L}}
\newcommand{\cF}{\mathcal{F}}
\newcommand{\cD}{\mathcal{D}}
\newcommand{\cO}{\mathcal{O}}
\newcommand{\cP}{\mathcal{P}}
\newcommand{\cV}{\mathcal{V}}
\newcommand{\cI}{\mathcal{I}}
\newcommand{\ind}{\mathds{1}}
\newcommand{\cE}{\mathcal{E}}
\newcommand{\Tr}{\textrm{Tr}}
\newcommand{\veps}{\varepsilon}

\newcommand{\iid}{{ \it i.i.d }}

\newcommand{\impl}{\Longrightarrow}

\DeclareMathOperator*{\argmax}{argmax}
\DeclareMathOperator*{\argmin}{argmin}

\newcommand{\equi}{\Longleftrightarrow}

\DeclareMathOperator{\vol}{{vol}}
\DeclareMathOperator{\alt}{\mathrm{Alt}}

\DeclareMathOperator{\diag}{\mathrm{diag}}

\newcommand{\vmu}{\ensuremath{{\mu}}} 

\newcommand{\vv}[1]{\ensuremath{vec}{#1}}

\newcommand{\w}{\ensuremath{\bm \omega}}
\newcommand{\D}{\ensuremath{{\Delta_K}}}

\newcommand{\lp}{\ensuremath{\left(}}
\newcommand{\rp}{\ensuremath{\right)}}
\newcommand{\lb}{\ensuremath{\left\{}}
\newcommand{\rb}{\ensuremath{\right\}}}
\newcommand{\lsb}{\ensuremath{\left[}}
\newcommand{\rsb}{\ensuremath{\right]}}
\newcommand{\wt}{\ensuremath{\widetilde}}
\newcommand{\wh}{\ensuremath{\widehat}}

\renewcommand{\complement}{\mathsf{c}}

\newcommand{\indi}[1]{\mathds{1}\left(#1\right)}
\newcommand{\transpose}{^\mathsf{\scriptscriptstyle T}}
\newcommand{\simplex}{\triangle_{K}}
\newcommand{\eqdef}{ := }
\begin{document}
%

%

\twocolumn[

\aistatstitle{Pareto Set Identification With Posterior Sampling}

\aistatsauthor{Cyrille Kone \And Marc Jourdan \And Emilie Kaufmann}

\aistatsaddress{Univ. Lille, CNRS, Inria, Centrale Lille, \\ UMR 9189-CRIStAL, F-59000 Lille, France \And EPFL,\\  Lausanne, Switzerland \And Univ. Lille, CNRS, Inria, Centrale Lille, \\ UMR 9189-CRIStAL, F-59000 Lille, France} ]

\begin{abstract}
  The problem of identifying the best answer among a collection of items having real-valued distribution is well-understood. 
  Despite its practical relevance for many applications, fewer works have studied its extension when multiple and potentially conflicting metrics are available to assess an item's quality.
  Pareto set identification (PSI) aims to identify the set of answers whose means are not uniformly worse than another.
  This paper studies PSI in the transductive linear setting with potentially correlated objectives.
  Building on posterior sampling in both the stopping and the sampling rules, we propose the \hyperlink{PSIPS}{PSIPS} algorithm that deals simultaneously with structure and correlation without paying the computational cost of existing oracle-based algorithms.
  Both from a frequentist and Bayesian perspective, \hyperlink{PSIPS}{PSIPS} is asymptotically optimal.
  We demonstrate its good empirical performance in real-world and synthetic instances.
\end{abstract}

\section{INTRODUCTION}
\label{sec:intro}

When facing a pure exploration problem, the decision maker aims to answer a question about a set of unknown distributions (e.g. modeling the treatment's effects) from which she gathers observations (e.g., measure of its outcomes) while providing theoretical guarantees on the candidate's answer.
Practitioners might have multiple, potentially conflicting, metrics to optimize simultaneously~\citep{zuluaga_active_2013}.
When there is no clear trade-off between those metrics, aggregation into a unique objective is unrealistic despite being appealing.
For instance, in clinical trials, both the effectiveness and safety of a treatment should be high~\citep{Jennison1993GroupST}, but specifying the safety cost of an increased efficacy would be unethical. 
Likewise, in chip design, the runtime of the hardware and its energy consumption should be low~\citep{almer}. 
Moreover, there may exist a statistical correlation among the objectives (e.g., toxicity-efficacy in clinical trials) or observable features that characterize each item (e.g., chip area and design parameters in chip design~\citep{zuluaga_active_2013}). 

To further our investigations, we adopt the well-studied stochastic bandit model~\citep{audibert_best_2010,lattimore_bandit_2020, drugan_designing_2013} in the multi-dimensional setting with dimension $d \in \bN$. 
In this decision-making game, the learner interacts sequentially with an environment of $K \in \bN$ \textit{arms}.
Each arm $i \in [K] \eqdef \{1, \dots, K\}$ is associated with an unknown probability distribution over $\bR^{d}$ denoted by $\nu_{i} \in \cD$, where $\cD$ is the set of possible distributions $\cD \subseteq \cP(\bR^{d})$, and having an unknown finite mean $\mu_i  \eqdef (\mu_i(c))_{c \in [d]} \in \bR^{d}$, where $\mu_i(c) \eqdef \bE_{X \sim \nu_{i}}[X](c)$ denotes the mean of the objective $c$ for the arm $i$.
A \textit{bandit instance} is uniquely characterized by its vector of distributions $\bm{\nu} \eqdef (\nu_{i})_{i \in [K]} \in \cD^{K}$ admitting $\bm \mu \eqdef  (\mu_i)_{i \in [K]} \in \cM$ as its matrix of means, where $\cM \subseteq \bR^{K \times d}$ is the set of possible vectors of means.

We consider the set $\cD = \{\cN(\lambda,\Sigma) \mid \lambda \in \bR^d\}$ of Gaussian distributions with known covariance matrix $\Sigma$.
$\Sigma$ models the correlation between objectives (e.g., toxicity-efficacy in dose-finding trials).
To model the dependency between the arms (e.g. treatments having similar active ingredients), we consider the \emph{linear} setting~\citep{soare2014best,degenne2020gamification} in dimension $h \in \bN$, where $\bm \mu$ is fully characterized by the set of arms vectors $\cA \eqdef \{a_i\}_{i \in [K]} \subseteq \bR^{h}$ and the regression matrix $\bm \theta \in \Theta$, where $\Theta \subseteq \bR^{h \times d}$ is the set of possible regression matrices, namely $\cM \eqdef \{ \bm \mu \in \bR^{K \times d} \mid \exists \bm \theta \in \Theta, \:\bm \mu = \bm A \bm \theta \}$ where $\bm A = (a_i)_{i \in [K]} \in \bR^{K \times h}$.
We will use $i \in [K]$ and $a \in \cA$ to denote arms and $\bm \mu$ and $\bm \theta$ to denote means.
The above setting encompasses two well-studied models: \emph{unstructured multi-dimensional} ($h = K$ and $\bm A = I_{K}$) and \emph{linear one-dimensional} ($d = 1$).

\paragraph{Pareto Set Identification}
We focus on Pareto set identification (PSI) for \emph{transductive} linear bandits, in which there exists a finite set of \emph{answers} $\cZ \subseteq \bR^{h}$ which can be different from the set of arms $\cA$~\citep{fiez2019sequential,zhaoqi_peps}.
For example, the treatments of interests $\cZ$ might be too numerous (e.g. $\cA \subseteq \cZ$) or too costly (e.g. $\cA \cap \cZ = \emptyset$) compared to the treatments $\cA$ that could be administered during clinical trials.
Given an answer $z \in \cZ$, its vector of means is denoted by $\mu_{z} \eqdef \bm \theta\transpose z \in \bR^d$.
In that task, the goal of the learner is to identify the set of answers whose means are not uniformly worse coordinate-wise than another, known as the \textit{Pareto set} and denoted as $S^\star(\bm\theta) \subseteq \cZ$ (or $S^\star$ when $\bm\theta$ is clear from the context). 
The Pareto set contains the answers that satisfy optimal trade-offs in the objectives.
The linear setting is the case $\cZ = \cA$.

An answer $z \in \cZ$ is said to be \textit{Pareto dominated} by an answer $x \in \cZ$, which we denote by $\mu_z \prec \mu_x$ or $z \prec x$, if 
\begin{equation*} 
	\forall c\in [d], \: \mu_{z}(c) \le \mu_{x}(c) \: \text{ and } \: \exists c \in [d], \: \mu_{z}(c) < \mu_{x}(c) \: .
\end{equation*}
The Pareto set of a matrix of means $\bm\lambda \in \bR^{h \times d}$ is the set of answers that are not Pareto dominated, i.e.
\begin{equation*} 
	S^\star(\bm \lambda) \eqdef \left\{  z \in \cZ \mid \nexists x \in \cZ \backslash\{z\} , \: \lambda_z \prec \lambda_x \right\}  \: .
\end{equation*}
In the one-dimensional case ($d=1$), we recover best answer identification (BAI) since the Pareto set reduces to the set of best answers, i.e. $S^\star(\bm \lambda) =\argmax_{z \in \cZ} \lambda_z$.

\paragraph{Identification Strategy}
At each time $t$, the learner chooses an arm $a_t \in \cA$ based on the observations previously collected and obtains a sample $X_{t} \in \bR^d$, random variable with conditional distribution $\nu_{a_t}$ given $a_t$.
It then proceeds to the next stage.
An \emph{algorithm} for the learner in this interaction should specify a \emph{sampling rule} that determines $a_t$ based on previously observed samples and some exogenous randomness. 
Formally, $a_t$ is $\cF_t$-measurable with the $\sigma$-algebra $\cF_t \eqdef \sigma \left(U_{1}, a_{1}, X_{1}, \cdots, a_{t-1}, X_{t-1}, U_{t} \right)$, called \emph{history} before time $t$, where $U_{t} \sim \cU([0,1])$ materializes the possible independent randomness used by the algorithm at time $t$.
The empirical allocation over arms is $\bm N_t \eqdef (N_{t,a})_{a \in \cA}$ where $N_{t,a} \eqdef \sum_{s \in [t-1]}\indi{a_s = a}$ and $\frac{\bm N_{t}}{t-1} \in \simplex \eqdef \{w \in \bR_{+}^{K} \mid  \sum_{a \in \cA} w_a = 1 \}$ where $\simplex$ denotes the probability simplex of dimension $K-1$.

As the learner aims at identifying $S^\star(\bm\theta)$, its algorithm should include a \textit{recommendation rule}.
At time $t$, the agent recommends a \textit{candidate answer} $\widehat S_t \subseteq \cZ$ for the Pareto set before pulling arm $a_t$, therefore $\widehat S_t$ is $\cF_t$-measurable.

\paragraph{Fixed-confidence PSI}
There are several ways to evaluate the performance of a PSI algorithm.
The two major theoretical frameworks are the \textit{fixed-confidence} setting~\citep{even-dar_action_nodate,jamieson_best-arm_2014,garivier_optimal_2016}, which will be the focus of this paper, and the \textit{fixed-budget} setting~\citep{audibert_best_2010,gabillon_best_2012}.
In the fixed-confidence setting~\citep{kone2023adaptive,crepon24a}, the agent aims at minimizing the number of samples used to identify $S^\star(\bm\theta)$ with confidence $1 - \delta \in (0,1)$.
In the fixed-budget setting~\citep{kone24a}, the objective is to minimize the probability of misidentifying $S^\star(\bm\theta)$ with a fixed number of samples.

In the fixed-confidence setting, the learner should define a \textit{stopping rule} that is a stopping time for the filtration $(\cF_{t})_{t}$.
The \textit{sample complexity} of an algorithm corresponds to its stopping time $\tau_{\delta}$, which counts the number of rounds before termination. 
An algorithm is said to be \textit{$\delta$-correct} on the problem class $\cD^{K}$ with a set of regression matrices $\Theta$ if its probability of stopping and not recommending a correct answer is upper bounded by $\delta$, i.e. $\bP_{\bm{\nu}}(\tau_{\delta} < + \infty, \: \widehat S_{\tau_{\delta}} \ne S^\star(\bm \theta)) \le \delta$ for all instances $\bm{\nu} \in \cD^{K}$ having regression matrix $\bm \theta \in \Theta$.
A fixed-confidence PSI algorithm is judged based on its expected sample complexity $\bE_{\bm{\nu}}[\tau_{\delta}]$, i.e., the expected number of samples collected before it stops and returns a correct answer with confidence $1-\delta$. 
The learner should design a $\delta$-correct algorithm minimizing $\bE_{\bm{\nu}}[\tau_{\delta}]$.

\paragraph{Notation}
For any $\bm w \in \bR_{+}^{K}$, let $V_{\bm w} \eqdef \sum_{a \in \cA} w_{a} a a\transpose$ be the \textit{design matrix}, which is symmetric and positive semi-definite, and definite if and only if $\text{Span}(\{a \in \cA \mid w_a \ne 0\}) = \bR^h$.
For any symmetric positive semi-definite matrix $V$, we define the semi-norm $\|x\|_V \eqdef \sqrt{x \transpose V x}$.
It is a norm if $V$ is positive definite. $\| | A |\|$ is the operator norm of $A$.
Let $A\otimes B$ denote the Kronecker product of $A, B$ and
let $\vv(A)$ denote the concatenation of the columns of $A$.  

\paragraph{Lower Bound}
A $\delta$-correct algorithm should distinguish $\bm \theta$ from all the instances $\bm \lambda \in \Theta$ having a different Pareto set. 
To better disentangle the structure induced by PSI from the underlying structure $\Theta$, we define the set of alternative instances as $\alt (S^\star) \eqdef \{\bm\lambda \in \bR^{h \times d} \mid S^\star(\bm\lambda) \ne S^\star\}$ (i.e. without $\Theta$).
The requirement to distinguish $\bm \theta$ from $\Theta \cap \alt (S^\star)$ leads to a lower bound on the expected sample complexity on any instance~\citep{garivier_optimal_2016,garivier2019explore,crepon24a}.
\begin{restatable}{lemma}{LowerBound} \label{lem:lower_bound}
	An algorithm which is $\delta$-correct on all problems in $\cD^K$ satisfies that, for all $\bm{\nu} \in \cD^K$ with regression matrix $\bm \theta \in \Theta$, $\bE_{\bm\nu}[\tau_{\delta}] \ge T^\star(\bm \theta) \log(1/(2.4\delta))$ where $T^\star(\bm \theta)$ is a \emph{characteristic time} whose inverse is defined as
	\begin{equation*} 
		2T^\star(\bm \theta)^{-1} \eqdef \sup_{\bm w \in \simplex} \inf_{\bm \lambda \in \Theta \cap \alt (S^\star(\bm \theta))} \|\vv(\bm \theta - \bm \lambda)\|^2_{\Sigma^{-1} \otimes V_{\bm w}} \: ,
	\end{equation*}
	and its maximizer set of \emph{optimal allocations} is $w^\star(\bm \theta)$.
\end{restatable}
We say that an algorithm is \textit{asymptotically optimal} if its sample complexity matches that lower bound, namely if $\limsup_{\delta \to 0}\bE_{\bm{\nu}}[\tau_{\delta}] /\log(1/\delta)\le T^\star(\bm \theta)$.
\subsection{Algorithms}
\label{ssec:algorithms}

We use the $\xi$-regularized least-square estimator for $\bm \theta$, i.e. $\bm{\widehat{\theta}}_t \eqdef (V_{\bm N_t} + \xi I_h)^{-1} \sum_{s \in [t-1]} a_s X_s\transpose$ with $\xi \ge 0$. 
Given the known covariance $\Sigma$ for the objectives, we use the $\xi$-regularized empirical covariance $\bm \Sigma_t \eqdef \Sigma \otimes(V_{\bm N_t} + \xi I_h)^{-1}$.
We recommend the empirical Pareto set $\widehat S_t \eqdef S^\star(\bm{\widehat{\theta}}_t)$.

\paragraph{Stopping Rules}
Before designing a sampling rule matching the asymptotic lower bound, one should specify a good stopping rule.
Elimination-based~\citep{even-dar_action_nodate,karnin_almost_2013} and gap-based~\citep{kalyanakrishnan_pac_2012} stopping rules are sub-optimal in their $\delta$-dependency. 
Hence, they fail to reach asymptotic optimality even for known $w^\star(\bm \theta)$.
The well-studied Chernoff stopping times~\citep{garivier_optimal_2016} use the generalized loglikelihood ratio (GLR) statistic.
The GLR stopping rule is defined as $\tau_{\delta}^{\textrm{GLR}} \eqdef \inf\{t \mid \textrm{GLR}(t) > \beta(t-1, \delta)\}$ where $\beta$ is a threshold. 
The $\xi$-regularized GLR statistic is
\begin{equation} 
\label{eq:GLRT}
	2\text{GLR}(t) \eqdef \inf_{\bm \lambda \in \Theta \cap \alt (\widehat S_t)} \|\vv(\bm{\widehat{\theta}}_t - \bm \lambda)\|^2_{\bm \Sigma_t^{-1}} \: .
\end{equation}
Provided $\beta$ is well chosen, the GLR stopping rule ensures $\delta$-correctness regardless of the sampling rule.
Moreover, combining a well-chosen sampling rule with the GLR stopping rule yields an asymptotically optimal algorithm.

Due to its $\delta$-correctness and simplicity, the GLR stopping rule has become the standard choice.
However, in some specific pure exploration problems, the computational cost of the GLR statistic becomes a major bottleneck.
PSI for transductive linear bandits is an example of a setting where evaluating $\text{GLR}(t)$ is numerically challenging.
In transductive linear BAI ($d=1$), the computational cost of~\eqref{eq:GLRT} heavily depends on $\Theta$, and there are no procedure for general $\Theta$.
While there is a closed form solution in $\cO(h^2 |\cZ|)$ when $\Theta = \bR^{h}$, a more involved one-dimensional optimization procedure should be solved $|\cZ|$ times when $\Theta$ is a ball~\citep{degenne2020gamification}.
In the unstructured multi-dimensional setting, to the best of our knowledge, computing~\eqref{eq:GLRT} is only tractable for the setting with independent objectives, i.e. $\Sigma$ diagonal, $\cZ = \cA$, $h = K$, $\bm A = I_{K}$ and $\Theta = \bR^{K \times d}$.
~\citet{crepon24a} proposed the best-known procedure in that restrictive setting, which requires to solve $\cO(Kd^3\lvert \widehat S_t \rvert^d)$ convex problems at each round. 
Their computationally costly procedure is difficult to extend to correlated objectives or structured settings.

We introduce the posterior sampling (PS) stopping rule.
Based on at most $M(t-1,\delta)$ independent realizations from a posterior distribution inflated by $c(t-1,\delta)$, it stops when
\begin{equation} \label{eq:PS_stopping_rule}
	\tau_{\delta}^{\textrm{PS}} \eqdef \inf\{t \mid \forall m \in [M(t-1,\delta)], \: \bm \theta^{m}_{t} \notin \Theta \cap \alt(\widehat S_t) \} \: ,
\end{equation}
with $\vv(\bm \theta^{m}_{t} - \bm{\widehat{\theta}}_t) \mid \cF_t \sim \cN(0_{hd}, c(t-1,\delta) \bm \Sigma_t )$ for all $m \in  [M(t-1,\delta)]$.
For precise choice of $M(t,\delta)$ and $c(t,\delta)$ to ensure $\delta$-correctness, we refer the reader to Section~\ref{subsec:stopping_rule}.
The PS stopping rule allows tackling structured settings with correlated objectives and has a lower computational cost than the GLR stopping rule.

\paragraph{Sampling Rules}
Given a stopping rule, the sampling rule should make it stop as soon as possible.
While gap-based sampling rules will be sub-optimal in their asymptotic allocation, the asymptotic optimality of Track-and-Stop~\citep{garivier_optimal_2016} will come at the cost of a computational intractability since it computes $w_t \in w^\star(\bm{\widehat{\theta}}_t)$ at time $t$.
The game-based approach came from seeing the lower bound as the solution of a two-player game~\citep{degenne2019non}.
Using two no-regret learning algorithms against each other yields a saddle-point algorithm that sequentially approximates $T^\star(\bm \theta)^{-1}$.
The popularized instance uses the best response as min learner, i.e. 
\begin{equation} \label{eq:BR_oracle}
	\bm\lambda^{\textrm{BR}} :  \omega \mapsto \argmin_{\bm \lambda \in \Theta \cap \alt (\widehat S_t)} \|\vv(\bm{\widehat{\theta}}_t - \bm \lambda)\|^2_{\Sigma^{-1} \otimes (V_{\omega}+\xi I_h)} \: ,
\end{equation}
whose computational cost is the same as the GLR statistic~\eqref{eq:GLRT}.
As for the PS stopping rule, we rely on an inflated posterior sampling (PS) learner for the min learner.
Therefore, our algorithm can tackle structured settings with correlated objectives and has low computational cost.

\subsection{Contributions}
As main contribution, we propose the \hyperlink{PSIPS}{PSIPS} algorithm, which is the first computationally efficient algorithm for transductive linear PSI.
\hyperlink{PSIPS}{PSIPS} builds on posterior sampling in both the PS stopping rule and the min learner of the game-based sampling rule.
By removing the Oracle calls to~\eqref{eq:GLRT} or~\eqref{eq:BR_oracle}, \hyperlink{PSIPS}{PSIPS} deals with the transductive linear PSI structure in a computationally efficient way. 

For the PS stopping rule and regardless of the sampling rule, we exhibit choices of $M(t,\delta)$ and $c(t,\delta)$ ensuring $\delta$-correctness in the unstructured setting with independent or correlated objectives and in transductive linear BAI ($d=1$ and $\Theta = \bR^{h}$).
When $\lim_{\delta \to 0} \frac{c(t,\delta) \log M(t,\delta)}{\log(1/\delta)} \le 1$ (satisfied by our choices), \hyperlink{PSIPS}{PSIPS} is shown to be asymptotically optimal for transductive linear PSI (Theorem~\ref{thm:asymptotic_upper_bound_sample}). 
From a Bayesian perspective, the posterior probability that \hyperlink{PSIPS}{PSIPS} misidentifies the Pareto set decays exponentially fast (as a function of $t$) with the optimal rate $T^\star(\bm \theta)$ (Theorem~\ref{thm:asymptotic_rate_convergence}).  
Our experiments on both real-world and synthetic instances showcase the superior performance of our algorithm both in terms of sample complexity and computational cost.

\subsection{Related Work}

Pioneering the PSI problem for bandits,~\citet{auer_pareto_2016} proposed a $\delta$-correct algorithm based on uniform sampling with an accept/reject mechanism. 
They proved a high-probability lower bound on the sample complexity of any $\delta$-correct algorithm for PSI scaling as $H(\bm \mu) \log(1/\delta)$ with $H(\bm \mu) \eqdef \sum_{i \in [K]} \Delta_i^{-2}$ where $\Delta_i$'s are PSI's ``sub-optimality'' gaps. 
For $\Sigma=\sigma^2I_d$,~\citet{crepon24a} proposed a $\delta$-correct gradient-based algorithm.
Their sampling rule requires computing a (super)-gradient of $\bm N_t \mapsto \textrm{GLR}(t)$, whose per-round computational cost is $\cO(Kd^3\lvert \widehat S_t\rvert^d)$. 
The GLR stopping rule is obtained by solving $\cO(Kd^3\lvert \widehat S_t \rvert^d)$ convex problems per round. 
It cannot be extended for the non-isotropic case. 
For fixed-budget PSI,~\citet{kone24a} proposed an algorithm whose error probability is upper bounded by $\cO( \exp( - T/(H(\bm \mu) \log K)))$.
Up to $\log K$, this is tight. 

Gaussian processes were used to study PSI. Each arm has a feature whose mean is estimated by Gaussian process regression. 
~\citet{zuluaga_active_2013, zuluaga_e-pal_2016} introduced a racing algorithm with correctness guarantees. 
The linear setting is addressed by~\citet{kim2024learningparetousingbootstrapped}.
Evolutionary algorithms and other heuristics are also used for PSI, see~\citet{knowles_parego_2006, nsga}.
Other multi-objective problems exist.
\citet{ararat_vector_2021} studied the identification of the non-dominated set of any cone-induced partial order of which PSI is a special case. 
The authors proposed an extension of the algorithm of~\citet{auer_pareto_2016}. 
In the fixed-budget setting,~\citet{katz-samuels_feasible_2018} aims at identifying the arms belonging to a given polyhedron. 
For multi-objective regret minimization, we refer to~\citet{drugan_designing_2013, xu_pareto_regret}. 

In the pure exploration literature, there is a surge of interest in algorithms based on posterior sampling.
This literature aims at extending the success of Thompson Sampling (TS) for regret minimization~\citep{thompson1933likelihood,kaufmann2012thompson,agrawal2013thompson}.
~\citet{russo_simple_2016} proposed the Top Two TS (TTTS) sampling rule for BAI and showed posterior contraction at the optimal rate.
Having only Bayesian guarantees, TTTS is not paired with a stopping rule.
~\citet{shang_fixed-confidence_2020} considered a variant of the Top Two algorithm for BAI.
They propose a Bayesian stopping rule that requires computing the posterior probability for an arm's optimality, which is computationally inefficient.
In combinatorial pure exploration (a subcase of our setting), the TS-Explore algorithm~\citep{wang2022thompson} is the first to leverage posterior samples in the sampling and the stopping rules.
While their stopping rule is similar to the PS stopping rule defined in~\eqref{eq:PS_stopping_rule}, the authors use a gap-based sampling rule that is asymptotically sub-optimal.
In transductive linear BAI (subcase of our setting), the PEPS algorithm~\citep{zhaoqi_peps} is a game-based algorithm using posterior sampling for the min learner. 
PEPS relies on phases to tune the learning rates that depend on an upper bound on the stochastic losses. 
PEPS achieves the optimal convergence rate of the posterior, yet it lacks a stopping rule and fixed-confidence or fixed-budget guarantees.

\section{PSIPS ALGORITHM}
\label{sec:algo}

We propose the \hyperlink{PSIPS}{PSIPS} (PSI with Posterior Sampling) algorithm for PSI in the transductive linear setting with correlated objectives.
\hyperlink{PSIPS}{PSIPS} combines the PS stopping rule and the game-based sampling rule using a PS min learner.

\begin{algorithm}[t]
	\SetAlgoLined
    \label{algo:LeBAI}
		{\bfseries Input:} regularization $\xi \ge 0$, budget $M$ and inflation $c$ functions, exploration allocation $\omega_{\exp}$ and rate $\alpha > 0$, inflation rate $\eta_{\lambda}$, AdaHedge learner denoted by $\cL^{\cA}$\;
		Get $(\bm Z_{1},V_{1})$ and $\bm{\widehat{\theta}}_1 = V_{1}^{-1} \bm Z_{1} $ and $\Sigma_1 = \Sigma \otimes V_{1}^{-1}$\;
        \For{$t \ge 1$}{
        Get $\widehat S_t = S^\star(\bm{\widehat{\theta}}_t)$\;
        Set $m = 0$, $m_{t} = + \infty$ and $m_{t,\delta} = + \infty$\;
        \While{$\max\{m_{t},m_{t,\delta}\} = + \infty$}{
        	Set $m \mapsfrom m +1$ and get $\bm v_{t}^{m} \sim \cN(\bm 0_{hd}, \bm \Sigma_t)$\;
        	\tcp{Stopping Rule}
        	\If{$m_{t,\delta} = + \infty$}{
        		$\vv(\bm \theta^{m}_{t}) = \vv(\bm{\widehat{\theta}}_t) + \sqrt{c(t-1,\delta)} \bm v_{t}^{m}$\; 
	        	{\bfseries if} $\bm \theta^{m}_{t} \in \Theta \cap \alt(\widehat S_t)$ {\bfseries then} $m_{t,\delta} \mapsfrom m$\; 
	        	{\bfseries else if} $m \ge M(t-1,\delta)$ {\bfseries then  break} and {\bfseries return} $\widehat S_t$\; 
        	}
        	\tcp{Min Learner}
        	\If{$m_{t} = + \infty$}{
        		 $\vv(\bm \lambda^{m}_{t}) = \vv(\bm{\widehat{\theta}}_t) + \eta_{\lambda}^{-1/2} \bm v_{t}^{m}$\;
	        	{\bfseries if} $\bm \lambda^{m}_{t} \in \Theta_t \cap \alt(\widehat S_t)$ {\bfseries then} $m_{t} \mapsfrom m$\; 
        	}
        }
        Get $\omega_t$ from learner $\cL^{\cA}$ and set $\bm \lambda_{t}  = \bm \lambda^{m_t}_{t}$\;
        Set $\wt \w_t = (1-\gamma_t) \omega_t + \gamma_t \omega_{\exp}$ with $\gamma_t = t^{-\alpha}$\;
        Get arm $a_t \sim \wt \w_t$ and collect $X_t \sim \nu_{a_t}$\;
        Feed gain $g_{t}(\omega) = \|\vv(\bm{\widehat{\theta}}_t - \bm \lambda_t)\|^2_{\Sigma^{-1} \otimes V_{\omega}}$ to $\cL^{\cA}$\;
        Update $\bm{\widehat{\theta}}_{t+1} = V_{t+1}^{-1} \bm Z_{t+1}$ and $\Sigma_{t+1} = \Sigma \otimes V_{t+1}^{-1}$ with $\bm Z_{t+1} = \bm Z_{t} + a_{t} X_{t}\transpose$ and $V_{t+1} = V_{t} + a_t a_t\transpose$ \;
        }
	\caption{\protect\hypertarget{PSIPS}{PSIPS}}\label{alg:alg1}
\end{algorithm}

\subsection{The Posterior Sampling (PS) Stopping Rule}
\label{subsec:stopping_rule}

We introduce the posterior sampling (PS) stopping rule.
To specify the PS stopping rule, a budget function $M : \bN \times (0,1) \to \bR_{+}$ and an inflation function $c : \bN \times (0,1) \to \bR_{+}$ are given.
The PS stopping rule defined in~\eqref{eq:PS_stopping_rule} stops when $M(t-1,\delta)$ independent realizations from a posterior distribution inflated by $c(t-1,\delta)$ agrees with the evidence suggesting that $S^\star(\bm \theta) = \widehat S_t$.
Conditioned on $\cF_t$, let $(\bm v_{t}^{m})_{m \in [M(t,\delta)]}$ be i.i.d. draws from the centered posterior distribution $\Pi_{t} \eqdef \cN(0_{hd}, \bm \Sigma_t)$.
For all $m \in [M(t-1,\delta)]$, let $\vv(\bm \theta^{m}_{t}) \eqdef \vv(\bm{\widehat{\theta}}_t) + \sqrt{c(t-1,\delta)} \bm v_{t}^{m}$. Then,
\[	
	\tau_{\delta}^{\textrm{PS}} \eqdef \inf\{t \mid \forall m \in [M(t-1,\delta)], \: \bm \theta^{m}_{t} \notin \Theta \cap \alt(\widehat S_t) \} \: .
\]
When $\tau_{\delta}^{\textrm{PS}}< + \infty$, we recommend $\widehat S_{\tau_{\delta}^{\textrm{PS}}} \eqdef S^\star(\bm{\widehat{\theta}}_{\tau_{\delta}^{\textrm{PS}}})$.
When $t<\tau_{\delta}^{\textrm{PS}}$, let $m_{t,\delta} \eqdef \inf\{ m \mid \: \bm \theta^{m}_{t} \in \Theta \cap \alt(\widehat S_t) \}$.
Intuitively, $\bm \theta^{m_{t,\delta}}_{t}$ is a randomized approximation of the best-response Oracle $\bm\lambda^{\textrm{BR}}(\bm N_t)$ as in~\eqref{eq:BR_oracle}.

\paragraph{Computational Cost}
To reduce the computational cost of \hyperlink{PSIPS}{PSIPS}, we draw $(\bm v_{t}^{m})_{m}$ sequentially as the PS stopping condition is often infringed before $M(t,\delta)$ realizations are drawn (see Figure~\ref{fig:mavgrejection}).
We re-use the realizations $(\bm v_{t}^{m})_{m}$ in the PS min learner of our sampling rule with a different inflation parameter (see Section~\ref{subsec:sampling_rule}).

Compared to the Oracle call to~\eqref{eq:GLRT}, testing that $\bm \theta^{m}_{t} \notin \Theta \cap \alt(\widehat S_t)$ has a lower computational cost.
It scales as the sum of the costs of membership to $\Theta$ and to $\alt(\widehat S_{t})^\complement$, which is at most $\cO( d h |\widehat S_t|\max\{|\widehat S_t| , |\cZ|- |\widehat S_t|\})$ since
\begin{align} \label{eq:complementary_alt}
	&\alt(\widehat S_{t})^\complement = \bigcap_{(z,x) \in\widehat S_{t}^2, x \ne z} \bigcup_{c \in [d]} \{\bm \lambda \mid  \langle E_{z,x}(c), \vv(\bm \lambda) \rangle \ge 0 \} \nonumber \\
	&\quad \cap  \bigcap_{z \notin \widehat S_t}  \bigcup_{x \in \widehat S_{t}} \bigcap_{c \in [d]}\{\bm \lambda \mid \langle E_{x,z}(c), \vv(\bm \lambda) \rangle \ge 0 \} \: ,
\end{align} 
where $E_{z,x}(c) = e_c \otimes (z-x) $ and $e_c = (\indi{c'=c})_{c' \in [d]}$.

To update the candidate answer, we check whether $\bm{\widehat{\theta}}_t \in \alt(\widehat S_{t-1})^\complement$, in which case $\widehat S_t = \widehat S_{t-1}$.
Otherwise, we compute $\widehat S_t = S^\star(\bm{\widehat{\theta}}_t)$.
While a naive implementation is at most in $\cO(d h |\cZ|^2)$,~\cite{kung_ps_algo} proposed an algorithm having a cost scaling as $\cO(h |\cZ| (\log |\cZ|)^{\max\{1,d-2\}})$.
\paragraph{Correctness}
Lemma~\ref{lem:delta_correctness} exhibits choices of budget $M(t,\delta)$ and inflation $c(t,\delta)$ ensuring $\delta$-correctness in the unstructured setting with independent or correlated objectives and in transductive linear BAI ($d=1$ and $\Theta = \bR^{h}$). Letting $X\sim \cN(0, 1)$, we denote by $R(x) := \frac{\bP(X> x )}{f_X(x)}$, the Mills ratio~\citep{john_mills} of $X$, with $f_X$, the density of $X$. 
 
\begin{lemma} \label{lem:delta_correctness}
	Let $\delta \in (0,1)$, $s > 1$, $\zeta$ be the Riemann $\zeta$ function. 
	Define $r(\delta,n) \eqdef \left( \frac{1}{\sqrt{2\pi}} R\left(\sqrt{ \frac{2}{n}\log(1/\delta)}\right)\right)^n$. 
	Let $\beta(t, \delta)$ be an anytime upper bound on $\frac 12\left\| \theta - \theta_{t} \right\|_{\bm \Sigma_t^{-1}}^2$ with probability at least $1-\delta$, as in Lemma~\ref{lem:concentration_threshold}. 
	Regardless of the sampling rule, the PS stopping rule with $c(t,\delta) = \beta(t,\frac{\delta}{2})/\log \frac{1}{\delta}$ and $M(t,\delta) = \left\lceil \frac{\log(2 t^s \zeta(s) /\delta)}{\delta q(t,\delta)} \right\rceil $ ensures that the algorithm is $\delta$-correct for any $\delta \in (0,1)$,\\
	\textbf{1)} when $\Theta = \bR^{K \times d}$, $\cZ = \cA$, $h = K$, $\bm A = I_{K}$, by taking $q(t,\delta) = \min \{r(\delta, d), r(\delta, d + |\widehat S_{t+1}|)\}$ for $\Sigma$ diagonal, and otherwise $q(t,\delta) = \det(\Sigma \bar \sigma)^{-1/2} \min \{ r (\delta^{\frac{d_{\tiny \Sigma}}{d}}, d), \\ r (\delta^{\frac{d_{\tiny \Sigma} + |\widehat S_{t+1}|}{d + |\widehat S_{t+1}|} },d + |\widehat S_{t+1}| )  \} $ with $\bar \sigma = \| | \Sigma^{-1} |\|$ and $d_{\tiny\Sigma} = \| 1_d\|_{(\bar\sigma\Sigma)^{-1}}^2$.\\
	\textbf{2)} when $d=1$ and $\Theta = \bR^{h}$, by taking $q(t,\delta) = r(\delta,1)$.\\
The above choices satisfy $\limsup_{\delta \to 0} \frac{c(t, \delta) \log M(t, \delta)}{\log(1/\delta)}  \le 1$.
\end{lemma}
Lemma~\ref{lem:delta_correctness} does not propose choices of $(M,c)$ for transductive linear PSI with general $\Theta$.
In Section~\ref{sssec:proof_sketch_delta_correctness}, our proof sketch highlights the challenge arising due to the structure.

\subsection{Game-based Sampling Rule} 
\label{subsec:sampling_rule} 

We introduce a game-based sampling rule~\citep{degenne_pure_2019} using a PS min learner.
By combining a max learner playing $\omega_t$ and a min learner playing $\bm \lambda_t$, it yields a saddle-point algorithm which approximates $T^\star(\bm \theta)^{-1}$.

\paragraph{Initialization}
In the unstructured setting, we pull each arm once and observe $X_{i} \sim \nu_{i}$ for all $i \in [K]$.
We take $\xi = 0$ and set $(\bm Z_{1} , V_{1}) = (\bm X_{1} , I_{K})$ where $\bm X_{1} = (X_{i}\transpose )_{i \in [K]}$.
In the structured setting, there is no initial pull of arms.
We use $\xi > 0$ and set $(\bm Z_{1},V_{1}) = ( \bm 0_{h \times d}, \xi I_{h})$.

\paragraph{Max Learner}
As in~\citet{degenne_pure_2019}, we opt for AdaHedge~\citep{rooij_ada_hedge} as the max learner.
We add forced exploration by mixing the played $\omega_t$ with an exploration allocation $\omega_{\mathrm{exp}} \in \simplex$.
Formally, we pull $a_t \sim \widetilde \omega_t$ with $\widetilde \omega_t \eqdef (1-\gamma_t) \omega_t + \gamma_t \omega_{\mathrm{exp}} $ where $\gamma_t = 1/t^{\alpha}$ with $\alpha \in (0,1)$.
The allocation $\omega_{\mathrm{exp}}$ should be chosen such $\lambda_{\min}(V_{\omega_{\mathrm{exp}}}) > 0$.
For example, we can use a uniform allocation on a set of arms spanning $\bR^h$ or the G-optimal design~\citep{zhaoqi_peps}.
Forced exploration ensures that $\widehat S_t$ converges to $S^\star(\bm \theta)$ despite the initial fluctuation of the min learning space $\Theta \cap \alt(\widehat S_t)$.
 
\paragraph{Min Learner}
As in~\citet{zhaoqi_peps}, we propose a PS min learner. 
The PS min learner draws independent realization from a posterior distribution inflated by $\eta_{t} > 0$ until one disagrees with $\widehat S_t$.
Let $(\bm v_{t}^{m})_{m}$ be i.i.d. draws from $\Pi_{t}$.
For all $m \ge 1$, let $\vv(\bm \lambda^{m}_{t}) \eqdef \vv(\bm{\widehat{\theta}}_t) + \eta_{t}^{-1/2} \bm v_{t}^{m}$. Then,
\begin{equation} \label{eq:PS_min_learner}
	\bm \lambda_{t}  \eqdef \bm \lambda^{m_t}_{t} \text{ with } m_t = \inf\{m \mid \: \bm \lambda^{m}_{t} \in \Theta_t \cap \alt(\widehat S_t) \} \: ,
\end{equation}
where $\Theta_t \eqdef \Theta$ when $\Theta$ is bounded, and $\Theta_t$ is a bounded confidence region otherwise. 
Equivalently, we could draw $\vv(\bm \lambda_{t})$ from the distribution $\cN(\vv(\bm{\widehat{\theta}}_t) ,  \eta_{t}^{-1/2} \bm \Sigma_t)$ truncated to $\Theta_t \cap \alt(\widehat S_t)$. 
The inflation $\eta_{t}$ is chosen as an upper bound on the magnitude of the stochastic loss. 
We refer the reader to Appendix~\ref{app:saddle_point_convergence} for more details on how $\eta_{t}$ and $\Theta_t$ are defined sequentially. 

\paragraph{Computational Cost}
At time $t < \tau_{\delta}^{\textrm{PS}}$, we re-use sequentially the realizations $(\bm v_{t}^{m})_{m \in [m_{t,\delta}]}$ drawn by the PS stopping rule with a different inflation parameter.
When $m_{t,\delta} < m_t$, we draw fresh realization sequentially.
Despite the different inflations, $m_t$ is expected to be close to $m_{t,\delta}$, i.e. $\cO(M(t,\delta))$  (see Figure~\ref{fig:mavgrejection}).
The computational cost of~\eqref{eq:PS_min_learner} is lower than a best-response Oracle to~\eqref{eq:GLRT}.

\section{THEORETICAL GUARANTEES}
\label{sec:analysis}

We show that under a generic condition on $(c, M)$, which is satisfied in Lemma~\ref{lem:delta_correctness}, the expected sample complexity of \hyperlink{PSIPS}{PSIPS} is asymptotically optimal (Theorem~\ref{thm:asymptotic_upper_bound_sample}). 
From a Bayesian perspective, the rate of decay for its posterior probability of misidentifying the Pareto set is shown to be asymptotically tight (Theorem~\ref{thm:asymptotic_rate_convergence}).
Theorems~\ref{thm:asymptotic_upper_bound_sample} and~\ref{thm:asymptotic_rate_convergence} hold for unstructured PSI ($\Theta = \bR^{K \times d}$, $\cZ = \cA$, $h = K$, $\bm A = I_{K}$) and transductive linear PSI (bounded convex $\Theta$).

\paragraph{Expected Sample Complexity}
Theorem~\ref{thm:asymptotic_upper_bound_sample} shows that \hyperlink{PSIPS}{PSIPS} is asymptotically optimal when using suitable choices of budget and inflation.
In comparison, known upper bounds on $\bE_{\bm{\nu}}[\tau_\delta]$ were either sub-optimal~\citep{auer_pareto_2016}, or restricted to the unstructured setting with independent objectives~\citep{crepon24a} or transductive linear BAI ($d=1$). 
The proof is sketched in Section~\ref{sssec:proof_sketch_sample_complexity}. 
\begin{restatable}{theorem}{upperBoundSc} \label{thm:asymptotic_upper_bound_sample}
	Using budget $M$ and inflation $c$ such that $\limsup_{\delta \to 0} \frac{c(t, \delta) \log M(t, \delta)}{\log(1/\delta)}  \le 1$, the \hyperlink{PSIPS}{PSIPS} algorithm satisfies that $\limsup_{\delta \rightarrow 0} \bE_{\bm{\nu}}[\tau_\delta^{\textrm{PS}}] /\log(1/\delta) \leq T^\star(\bm \theta)$ for all $\bm{\nu} \in \cD^K$ with regression matrix $\bm \theta \in \Theta$, both for unstructured PSI ($\Theta = \bR^{K \times d}$, $\cZ = \cA$, $h = K$, $\bm A = I_{K}$) and transductive linear PSI (bounded convex $\Theta$).
\end{restatable} 

Lemma~\ref{lem:delta_correctness} gives choices of $(c,M)$ satisfying the above condition and ensuring $\delta$-correctness in unstructured PSI.
Theorem~\ref{thm:asymptotic_upper_bound_sample} also holds on the set $\widetilde \cD$ of multi-variate $\Sigma$-sub-Gaussian, where $\kappa \in \widetilde \cD$ with mean $\lambda$ implies that $\bE_{X \sim \kappa}[e^{u\transpose (X - \lambda)}] \le e^{\frac 12 u\transpose \Sigma^{-1} u}$ for all $u \in \bR^d$~\citep{degenne2016combinatorial}. 
This includes distributions with Bernoulli marginals or with bounded support. 
Optimality is only achieved for multi-variate Gaussian distributions.

\paragraph{Posterior Probability of Misidentification}
Theorem~\ref{thm:asymptotic_rate_convergence} shows that the posterior probability that \hyperlink{PSIPS}{PSIPS} misidentifies the Pareto set decays exponentially fast (as a function of $t$) with a rate $T^\star(\bm \theta)$, which is shown to be asymptotically optimal by a lower bound.
In comparison, known similar Bayesian guarantees were restricted to BAI ($d=1$) in the unstructured and structured settings (e.g. TTTS in~\citet{russo_simple_2016} and PEPS in~\citet{zhaoqi_peps}).  
\begin{restatable}{theorem}{postConvergence} \label{thm:asymptotic_rate_convergence}
	Let $\wt \Pi_t \eqdef \cN( \bm{\widehat{\theta}}_t, \Sigma \otimes V_{\bm N_t}^{-1})$ be the posterior distribution under a flat Gaussian prior (without inflation).
	For all $\bm{\nu} \in \cD^K$ with regression matrix $\bm \theta \in \Theta$, it almost surely holds that $\limsup_{t \to + \infty} - t^{-1} \log \bP_{\wt \Pi_t \mid \cF_t} ( \Theta \cap \alt(S^\star)  ) \le T^\star(\bm \theta)^{-1}$ for any algorithm, and \hyperlink{PSIPS}{PSIPS} almost surely satisfies that $\liminf_{t \to + \infty} - t^{-1} \log \bP_{\wt \Pi_t \mid \cF_t} ( \Theta \cap \alt(S^\star) ) \ge T^\star(\bm \theta)^{-1} $, both for unstructured PSI ($\Theta = \bR^{K \times d}$, $\cZ = \cA$, $h = K$, $\bm A = I_{K}$) and transductive linear PSI (bounded convex $\Theta$).
\end{restatable}

\subsection{Proof Sketches}
\label{ssec:proof_sketch}

Before sketching the proofs of Lemma~\ref{lem:delta_correctness} and Theorem~\ref{thm:asymptotic_upper_bound_sample}, we introduce the concentration event $\cE_{\delta} \eqdef \bigcap_{t \in \bN} \cE_{t,\delta}$ with
\begin{equation} \label{eq:concentration_event}
	\cE_{t,\delta} \eqdef \{\left\| \vv(\theta - \theta_{t} )\right\|_{\bm \Sigma_t^{-1}}^2  \le 2 \beta(t-1,\delta)\} \: .
\end{equation}
Lemma~\ref{lem:concentration_threshold} in Appendix~\ref{app:correctness} gives choices of $\beta(t,\delta)$ such that $\bP_{\bm{\nu}}(\cE_{\delta}^{\complement}) \le \delta$. 
They satisfy that $\lim_{\delta \to 0}\beta(\cdot, \delta) / \log(1/\delta) = 1$ and $\beta(t,\cdot) =_{+ \infty} \cO( \log\log t)$.
Using $\beta(t, \delta)$ in the GLR stopping rule yields $\delta$-correctness for any sampling rule. 

\subsubsection{Proof Sketch of Lemma~\ref{lem:delta_correctness}}
\label{sssec:proof_sketch_delta_correctness}

Under $\cE_{\delta/2}$, the budget $M$ and the inflation $c$ are chosen to ensure $\delta$-correctness of the PS stopping rule, regardless of the sampling rule.
Formally, we need to prove that
\[
	\bP_{\bm \nu}(\cE_{\delta/2} \cap \{ \tau_{\delta}^{\mathrm{PS}} < + \infty, \widehat S_{\tau_{\delta}^{\mathrm{PS}}} \neq S^\star\}) \le \delta/2 \: .
\]
Let $\widehat \Pi_t \eqdef \cN( \bm{\widehat{\theta}}_t, c(t-1,\delta)\bm \Sigma_{t})$.
By union bound, the stopping in~\eqref{eq:PS_stopping_rule} and $1-x \leq \exp(-x)$, it suffices to upper bound
\[
	 \ind_{\cE_{t,\delta/2} \cap \{\widehat S_t \neq S^\star\} } \exp\left(- M(t-1,\delta)\bP_{\widehat \Pi_t \mid \cF_t}(\Theta \cap \alt(\widehat S_t))\right) 
\]
by $\frac{\delta}{2 \zeta(s) (t-1)^s}$ since their sum is smaller than $\delta/2$.
Therefore, we should lower bound $\bP_{\widehat \Pi_t \mid \cF_t}(\Theta \cap \alt(\widehat S_t))$ under the event $\cE_{t,\delta/2} \cap \{\widehat S_t \neq S^\star\} $ to conclude the proof.

\paragraph{Unstructured Setting}
Having $\widehat S_t \ne S^\star$ implies that $\bm \theta \in \alt(\widehat S_t)$, see~\eqref{eq:complementary_alt}.
Either (1) there exists $(i,j) \in \widehat  S_t$ with $i \ne j$ such that $\mu_{i}(c) < \mu_{j}(c) $ for all $c \in [d]$, and we define $\mathcal J_t = \{j\}$.
Or (2) there exists $i \notin \widehat S_t$ and $c \in [d]^{|\widehat S_t|}$ such that $\mu_{j}(c_j) < \mu_{i}(c_j)$ for all $j \in \widehat  S_t$, and $\mathcal J_t = \widehat S_t$. 
Since $\Theta = \bR^{K \times d}$, we can show that $\ind_{\widehat S_t \neq S^\star}\bP_{\widehat \Pi_t \mid \cF_t}(\alt(\widehat S_t)) \ge$
\[
	\prod_{k\in \{i\} \cup \mathcal J_t } \bP_{X \sim \cN(0_d, \Sigma)} \left( X > \sqrt{\frac{N_{t,k}}{c(t-1,\delta)}} (\mu_{k} - \hat \mu_{t,k}) \right) \: ,
\]
where $X > x$ denotes $X(c) > x(c)$ for all $c \in [d]$.

When $\Sigma$ is diagonal, we have $\bP_{X \sim \cN(0_d, \Sigma)}(X > x) = \prod_{c \in [d]}\bP_{X \sim \cN(0, \Sigma_{c,c})}(X > x(c))$.
When $\Sigma$ is not diagonal, we derive a lower bound on $\bP_{X \sim \cN(0_d, \Sigma)}(X > x)$ for any vector $x\in \bR^d$ (Lemma~\ref{lem:mills_cov}).
To further lower bound those quantities, we introduce the ratio of the tail distribution to the density function for Gaussian, known as the Mills ratio~\citep{john_mills}. 
Under the event $\cE_{t,\delta/2}$, we have $N_{t,k}\|\mu_{k} - \hat \mu_{t,k}\|_{\Sigma^{-1}} \le 2\beta(t-1,\delta/2)$.
Using that the Mills ratio is non-increasing and log-convex, we obtain that
\[
	\ind_{\cE_{t,\delta/2} \cap \{\widehat S_t \neq S^\star\} }  \bP_{\widehat \Pi_t \mid \cF_t}( \alt(\widehat S_t)) \ge \delta q(t-1,\delta)
\]
by taking $c(t,\delta) = \beta(t,\delta/2)/\log(1/\delta)$.
This concludes the proof by considering $M(t,\delta) = \left\lceil \frac{\log(2 t^s \zeta(s) /\delta)}{\delta q(t,\delta)} \right\rceil $.

\paragraph{Structured Setting}
Both $\Theta$ and the arms $\cA \subseteq \bR^{h}$ will introduce correlations between arms.
Therefore, one cannot lower bound $\bP_{\widehat \Pi_t \mid \cF_t}(\Theta \cap \alt(\widehat S_t))$ by considering arms separately.
Consequently, choosing $M$ and $c$ to ensure $\delta$-correctness for PSI in the transductive linear setting is a challenging open problem.

When $d=1$ (BAI) and $\Theta = \bR^h$, we can show that $\ind_{\widehat S_t \neq S^\star}\bP_{\widehat \Pi_t \mid \cF_t}(\alt(\widehat S_t)) \ge \bP_{X \sim \cN(0,1)}(X > \|\vv(\bm{\widehat{\theta}}_t - \bm \theta)\|_{\bm \Sigma_t^{-1}} / \sqrt{c(t-1,\delta)})$.
Then, under $\cE_{t,\delta/2}$, we conclude similarly using the Mills ratio properties.

\paragraph{Technical Challenge}
The $\delta$-correctness of the GLR stopping rule is obtained quite simply by concentration.
We have $\text{GLR}(t) \le \frac 12 \|\vv(\bm{\widehat{\theta}}_t - \bm \theta)\|^2_{\bm \Sigma_t^{-1}} \le \beta(t-1,\delta)$ under $\cE_{t,\delta} \cap \{\widehat S_t \ne S^\star\}$, hence $\bP_{\bm \nu}(\cE_{\delta} \cap \{ \tau_{\delta}^{\mathrm{GLR}} < + \infty, \widehat S_{\tau_{\delta}^{\mathrm{GLR}}} \neq S^\star\}) = 0$.
In contrast, under a similar event, proving the $\delta$-correctness of the PS stopping rule requires to control the randomness of $\widehat \Pi_t \mid \cF_t$ by deriving anti-concentration results on $\bP_{\widehat \Pi_t \mid \cF_t}(\Theta \cap \alt(\widehat S_t))$.

\subsubsection{Proof Sketch of Theorem~\ref{thm:asymptotic_upper_bound_sample}}
\label{sssec:proof_sketch_sample_complexity}
Studying the expected sample complexity of an algorithm using the PS stopping rule requires to control the randomness of $\widehat \Pi_t \mid \cF_t$ by deriving concentration results on $\bP_{\widehat \Pi_t \mid \cF_t}(\Theta \cap \alt(\widehat S_t))$ since a direct union bound yields that
\[
	\bP_{\bm \nu}(\tau_\delta^{\mathrm{PS}} > t) \le M(t-1,\delta) \bE_{\bm \nu}[\bP_{\widehat \Pi_t \mid \cF_t}(\Theta \cap \alt(\widehat S_t))] \: .
\]
Since $\alt(\widehat S_t)$ is a union of convex sets (Lemma~\ref{lem:alt_convexity}), we use a Gaussian concentration result on convex sets (Lemma~\ref{lem:mvn-upb}) for $\widehat \Pi_t \mid \cF_t$.
This relates $\bP_{\widehat \Pi_t \mid \cF_t}(\Theta \cap \alt(\widehat S_t))$ to the GLR statistic defined in~\eqref{eq:GLRT}, namely we show that
\[
	\bP_{\widehat \Pi_t \mid \cF_t}(\Theta \cap \alt(\widehat S_t)) \le \alpha_{0} \exp(- \text{GLR}(t)/c(t-1, \delta) ) \: ,
\]
where $\alpha_{0} \le |\cZ|(|\cZ| + d^{|\cZ|})/2$ (Lemma~\ref{lem:prob_to_glr}). 
Once the event $\{\tau_\delta^{\mathrm{PS}} > t\}$ is linked to the GLR statistic, it suffices to lower bound $\text{GLR}(t)$ to conclude the proof.
The rest of the proof would be conducted similarly when using the GLR stopping rule.

Studying the saddle-point convergence of our algorithm yields that, there exists $(\cE_{t})_{t\geq 1}$ and $T_0 \in \bN$ such that for all $t > T_0$, when $\cE_t$ holds, $\text{GLR}(t)\geq (t-1)/ T^\star(\bm \theta) - f(t)$ where $f(t) =_{+ \infty} o(t)$ and $\sum_{t \in \bN}\bP(\cE_t^c) = C_0 < + \infty$. 
Then, we have $\bE_{\bm \nu}[\tau_\delta^{\mathrm{PS}}] \le T_0 + C_0 + $
\begin{align*}
	&\sum_{t > T_0} \alpha_0 \exp \left(- \frac{t - 1 - T^\star(\bm \theta)f(t)}{T^\star(\bm \theta) c(t-1, \delta)} + \log  M(t-1,\delta) \right) \: .
\end{align*}
Using that $\limsup_{\delta \to 0} \frac{c(t, \delta) \log M(t, \delta)}{\log(1/\delta)}  \le 1$, a direct inversion of the above sum concludes the proof.

\section{EXPERIMENTS}
\label{sec:expe}
\newcommand{\algoref}[0]{\hyperlink{PSIPS}{PSIPS}}

We evaluate the performance of \hyperlink{PSIPS}{PSIPS} on real-world inspired and synthetic instances. 
As benchmarks, we consider the following PSI algorithms: APE~\citep{kone2023adaptive}, the gradient-based algorithm of~\citet{crepon24a} denoted as \elise{}, the \emph{oracle} algorithm which samples arm following optimal weights, i.e $a_t \sim \w^\star(\bm \theta)$, and round-robin \emph{uniform} sampling (RR). 
 
 APE relies on confidence bounds to pull an arm at each round and to tailor its stopping rule.
 While APE can be extended to the structured setting, it does not exploit correlation between objectives as its confidence intervals only consider marginals variance. 
 \elise{} computes a supergradient of the GLR~\eqref{eq:GLRT} and uses the GLR stopping rule. 
 Similarly,~\elise{} has no efficient implementation for correlated objectives or the structured setting. 
 Since~\elise{} is computationally expensive, we re-use the empirical results from~\citet{crepon24a}.  
 
 In our experiments, we use the heuristics $c(t, \delta) = 1 + \frac{\log\log t}{\log(1/\delta)}$ and $M(t,\delta) =\frac{1}{\delta}\log \frac{t}{\delta}$. For a fair comparison, APE uses $\beta(t, \delta) = \log(1/\delta) + \log\log t$ in its bounds, hence improving its performance. We report the averaged results over $500$ independent runs with boxplots or shaded area for standard deviation. The observed empirical error is order of magnitude lower than $\delta$. 
 Appendix~\ref{app:ssec_supplementary} contains supplementary experiments (e.g. higher $(K,d)$).
 
 \paragraph{Cov-Boost Trial} 
As prior work in PSI, we use the dataset from~\citet{munro_safety_2021} to simulate a bandit instance for PSI. The dataset consists of the average response of patients cohorts in a covid$19$ trial including $20$ vaccine strategies. Among the indicators recorded to measure the efficacy of each strategy, following \cite{crepon24a}, we keep three of them: titres of neutralizing antibodies and immunoglobuline G), and the wild-type cellular response (cf Table~\ref{tab:arith_mean}). The ideal vaccine strategy would maximize all the three indicators, but the average response for the considered metrics reveals a Pareto set of two vaccine strategies in the trial.  
 
\begin{figure}[t]
\centering
\begin{minipage}{0.46\linewidth}
\centering
    \includegraphics[width=\linewidth]{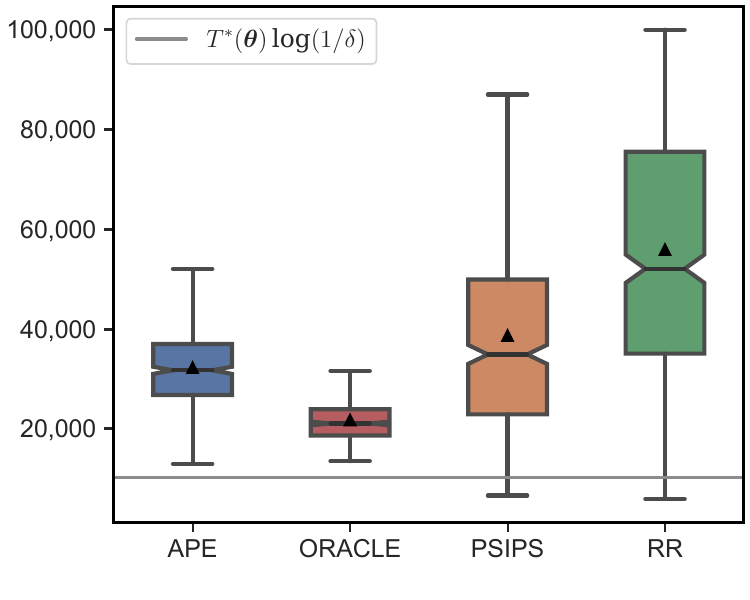}
\end{minipage}
\hspace{0.1cm}
\begin{minipage}{0.46\linewidth}
\centering
    \includegraphics[width=\linewidth]{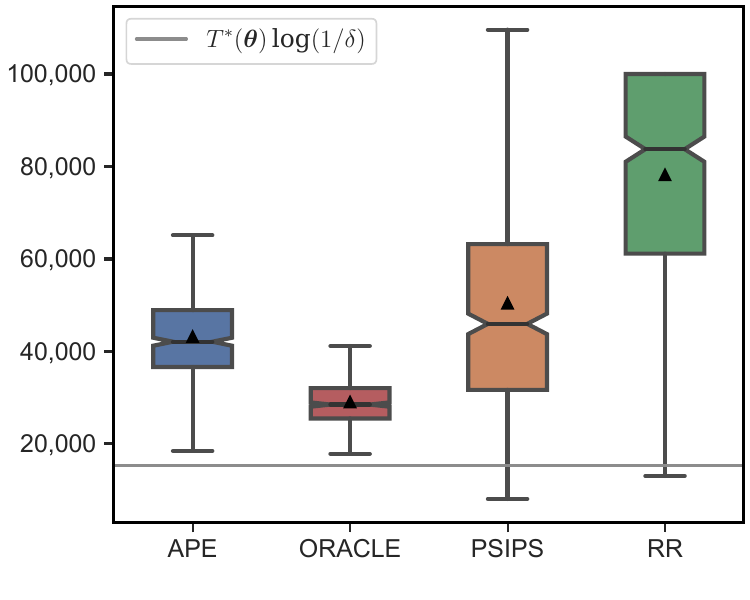}
\end{minipage}
\caption{Empirical stopping times on the covid19 experiment with $\delta=0.01$ (left) and $\delta=0.001$ (right).}
 \label{fig:covboost}
\end{figure} 

Figure~\ref{fig:covboost} shows that~\hyperlink{PSIPS}{PSIPS} has a high variability of stopping time, which was expected due to the use of posterior sampling. While~\hyperlink{PSIPS}{PSIPS} outperforms uniform sampling, it performed on average slightly worse than APE and Oracle on the Cov-Boost instance. For $\delta=0.1$, we reported an average sample complexity of $20456$ for \algoref{} compared to $17909$ reported for \elise{} by \cite{crepon24a}. However, due to the high computational cost of their algorithm, we could not run it together with \algoref{} in this experiment. 

 \paragraph{Correlated Objectives} 
 To assess the performance of \hyperlink{PSIPS}{PSIPS} when the objectives are correlated ($\Sigma$ not diagonal), we choose a Gaussian instance with five arms in dimension two (see Appendix~\ref{app:add_experiments} for the values). The covariance matrix $\Sigma_\rho$ is chosen with unit variance and an off-diagonal correlation coefficient $\rho \in (-1, 1)$: $\rho$ close to $1$ (resp. $-1$) means the two objectives are strongly positively (resp. negatively) correlated and $\rho=0$ means independent objectives. 

\begin{figure}[t]
\centering
\begin{minipage}{.46\linewidth}
  \centering
  \includegraphics[width=\linewidth]{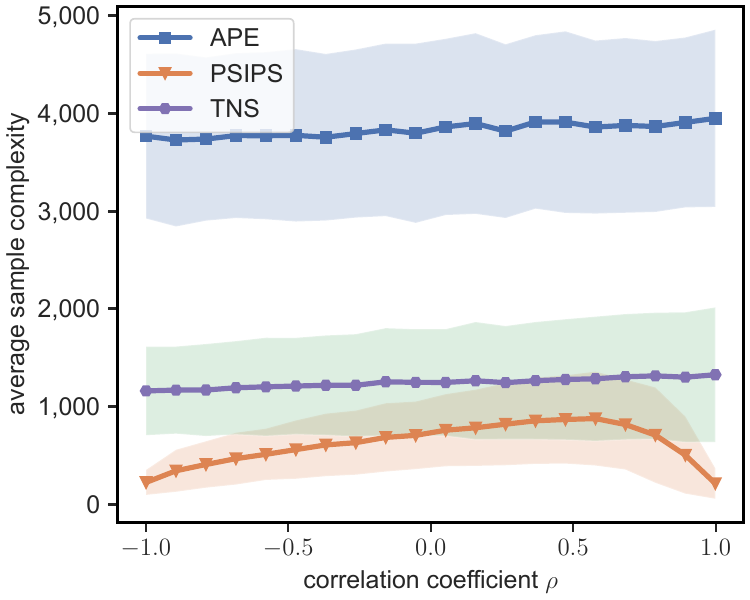}
  \captionof{figure}{Impact of the correlation $\rho$ on the stopping time.}
\label{fig:mcorrelation}
\end{minipage}
\hspace{.1cm}
\begin{minipage}{.46\linewidth}
  \centering
  \includegraphics[width=\linewidth]{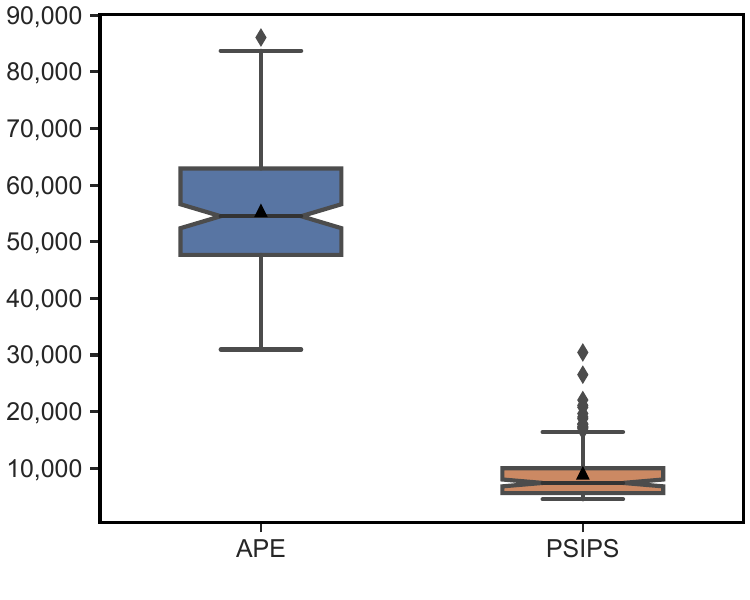}
  \captionof{figure}{Empirical stopping time on the NoC experiment.}
  \label{fig:mnoc}
\end{minipage}
\end{figure} 

Figure~\ref{fig:mcorrelation} shows the sample complexity of \algoref{} versus $\rho$, including \elise{} as heuristic. It reveals that~\algoref{} has a decreasing sample complexity for stronger (negative) correlation between objectives; reducing by up to factor 3 w.r.t the uncorrelated case $(\rho=0)$. Due to its asymptotic optimality,~\algoref{} inherits the properties of $T^\star(\bm \theta)$, which is most likely decreasing with negative $\rho$ on this specific instance. Both \elise{} and APE performance are independent of $\rho$.

\begin{figure}[b]
  \centering
  \includegraphics[width=0.45\linewidth]{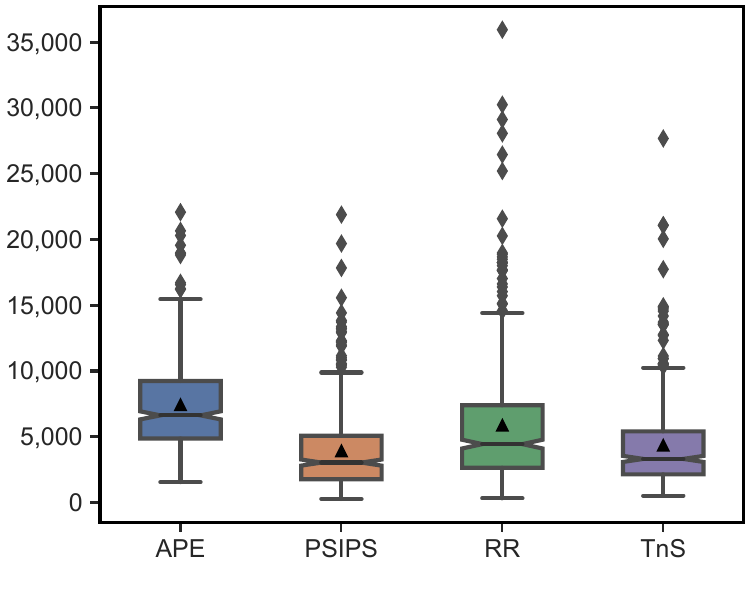}
  \includegraphics[width=0.45\linewidth]{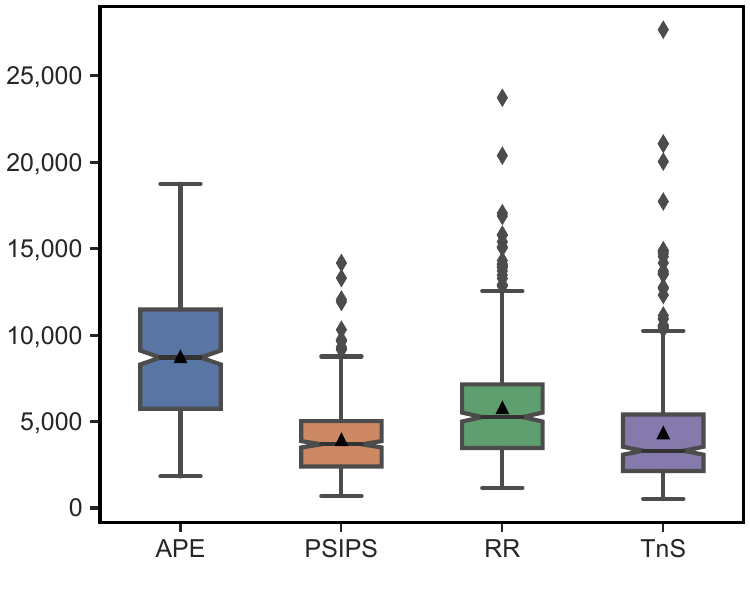}
  \captionof{figure}{Empirical stopping time on random Gaussian (left) and Bernoulli (right) instances.}
  \label{fig:mbayes}
\end{figure} 

 \paragraph{Robustness on Random Instances} 
 To evaluate the robustness of \hyperlink{PSIPS}{PSIPS}, we measure its performance on 250 randomly picked Bernoulli (marginals) and Gaussian instances with $K=5, d=2$ and $\Sigma = I_2/2$.  

 Figure~\ref{fig:mbayes} showcases the competitive performance of \hyperlink{PSIPS}{PSIPS}, which remains good when used on sub-Gaussian instances. 

\paragraph{Structured Instance} 
Network on Chip (NoC, \cite{almer, zuluaga_active_2013}) is a bi-objective optimization dataset designed for hardware development. The objective is to minimize two key performance criteria: energy consumption and runtime of NoC implementations. The dataset includes 259 designs, each characterized by four descriptive features, and its Pareto set comprises $4$ arms (see Appendix~\ref{app:add_experiments}).  

Figure~\ref{fig:mnoc} shows the superior performance of \algoref{} on the structured instance as it can exploit the linear structure to reduce the sample complexity of PSI. We did not include RR as it was more than five times worse than \algoref{}. 

\begin{figure}[t]
\centering
\begin{minipage}{.46\linewidth}
  \centering
  \includegraphics[width=\linewidth]{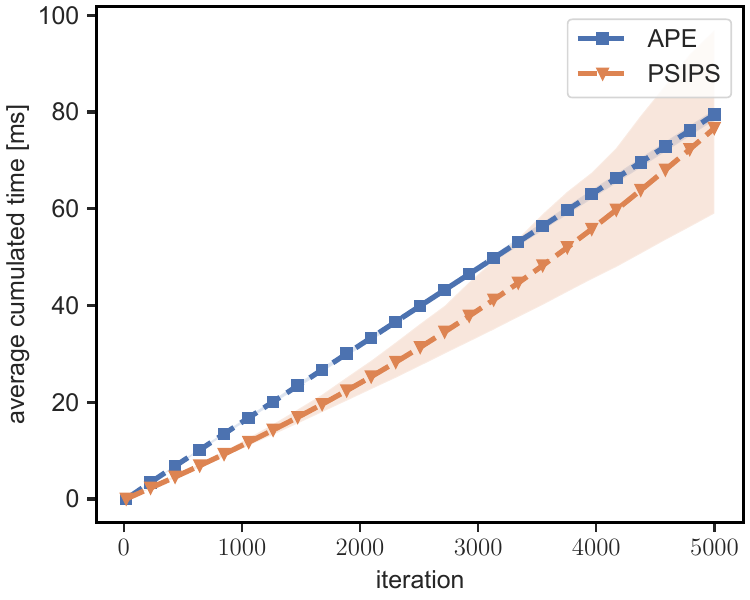}
  \captionof{figure}{Average runtime for the first $T$ iterations in the covid19 experiment.}
  \label{fig:mavgtime}
\end{minipage}%
\hspace{.1cm}
\begin{minipage}{.46\linewidth}
  \centering
  \includegraphics[width=\linewidth]{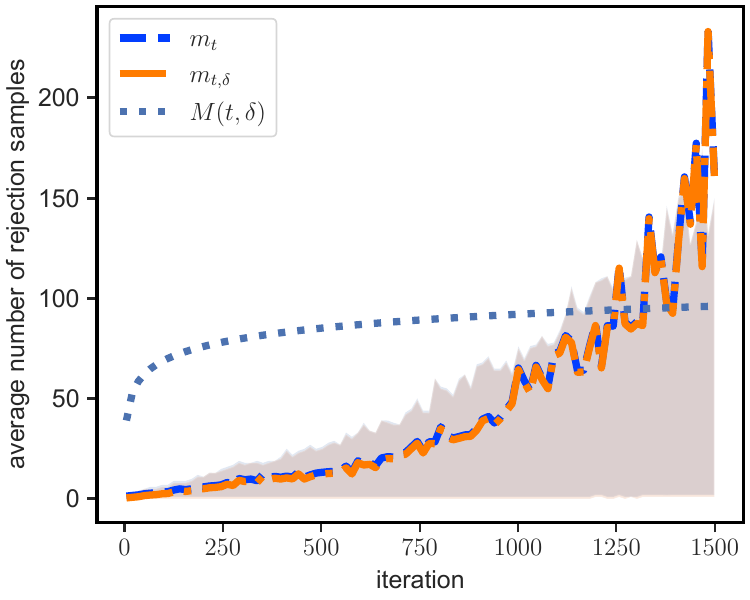}
  \captionof{figure}{Average number of per-round rejection samples $m_t, m_{t, \delta}$.}
\label{fig:mavgrejection}
\end{minipage}
\end{figure} 

\paragraph{Computational Cost}
We evaluate the number of rejection samples used by \hyperlink{PSIPS}{PSIPS} on a Gaussian instance with $\Sigma= I_2/2$ defined as $\mu_1 = (1, 1)^\T$ and $\mu_i = R_{\pi/5}\mu_{i-1}$ for $i\in \{2,3,4,5\}$ where $R_{\pi/5}$ is the $\pi/5$ rotation matrix. 
Without actually stopping, we record both $m_t$ and $m_{t, \delta}$ at each round, averaged over $1000$ runs. 
We also report the cumulated runtime of \hyperlink{PSIPS}{PSIPS} and APE on the covid$19$ experiment. 

Figure~\ref{fig:mavgtime} shows that the computational cost of \hyperlink{PSIPS}{PSIPS} is comparable to APE, which is order of magnitudes smaller than \elise{} (see Figure 4 in~\citet{crepon24a}). 
Figure~\ref{fig:mavgrejection} reveals that, while finding an alternative is initially faster, more rejections are needed before finding an alternative when the posterior concentrates. 
It is precisely at this time that the PS stopping rule triggers. 
\section{PERSPECTIVES}
\label{sec:conclusion}

We proposed the \hyperlink{PSIPS}{PSIPS} algorithm for Pareto set identification in the transductive linear setting with correlated objectives.
By leveraging posterior sampling in both the stopping and the sampling rules, \hyperlink{PSIPS}{PSIPS} is a computationally efficient algorithm that deals with structure and correlation, without using costly oracle calls as existing algorithms.
We show that \hyperlink{PSIPS}{PSIPS} is asymptotically optimal both from a frequentist and Bayesian perspective.
Moreover, \hyperlink{PSIPS}{PSIPS} has competitive empirical performance.

While our choice of $c$ and $M$ ensures $\delta$-correctness of the PS stopping rule in the unstructured setting with independent or correlated objectives, we lack $\delta$-correctness guarantees in the structured setting.
Even for transductive linear BAI ($d=1$), we obtain $\delta$-correctness only when $\Theta = \bR^h$.
Therefore, it fails to account for the boundedness assumption, which is a key assumption to derive the upper bound on the expected sample complexity in the linear setting.
To address this open problem, we need tight non-asymptotic lower bounds on the posterior probability of drawing a realization in $\Theta \cap \alt(\widehat S_t)$.

Despite being a theoretically convenient distributional assumption, using Gaussian distributions with known covariance matrix is too restrictive for many applications.
Therefore, generalizing our results to other classes of distributions is another interesting direction for future work.
For example, when the covariance matrices are unknown and different for each arm, the estimation of the Pareto set is intertwined with estimating the correlation structure.

\subsubsection*{Acknowledgements}

 Cyrille Kone is funded by an Inria/Inserm PhD grant.
 Marc Jourdan is partially supported by the THIA ANR program ``AI\_PhD@Lille''. 
 Emilie Kaufmann acknowledges the support of the French National Research Agency under the projects BOLD (ANR-19-CE23-0026-04) and FATE (ANR22-CE23-0016-01).
\bibliography{main.bib}

\begin{thebibliography}{47}
\providecommand{\natexlab}[1]{#1}
\providecommand{\url}[1]{\texttt{#1}}
\expandafter\ifx\csname urlstyle\endcsname\relax
  \providecommand{\doi}[1]{doi: #1}\else
  \providecommand{\doi}{doi: \begingroup \urlstyle{rm}\Url}\fi

\bibitem[Agrawal and Goyal(2013)]{agrawal2013thompson}
S.~Agrawal and N.~Goyal.
\newblock Thompson sampling for contextual bandits with linear payoffs.
\newblock In \emph{International conference on machine learning}. PMLR, 2013.

\bibitem[Almer et~al.(2011)Almer, Topham, and Franke]{almer}
O.~Almer, N.~Topham, and B.~Franke.
\newblock A learning-based approach to the automated design of mpsoc networks.
\newblock In \emph{Proceedings of the 24th International Conference on
  Architecture of Computing Systems}. Springer-Verlag, 2011.

\bibitem[Ararat and Tekin(2023)]{ararat_vector_2021}
C.~Ararat and C.~Tekin.
\newblock Vector optimization with stochastic bandit feedback.
\newblock In \emph{Proceedings of The 26th International Conference on
  Artificial Intelligence and Statistics}. PMLR, 2023.

\bibitem[Audibert et~al.(2010)Audibert, Bubeck, and Munos]{audibert_best_2010}
J.-Y. Audibert, S.~Bubeck, and R.~Munos.
\newblock Best arm identification in multi-armed bandits.
\newblock In \emph{23rd Conference on Learning Theory}, 2010.

\bibitem[Auer et~al.(2016)Auer, Chiang, Ortner, and Drugan]{auer_pareto_2016}
P.~Auer, C.-K. Chiang, R.~Ortner, and M.-M. Drugan.
\newblock Pareto front identification from stochastic bandit feedback.
\newblock In \emph{Proceedings of the 19th International Conference on
  Artificial Intelligence and Statistics}. {PMLR}, 2016.

\bibitem[Baricz(2008)]{prop_mills}
A.~Baricz.
\newblock Mills' ratio: Monotonicity patterns and functional inequalities.
\newblock \emph{Journal of Mathematical Analysis and Applications}, 2008.

\bibitem[Birnbaum(1942)]{mills_inequality}
Z.~W. Birnbaum.
\newblock {An Inequality for Mill's Ratio}.
\newblock \emph{The Annals of Mathematical Statistics}, 1942.

\bibitem[Crepon et~al.(2024)Crepon, Garivier, and M~Koolen]{crepon24a}
E.~Crepon, A.~Garivier, and W.~M~Koolen.
\newblock Sequential learning of the {P}areto front for multi-objective
  bandits.
\newblock In \emph{Proceedings of The 27th International Conference on
  Artificial Intelligence and Statistics}, Proceedings of Machine Learning
  Research. PMLR, 2024.

\bibitem[De~Rooij et~al.(2014)De~Rooij, Van~Erven, Gr\"{u}nwald, and
  Koolen]{rooij_ada_hedge}
S.~De~Rooij, T.~Van~Erven, P.~D. Gr\"{u}nwald, and W.~M. Koolen.
\newblock Follow the leader if you can, hedge if you must.
\newblock \emph{J. Mach. Learn. Res.}, 2014.

\bibitem[Deb et~al.(2002)Deb, Pratap, Agarwal, and Meyarivan]{nsga}
K.~Deb, A.~Pratap, S.~Agarwal, and T.~Meyarivan.
\newblock A fast and elitist multiobjective genetic algorithm: Nsga-ii.
\newblock \emph{IEEE Transactions on Evolutionary Computation}, 2002.

\bibitem[Degenne(2019)]{degen2019}
R.~Degenne.
\newblock \emph{Impact of structure on the design and analysis of bandit
  algorithms}.
\newblock PhD thesis, 2019.

\bibitem[Degenne and Koolen(2019)]{degenne_pure_2019}
R.~Degenne and W.~Koolen.
\newblock Pure exploration with multiple correct answers.
\newblock In \emph{Advances in Neural Information Processing Systems}. Curran
  Associates, Inc., 2019.

\bibitem[Degenne and Perchet(2016)]{degenne2016combinatorial}
R.~Degenne and V.~Perchet.
\newblock Combinatorial semi-bandit with known covariance.
\newblock \emph{Advances in Neural Information Processing Systems}, 2016.

\bibitem[Degenne et~al.(2019)Degenne, Koolen, and M{\'e}nard]{degenne2019non}
R.~Degenne, W.~M. Koolen, and P.~M{\'e}nard.
\newblock Non-asymptotic pure exploration by solving games.
\newblock \emph{Advances in Neural Information Processing Systems}, 2019.

\bibitem[Degenne et~al.(2020)Degenne, M{\'e}nard, Shang, and
  Valko]{degenne2020gamification}
R.~Degenne, P.~M{\'e}nard, X.~Shang, and M.~Valko.
\newblock Gamification of pure exploration for linear bandits.
\newblock In \emph{International Conference on Machine Learning}. PMLR, 2020.

\bibitem[Drugan and Nowe(2013)]{drugan_designing_2013}
M.-M. Drugan and A.~Nowe.
\newblock Designing multi-objective multi-armed bandits algorithms: A study.
\newblock In \emph{The 2013 International Joint Conference on Neural Networks
  (IJCNN)}, 2013.

\bibitem[Even-Dar et~al.(2006)Even-Dar, Mannor, and
  Mansour]{even-dar_action_nodate}
E.~Even-Dar, S.~Mannor, and Y.~Mansour.
\newblock Action elimination and stopping conditions for the multi-armed bandit
  and reinforcement learning problems.
\newblock \emph{Journal of Machine Learning Research}, 2006.

\bibitem[Fiez et~al.(2019)Fiez, Jain, Jamieson, and
  Ratliff]{fiez2019sequential}
T.~Fiez, L.~Jain, K.~G. Jamieson, and L.~Ratliff.
\newblock Sequential experimental design for transductive linear bandits.
\newblock \emph{Advances in neural information processing systems}, 2019.

\bibitem[Gabillon et~al.(2012)Gabillon, Ghavamzadeh, and
  Lazaric]{gabillon_best_2012}
V.~Gabillon, M.~Ghavamzadeh, and A.~Lazaric.
\newblock Best arm identification: A unified approach to fixed budget and fixed
  confidence.
\newblock In \emph{Advances in Neural Information Processing Systems}. Curran
  Associates, Inc., 2012.

\bibitem[Garivier and Kaufmann(2016)]{garivier_optimal_2016}
A.~Garivier and E.~Kaufmann.
\newblock Optimal best arm identification with fixed confidence.
\newblock In \emph{29th Annual Conference on Learning Theory}. PMLR, 2016.

\bibitem[Garivier et~al.(2019)Garivier, M{\'e}nard, and
  Stoltz]{garivier2019explore}
A.~Garivier, P.~M{\'e}nard, and G.~Stoltz.
\newblock Explore first, exploit next: The true shape of regret in bandit
  problems.
\newblock \emph{Mathematics of Operations Research}, 2019.

\bibitem[Jamieson and Nowak(2014)]{jamieson_best-arm_2014}
K.~Jamieson and R.~Nowak.
\newblock Best-arm identification algorithms for multi-armed bandits in the
  fixed confidence setting.
\newblock In \emph{48th Annual Conference on Information Sciences and Systems
  ({CISS})}, 2014.

\bibitem[Jennison and Turnbull(1993)]{Jennison1993GroupST}
C.~Jennison and B.~W. Turnbull.
\newblock Group sequential tests for bivariate response: interim analyses of
  clinical trials with both efficacy and safety endpoints.
\newblock \emph{Biometrics}, 1993.

\bibitem[Jourdan et~al.(2023)Jourdan, R{\'e}my, and
  Emilie]{pmlr-v201-jourdan23a}
M.~Jourdan, D.~R{\'e}my, and K.~Emilie.
\newblock Dealing with unknown variances in best-arm identification.
\newblock In \emph{Proceedings of The 34th International Conference on
  Algorithmic Learning Theory}. PMLR, 2023.

\bibitem[Kalyanakrishnan et~al.(2012)Kalyanakrishnan, Tewari, Auer, and
  Stone]{kalyanakrishnan_pac_2012}
S.~Kalyanakrishnan, A.~Tewari, P.~Auer, and P.~Stone.
\newblock {PAC} subset selection in stochastic multi-armed bandits.
\newblock In \emph{Proceedings of the 29th International Conference on
  International Conference on Machine Learning}. Omnipress, 2012.

\bibitem[Karnin et~al.(2013)Karnin, Koren, and Somekh]{karnin_almost_2013}
Z.~Karnin, T.~Koren, and O.~Somekh.
\newblock Almost optimal exploration in multi-armed bandits.
\newblock In \emph{Proceedings of the 30th International Conference on
  International Conference on Machine Learning}. {JMLR}, 2013.

\bibitem[Katz-Samuels and Scott(2018)]{katz-samuels_feasible_2018}
J.~Katz-Samuels and C.~Scott.
\newblock Feasible arm identification.
\newblock In \emph{Proceedings of the 35th International Conference on Machine
  Learning}. {PMLR}, 2018.

\bibitem[Kaufmann et~al.(2012)Kaufmann, Korda, and Munos]{kaufmann2012thompson}
E.~Kaufmann, N.~Korda, and R.~Munos.
\newblock Thompson sampling: An asymptotically optimal finite-time analysis.
\newblock In \emph{International conference on algorithmic learning theory}.
  Springer, 2012.

\bibitem[Kim et~al.(2024)Kim, Iyengar, and
  Zeevi]{kim2024learningparetousingbootstrapped}
W.~Kim, G.~Iyengar, and A.~Zeevi.
\newblock Learning the pareto front using bootstrapped observation samples,
  2024.

\bibitem[Knowles(2006)]{knowles_parego_2006}
J.~Knowles.
\newblock Parego: a hybrid algorithm with on-line landscape approximation for
  expensive multiobjective optimization problems.
\newblock \emph{IEEE Transactions on Evolutionary Computation}, 2006.

\bibitem[Kone et~al.(2023)Kone, Kaufmann, and Richert]{kone2023adaptive}
C.~Kone, E.~Kaufmann, and L.~Richert.
\newblock Adaptive algorithms for relaxed pareto set identification.
\newblock In \emph{Thirty-seventh Conference on Neural Information Processing
  Systems}, 2023.

\bibitem[Kone et~al.(2024)Kone, Kaufmann, and Richert]{kone24a}
C.~Kone, E.~Kaufmann, and L.~Richert.
\newblock Bandit {P}areto set identification: the fixed budget setting.
\newblock In \emph{Proceedings of The 27th International Conference on
  Artificial Intelligence and Statistics}. PMLR, 2024.

\bibitem[Kuelbs et~al.(1994)Kuelbs, Li, and Linde]{kuelbs_mvn_balls}
J.~Kuelbs, W.~Li, and W.~Linde.
\newblock The gaussian measure of shifted balls.
\newblock \emph{Probability Theory and Related Fields}, 1994.

\bibitem[Kung et~al.(1975)Kung, Luccio, and Preparata]{kung_ps_algo}
H.~T. Kung, F.~Luccio, and F.~P. Preparata.
\newblock On finding the maxima of a set of vectors.
\newblock \emph{J. ACM}, 1975.

\bibitem[Lattimore and Szepesvári(2020)]{lattimore_bandit_2020}
T.~Lattimore and C.~Szepesvári.
\newblock \emph{Bandit Algorithms}.
\newblock Cambridge University Press, 2020.

\bibitem[Li et~al.(2024)Li, Jamieson, and Jain]{zhaoqi_peps}
Z.~Li, K.~Jamieson, and L.~Jain.
\newblock Optimal exploration is no harder than {T}hompson sampling.
\newblock In \emph{Proceedings of The 27th International Conference on
  Artificial Intelligence and Statistics}. PMLR, 2024.

\bibitem[Lu and Li(2009)]{lu_mvn_upb}
D.~Lu and W.~V. Li.
\newblock A note on multivariate gaussian estimates.
\newblock \emph{Journal of Mathematical Analysis and Applications}, 2009.

\bibitem[Mills(1926)]{john_mills}
J.~P. Mills.
\newblock Table of the ratio: Area to bounding ordinate, for any portion of
  normal curve.
\newblock \emph{Biometrika}, 1926.

\bibitem[Munro et~al.(2021)Munro, Janani, Cornelius, and
  et~al.]{munro_safety_2021}
A.-P.-S. Munro, L.~Janani, V.~Cornelius, and et~al.
\newblock Safety and immunogenicity of seven {COVID}-19 vaccines as a third
  dose (booster) following two doses of {ChAdOx}1 {nCov}-19 or {BNT}162b2 in
  the {UK} ({COV}-{BOOST}): a blinded, multicentre, randomised, controlled,
  phase 2 trial.
\newblock \emph{The Lancet}, 2021.

\bibitem[Russo(2016)]{russo_simple_2016}
D.~Russo.
\newblock Simple bayesian algorithms for best arm identification.
\newblock In \emph{29th Annual Conference on Learning Theory}. PMLR, 2016.

\bibitem[Shang et~al.(2020)Shang, de~Heide, Menard, Kaufmann, and
  Valko]{shang_fixed-confidence_2020}
X.~Shang, R.~de~Heide, P.~Menard, E.~Kaufmann, and M.~Valko.
\newblock Fixed-confidence guarantees for bayesian best-arm identification.
\newblock In \emph{Proceedings of the Twenty Third International Conference on
  Artificial Intelligence and Statistics}. {PMLR}, 2020.

\bibitem[Soare et~al.(2014)Soare, Lazaric, and Munos]{soare2014best}
M.~Soare, A.~Lazaric, and R.~Munos.
\newblock Best-arm identification in linear bandits.
\newblock \emph{Advances in Neural Information Processing Systems}, 2014.

\bibitem[Thompson(1933)]{thompson1933likelihood}
W.~R. Thompson.
\newblock On the likelihood that one unknown probability exceeds another in
  view of the evidence of two samples.
\newblock \emph{Biometrika}, 1933.

\bibitem[Wang and Zhu(2022)]{wang2022thompson}
S.~Wang and J.~Zhu.
\newblock Thompson sampling for (combinatorial) pure exploration.
\newblock In \emph{International Conference on Machine Learning}. PMLR, 2022.

\bibitem[Xu and Klabjan(2023)]{xu_pareto_regret}
M.~Xu and D.~Klabjan.
\newblock Pareto regret analyses in multi-objective multi-armed bandit.
\newblock In \emph{Proceedings of the 40th International Conference on Machine
  Learning}. JMLR, 2023.

\bibitem[Zuluaga et~al.(2013)Zuluaga, Sergent, Krause, and
  Püschel]{zuluaga_active_2013}
M.~Zuluaga, G.~Sergent, A.~Krause, and M.~Püschel.
\newblock Active learning for multi-objective optimization.
\newblock In \emph{Proceedings of the 30th International Conference on Machine
  Learning}. {PMLR}, 2013.

\bibitem[Zuluaga et~al.(2016)Zuluaga, Krause, and
  P{{{\"u}}}schel]{zuluaga_e-pal_2016}
M.~Zuluaga, A.~Krause, and M.~P{{{\"u}}}schel.
\newblock e-pal: An active learning approach to the multi-objective
  optimization problem.
\newblock \emph{Journal of Machine Learning Research}, 2016.

\end{thebibliography}
\bibliographystyle{abbrvnat}



\begin{appendix}
\onecolumn
\section{OUTLINE}
\label{app:outline}

The appendices are organized as follows:
\begin{itemize}
  \item Notation are summarized in Appendix~\ref{app:notation}.
  \item The $\delta$-correctness of the stopping rule is proven in Appendix~\ref{app:correctness}.
  \item Appendix~\ref{app:saddle_point_convergence} shows that the game-based sampling rule yields a saddle-point algorithm.
  \item In Appendix~\ref{app:lr_poe}, we prove some results related to the generalized likelihood and the posterior probability of error.
  \item The asymptotic upper bound on the expected sample complexity is proven in Appendix~\ref{app:sample_complexity}.
  \item The convergence of the posterior distribution is proven in Appendix~\ref{app:posterior_convergence}.
  \item Appendix~\ref{app:concentration} gathers concentration results.
  \item Appendix~\ref{app:technicalities} gathers existing and new technical results which are used for our proofs.
  \item Implementation details and additional experiments are presented in Appendix~\ref{app:add_experiments}.
\end{itemize}

\begin{table}[H]
\caption{Notation for the setting.}
\label{tab:notation_table_setting}
\begin{center}
\begin{tabular}{ll}
  \toprule
$K := \lvert \cA \rvert$ & Number of arms \\
$\Theta \C \bR^{h\times d}$ & set of parameters \\ 
$\cA \subset \bR^h$ & set of measurements (arms) \\ 
$\cZ \subset \bR^h $  & set of items  \\ 
$\bm A \in \bR^{K\times h}$ & matrix of features  \\
$\bm \theta \in \Theta$ & true parameter of the instance\\ 
$\bm\mu := \bm A \bm \theta := (\mu_1 \dots \mu_K)^\T $ &  vector mean of each arm  \\ 
$S^\star(\bm \theta)$ & Pareto set of $\cZ$ \\ 
$\cD_\Theta := \max_{a \in \cA} \max_{\bm \theta, \bm \theta'}\| \vv(\bm \theta - \bm \theta')\|_{\Sigma^{-1}\otimes aa^\T}$ & diameter of the means space \\
  \bottomrule
\end{tabular}
\end{center}
\end{table}

\section{NOTATION} \label{app:notation}

We recall some commonly used notation:
the set of integers $[K] \eqdef \{1, \cdots, K\}$,
the complement $X^{\complement}$ and interior $\mathring X$ and the closure $\overline X$ of a set $X$, 
the indicator function $\indi{X}$ of an event,
the probability $\bP_{\nu}$ and the expectation $\bE_{\nu}$ under distribution $\nu$,
Landau's notation $o$, $\cO$, $\Omega$ and $\Theta$,
the $(K-1)$-dimensional probability simplex $\simplex \eqdef \left\{w \in \bR_{+}^{K} \mid w \geq 0 , \: \sum_{i \in [K]} w_i = 1 \right\}$.
$\zeta$ is the Riemann $\zeta$ function. When $\bm \lambda$ is a matrix and $\pi$ is a distribution supported on vectors, we write $\bm \lambda \sim \pi$ to denote $\vv(\bm \lambda) \sim \pi$. When $\bm M$ is a set of matrices $\vv(\bm M) := \{ \vv(\bm m) \mid \bm m \in \bm M \}$.  

While Table~\ref{tab:notation_table_setting} gathers problem-specific notation, Table~\ref{tab:notation_table_algorithms} groups notation for the algorithms.
\begin{table}[H]
\caption{Notation for algorithms.}
\label{tab:notation_table_algorithms}
\begin{center}
\begin{tabular}{ll}
  \toprule
$\alt(S)$  & set of parameters whose Pareto set is not $S$ \\
$S_t$ &  pareto set of the empirical mean at time $t$  \\
$\wh{\bm\theta}_t$  & (regularized) least-squares estimator of $\bm\theta$ \\
$a_t \in \cA$ &  measurement (arm) pulled at time $t$  \\
$V_t := \xi I_h + V_{\bm N_t}$ & design matrix after $t$ iterations \\
$a_t \C \cA$ & feature pulled at time $t$ \\ 
$L_\cA:= \max_{a\in \cA} \| a \|_2$ & maximum feature norm \\
$L_*:= 2\max_{a\in \cA} \| a \|_{V(\w_\text{exp})^{-1}}^2 $ & squared feature norm due to forced exploration \\ 
$L_\Theta := L_{\cM} \eqdef \max_{\bm \lambda \in \Theta} \max_{c \in [d]} \|\bm \lambda e_c \|_2$ & maximum regression norm 
\\
$\wt a:=  \Sigma^{-1/2} \otimes a$ & extended feature (matrix) \\
$\wt V_t:= \Sigma^{-1} \otimes V_t$ & extended desing matrix \\
$\Lambda_t:= \Theta_t \cap \alt(\wh S_t)$ & truncated alternative space \\
$\pi_t:= \cN(\vv(\bm \theta_t), \eta_t^{-1/2}\bm \Sigma_t; \Lambda_t)$ & Gaussian truncated to $\Lambda_t$ \\
$\bm \lambda_t \sim \pi_t$ & strategy of the \emph{min} player at time $t$
\\ 
$\w_t$ &  strategy of the \emph{max} player at time $t$\\
$\Delta_{\min}$  & the minimum Pareto gap (Section~\ref{ssec:conv_alt}) \\
$\cF_t:= \sigma\lp a_1, y_1, \dots a_{t+1}, y_{t}\rp$ & information gathered by the algorithm after $t$ rounds \\
$X_t\sim \nu_{a_t}$ &  observation collected from $\nu_{a_t}$  
\\
$\veps_t:= X_t - (I_d \otimes a_t)\vv(\bm \mu)$ & centered  (sub)Gaussian noise at time $t$  
\\
$\eta_t$ & learning rate for the \emph{min} learner (Section~\ref{ssec:learning_set}) \\ 
   \bottomrule 
\end{tabular}
\end{center}
\end{table}


\section{CORRECTNESS OF THE STOPPING RULE}
\label{app:correctness}

In this section, we prove the correctness of the PS (Posterior Sampling) stopping rule. We recall the definition of $\tau_\delta^\text{PS}$: 
\[	
	\tau_{\delta}^{\textrm{PS}} \eqdef \inf\left\{t \mid \forall m \in [M(t-1,\delta)], \: \bm \theta^{m}_{t} \notin \Theta \cap \alt(\widehat S_t) \right\} \: ,
\]
where, conditioned on $ \cF_{t}$, $ \bm\theta_t, \bm \theta_t^1,\dots, \bm \theta_t^m $ are \iid samples from $\cN(\vv(\wh{\bm\theta}_t), c(t-1, \delta)\bm \Sigma_{t})$. 
To prove the correctness of this stopping rule, we will have to control the randomness in the posterior, which we do by introducing the concentration event $\cE_{\delta} \eqdef \bigcap_{t \in \bN} \cE_{t,\delta}$ with
\begin{equation*} 
	\cE_{t,\delta} \eqdef \left\{\left\| \vv(\bm \theta - \widehat{\bm \theta}_{t} )\right\|_{\bm \Sigma_t^{-1}}^2  \le 2 \beta(t-1,\delta)\right\} \: ,
	\end{equation*}
for which, correct values of the thresholds are given in the lemma below.
\begin{restatable}{lemma}{concentrationThreshold} \label{lem:concentration_threshold}
    Let $s > 1$, $\zeta$ be the Riemann $\zeta$ function and $\overline{W}_{-1}(x) \eqdef - W_{-1}(-e^{-x})$ for all $x \ge 1$, where $W_{-1}$ is the negative branch of the Lambert $W$ function.
    Let $\cE_{\delta}$ in~\eqref{eq:concentration_event}.
    Then, we have $\bP_{\bm{\nu}}(\cE_{\delta}^{\complement}) \le \delta$ by taking
\begin{align*} 
    \beta(t, \delta) &= \frac{dK}{2} \overline{W}_{-1}\left(\frac{2}{dK} \log \frac{e^{Ks} \zeta(s)^K}{\delta} + \frac{2 s}{d} \log \left( 1 + \frac{d}{2s}\log \frac{t}{K} \right) + 1 \right)  
\end{align*}
in the unstructured setting, and taking
\begin{equation*} \label{eq:concentration_threshold_structured}
    \sqrt{\beta(t, \delta)}  = \sqrt{\log\left(\frac{1}{\delta}\left(\frac{L_{\cA}^2}{h \xi}t + 1 \right)^{\frac{dh}{2}}\right)} + \sqrt{\frac{d L_{\cM}^2}{2 \lambda_{\min}(\Sigma) \xi}} \: 
\end{equation*}
    in the transductive linear setting, where $L_{\cA} \eqdef \max_{a \in \cA} \|a\|_2$ and $L_{\cM} \eqdef \max_{\bm \lambda \in \Theta} \max_{c \in [d]} \|\bm \lambda e_c \|_2$.
\end{restatable}

The following result shows that to prove the $\delta$-correctness, it is sufficient to show some anti-concentration of the posterior under the event $\cE$. 
$\bP_{\bm \nu}(\cE_{\delta/2} \cap \{ \tau_{\delta}^{\mathrm{PS}} < + \infty, \widehat S_{\tau_{\delta}^{\mathrm{PS}}} \neq S^\star\}) \le \delta/2 \: .$

\begin{lemma}
\label{lem:correctness}
	For all $c, M$, the PS stopping rule satisfies 
\begin{equation}
\label{eq:ws-qu-op}
	\bP_{\bm \nu}(\tau_{\delta}^{\mathrm{PS}} < + \infty, \widehat S_{\tau_{\delta}^{\mathrm{PS}}} \neq S^\star) \leq \delta/2 + \bE\lsb \sum_{t\geq 1} \ind_{\cE_{t,\delta/2}} \ind_{\widehat S_t\neq S^\star} \exp\lp - M(t-1, \delta) \bP_{\bm \nu}\lp  \bm \theta_t \in \Theta \cap \alt(\wh S_t) \mid \cF_{t} \rp \rp \rsb. 
\end{equation}
\end{lemma}
\begin{proof}
Using the definition of the PS stopping rule in~\eqref{eq:PS_stopping_rule}, we obtain
\begin{eqnarray*}
	\bP_{\bm \nu}\lp \tau_\delta^{\textrm{PS}} < \infty, \widehat S_{\tau_{\delta}^{\mathrm{PS}}} \neq S^\star\rp &\leq& \delta/2 + \bP_{\bm \nu}\lp \cE_{\delta/2} \cap \{\tau_\delta^{\textrm{PS}} <\infty, \widehat S_{\tau_{\delta}^{\mathrm{PS}}} \neq S^\star \}\rp \\
	&\leq& \delta/2 + \sum_{t\geq 1} \bP_{\bm \nu}\lp \cE_{t,\delta/2} \cap \{\tau_\delta^{\textrm{PS}} = t, \widehat S_{\tau_{\delta}^{\mathrm{PS}}} \neq S^\star\}\rp\\
	&=& \delta/2 + \bE\lsb \sum_{t\geq 1} \ind_{\cE_{t,\delta/2}} \ind_{\wh S_t\neq S^\star} \bP_{\bm \nu}\lp \tau_\delta^{\textrm{PS}} = t \mid \cF_{t} \rp \rsb \\
	&\leq& \delta/2 + \bE\lsb \sum_{t\geq 1} \ind_{\cE_{t,\delta/2}} \ind_{\wh S_t\neq S^\star} \bP_{\bm \nu}\lp \forall m\leq M(t-1, \delta), \bm \theta_t^m \notin \Theta \cap \alt(\wh S_t) \mid \cF_{t} \rp \rsb \\
	&=& \delta/2 + \bE\lsb \sum_{t\geq 1} \ind_{\cE_{t,\delta/2}} \ind_{\wh S_t\neq S^\star} \bP_{\bm \nu}\lp  \bm \theta_t \notin  \Theta \cap \alt(\wh S_t) \mid \cF_{t} \rp^{M(t-1, \delta)} \rsb \\ 
	&=&  \delta/2 + \bE\lsb \sum_{t\geq 1} \ind_{\cE_{t,\delta/2}} \ind_{\wh S_t\neq S^\star} \lp 1 -\bP_{\bm \nu}\lp  \bm \theta_t \in \Theta \cap \alt(\wh S_t) \mid \cF_{t}\rp\rp^{M(t-1, \delta)} \rsb
\end{eqnarray*}
which follows since the samples $\bm \theta_t, \bm  \theta_t^1, \dots, \bm \theta_t^m$ are \iid given $\cF_{t}$. Further recalling that $1-x\leq \exp(-x)$, it follows
\begin{equation}
\label{eq:az-er-ty}
	\bP_{\bm \nu}\lp \tau_\delta^{\textrm{PS}} < \infty, \widehat S_{\tau_{\delta}^{\mathrm{PS}}} \neq S^\star\rp \leq \delta/2 + \bE\lsb \sum_{t\geq 1} \ind_{\cE_{t,\delta/2}} \ind_{\wh S_t\neq S^\star} \exp\lp - M(t-1, \delta) \bP_{\bm \nu}\lp  \bm \theta_t \in \Theta \cap \alt(\wh S_t) \mid \cF_{t} \rp \rp \rsb. 
\end{equation}
\end{proof}
Lemma~\ref{lem:correctness} shows that choosing $c, M$ such that  
$$\bE\lsb \sum_{t\geq 1} \ind_{\cE_{t,\delta/2}} \ind_{\wh S_t\neq S^\star} \exp\lp - M(t-1, \delta) \bP_{\bm \nu}\lp  \bm \theta_t \in \Theta \cap \alt(\wh S_t) \mid \cF_{t} \rp \rp \rsb $$
is not larger than $\delta/2$ will ensure the correctness of the PS stopping rule. However, it is not possible to pick $M(t-1, \delta) \propto \bP_{\bm \nu}( \bm \theta_t \in \Theta \cap \alt(\wh S_t) \mid \cF_{t})$ as the latter quantity is intractable due to multidimensional integration over a non-trivial domain. Instead, we will derive in the following sections, tractable lower bounds on $\ind_{\cE_{t,\delta/2}} \ind_{S_t\neq S^\star} \bP_{\bm \nu}( \bm \theta_t \in \Theta \cap \alt(\wh S_t) \mid \cF_{t})$ with 
\begin{align} \label{eq:def_alt}
	\alt(\widehat S_{t}) = \lp \bigcup_{x \ne z\in\widehat S_{t}^2} \bigcap_{c \in [d]} \{\bm \lambda \in \bR^{h \times d}\mid  \langle E_{z,x}(c), \vv(\bm \lambda) \rangle \ge 0 \} \rp
	\bigcup  \lp \bigcup_{z \notin \widehat S_t}  \bigcap_{x \in \widehat S_{t}} \bigcap_{c \in [d]}\{\bm \lambda \in \bR^{h \times d} \mid \langle E_{x,z}(c), \vv(\bm \lambda) \rangle \ge 0 \}\rp \: ,
\end{align} 
where $E_{z,x}(c) = e_c \otimes (z-x) $ and $e_c = (\indi{c'=c})_{c' \in [d]}$. 
 
\subsection{Unstructured Setting}
\label{app:ssec_correctness_unstructured}

In this section, we calibrate $c, M$ to ensure the correctness of the PS stopping rule in the unstructured setting. 
In this setting, we have $\Theta = \bR^{K \times d}, \bm A = I_K$, so that the quantity to control is 
$$ \ind_{\cE_{t,\delta/2}} \ind_{\wh S_t\neq S^\star} \bP_{\bm \nu}( \bm \theta_t \in \alt(\wh S_t) \mid \cF_{\infty}) \quad \text{with} \quad \vv(\bm \theta_t - \wh{\bm\theta}_t) \mid \cF_{t} \sim  \cN( \bm 0_{K d}, c(t-1, \delta)\bm \Sigma_{t}) \: ,
$$ 
and
\begin{eqnarray*} 
	\alt(\wh S_t) = \lp \bigcup_{i\neq j \in \wh S_t^2} \lb \bm \lambda := (\lambda_1 \dots \lambda_K)^\T \in  \Theta\mid \lambda_i \prec \lambda_j \rb 
	\rp \bigcup  \lp \bigcup_{i \notin \wh S_t} \lb \bm \lambda := (\lambda_1 \dots \lambda_K)^\T \in \Theta \mid \forall  j \in \wh S_t, \lambda_i \nprec \lambda_j \rb\rp. 
\end{eqnarray*} 
Indeed, recalling that $\alt(S)$ is the set of parameters for which the Pareto set differs from $S$. To change the Pareto set, we either add an arm or remove one from it. For a parameter instance $\bm \lambda := (\lambda_1 \dots \lambda_K)^\T$, if $i \in S$ and there exists $j \in S$ such that $\lambda_i \prec \lambda_j$, then as $i$ is a dominated arm in the instance $\bm \lambda$, we will have have $S^\star(\bm \lambda) \neq S$. Similarly, if for some $i \notin S$ it holds that for all $j \in S$, $\lambda_i \nprec \lambda_j$, then we cannot have $S^\star(\bm \lambda) = S$, otherwise, as $i\notin S^\star(\bm \lambda)$, an arm from $S^\star(\bm \lambda) = S$ would have dominated $i$. This is formally shown in Lemma~\ref{lem:alt_convexity}, where it is further proven that
$$ \alt(S) = \lp \bigcup_{(i,j) \in S^2, i \ne j} \lb \bm \lambda \in \bR^{d\times K} \mid   \lambda_i \prec \lambda_j \rp 
	\rp \bigcup \lp \bigcup_{i \in S^c} \bigcup_{\bar d^i\in [d]^{S}} \lb \bm \lambda \mid \forall j \in S, \lambda_i(d^i(\sigma(j))) \geq \lambda_j(d^i(\sigma(j))) \rp\rp $$
where $\sigma$ is any permutation that maps $S$ onto $\{ 1,\dots, \lvert S\rvert \}$.  In the sequel, we let
$$\alt^-(S) := \bigcup_{(i,j) \in S^2, i \ne j} \lb \bm \lambda  \mid   \lambda_i \prec \lambda_j \rb 
	 \;\text{and}\;\alt^+(S) := \bigcup_{i \in S^c} \bigcup_{\bar d^i\in [d]^{S}} \lb \bm \lambda \mid \forall j \in S: \lambda_i(d^i(\sigma(j))) \geq \lambda_j(d^i(\sigma(j)))\rb$$ so that 
$\alt(S) = \alt^+(S) \cup  \alt^-(S)$. We prove the following lemma. 
\begin{lemma}
\label{lem:anti_conc}
	It holds that  	
	\begin{align*}
	\ind_{\wh S_t \neq S^\star} \bP_{\wh \Pi_t \mid \cF_t}(\alt(\wh S_t) ) &\geq  \max \left\{ \max_{i \notin \wh S_t, \bar d \in [d]^{\lvert \wh S_t \rvert}} \ind_{(\forall j \in \wh S_t, \mu_i(\bar d^\sigma_j)\geq \mu_j(\bar d^\sigma_j))} \bP_{Y \sim \cN(0_d, \Sigma)}\lp Y > \sqrt{\frac{N_{t, i}}{c(t-1, \delta)}} (\mu_i - \wh \mu_{t, i}) \rp \right. \\  & \left.\cdot \prod_{j \in \wh S_t} \bP_{X \sim \cN(0, 1)}\lp X>  \sqrt{\frac{N_{t, j}}{c(t-1, \delta)}}  \| \mu_{j} - \wh \mu_{t,j}\|_{\Sigma^{-1}}\rp , \right.\\
	&\qquad \left.\max_{i\neq j \in \wh S_t^2, } \ind_{\lp \mu_i \prec \mu_j\rp}\bP_{Y\sim \cN(0_d, \Sigma)}\lp Y \geq \sqrt{\lp\frac1{N_{t, i}} + \frac1{N_{t, j}}\rp^{-1}\frac 1{c(t-1, \delta)}} \lp (\wh \mu_{t, i} - \wh \mu_{t, j}) - (\mu_i - \mu_j)\rp\rp \right\}.
\end{align*}
\end{lemma}

\begin{proof}
\begin{eqnarray*}
	\ind_{\wh S_t \neq S^\star} \bP_{\wh \Pi_t \mid \cF_t}( \alt^-(\wh S_t) ) &\geq&  \ind_{\wh S_t \neq S^\star} \max_{(i, j) \in \wh S_t, i \neq j} \bP_{\wh \Pi_t \mid \cF_t}( \{ \bm\lambda= (\lambda_1 \dots \lambda_K)^\T  \in \bR^{d \times K} : \lambda_i \prec \lambda_j \}) \\
	&\geq& \max_{(i, j) \in \wh S_t, i \neq j} \ind_{(\mu_i \prec \mu_j)}\bP_{\wh \Pi_t \mid \cF_t}( \{ \bm\lambda= (\lambda_1 \dots \lambda_K)^\T  \in \bR^{d \times K} : \lambda_i \prec \lambda_j \}) \\
	&=&  \max_{(i, j) \in \wh S_t, i \neq j} \ind_{(\mu_i \prec \mu_j)}\bP_{\bm \nu}(\theta_{t, i} \prec \theta_{t, j} \mid \cF_t), 
\end{eqnarray*}
where $\theta_{t, i} = \bm \theta^T \wt e_i$ and $\wt e_1, \dots, \wt e_K$ is the cananical basis of $\bR^K$. In the unstructured setting, it is simple to check that $\bm \theta_t^\T = (\theta_{t, 1} \dots \theta_{t, K})$ where $\theta_{t, i} \mid \cF_t \sim \cN(\wh\mu_{t, i}, c(t-1, \delta)\Sigma/ N_{t,i})$. 
Therefore, we have 

\begin{eqnarray*}
	\ind_{\wh S_t \neq S^\star} \bP_{\wh \Pi_t \mid \cF_t}( \alt^-(\wh S_t) ) &\geq& \max_{(i, j) \in \wh S_t, i \neq j} \ind_{(\mu_i \prec \mu_j)}\bP_{\bm \nu}((\theta_{t, i} - \theta_{t, j}) - (\wh\mu_{t,i} - \wh\mu_{t,j}) \prec -(\wh\mu_{t,i} - \wh\mu_{t,j}) \mid \cF_t) \\
	&\geq& \max_{(i, j) \in \wh S_t, i \neq j} \ind_{(\mu_i \prec \mu_j)}\bP_{\bm \nu}((\theta_{t, i} - \theta_{t, j}) - (\wh\mu_{t,i} - \wh\mu_{t,j}) \prec (\mu_i - \mu_j) -(\wh\mu_{t,i} - \wh\mu_{t,j}) \mid \cF_t)
\end{eqnarray*}
then, observe that $(\theta_{t, i} - \theta_{t, j})  \mid \cF_{t} \sim \cN\lp \wh\mu_{t, i} - \wh\mu_{t, j}, (N_{t, i}^{-1} + N_{t, j}^{-1})c(t, \delta)\Sigma\rp$, so that letting 
$Y \sim \cN(0_d, \Sigma)$, we have 
\begin{eqnarray*}
	\ind_{\wh S_t \neq S^\star} \bP_{\wh \Pi_t \mid \cF_t}( \alt^-(\wh S_t) ) &\geq& \max_{i\neq j \in \wh S_t^2} \ind_{\lp \mu_i \prec \mu_j\rp}\bP_{Y\sim \cN(0_d, \Sigma)}\lp Y \prec \sqrt{\lp \frac1 {N_{t, i}} + \frac1{N_{t, j}}\rp^{-1}\frac1{c(t-1, \delta)}} ((\mu_i - \mu_j) - (\wh\mu_{t, i} - \wh\mu_{t, j}))\rp \\
	&=& \max_{i\neq j \in S_t^2, } \ind_{\lp \mu_i \prec \mu_j\rp}\bP_{Y\sim \cN(0_d, \Sigma)}\lp Y \geq \sqrt{\lp\frac1{N_{t, i}} + \frac1{N_{t, j}}\rp^{-1}\frac 1{c(t-1, \delta)}} \lp (\wh \mu_{t, i} - \wh \mu_{t, j}) - (\mu_i - \mu_j)\rp\rp. 
\end{eqnarray*}
We now prove a similar result on $\alt^+(\wh S_t)$.  We have 
\begin{eqnarray*}
	\ind_{\wh S_t \neq S^\star} \bP_{\wh \Pi_t \mid \cF_t}( \alt^+(\wh S_t) ) &\geq& \max_{i \notin \wh S_t, \bar d \in [d]^{\lvert S_t \rvert}} \ind_{(\forall j \in \wh S_t, \mu_i(\bar d^\sigma_j)\geq \mu_j(\bar d^\sigma_j))} \bP_{\bm \nu}( \forall j \in \wh S_t, \theta_{t, i}(\bar d^\sigma_j) \geq \theta_{t, j}(\bar d^\sigma_j) \mid \cF_t) \\
	&\geq& \max_{i \notin \wh S_t, \bar d \in [d]^{\lvert \wh S_t \rvert}} \ind_{(\forall j \in \wh S_t, \mu_i(\bar d^\sigma_j)\geq \mu_j(\bar d^\sigma_j))}  \bP_{\bm \nu}( \forall j \in \wh S_t, (\theta_{t, i} - \mu_{i})(\bar d^\sigma_j) \geq (\theta_{t, j} - \mu_{j})(\bar d^\sigma_j) \mid \cF_t) \\
	&\geq& \max_{i \notin \wh S_t, \bar d \in [d]^{\lvert \wh S_t \rvert}} \ind_{(\forall j \in \wh S_t, \mu_i(\bar d^\sigma_j)\geq \mu_j(\bar d^\sigma_j))} \bP_{\bm\nu}((\theta_{t, i} - \mu_{i}) >0_d \mid \cF_t)\prod_{j \in \wh S_t} \bP_{\bm \nu}((\theta_{t, j} - \mu_{j})(\bar d^\sigma_j) >0 \mid \cF_t)
\end{eqnarray*}
then, noting again that $(\theta_{t, i} - \wh \mu_{t, i})\mid \cF_t  \sim \cN(0, c(t-1, \delta) \Sigma/N_{t, i})$, it follows that 
\begin{eqnarray}
	\ind_{\wh S_t \neq S^\star} \bP_{\wh \Pi_t \mid \cF_t}( \alt^+(\wh S_t) ) \geq \max_{i \notin \wh S_t, \bar d \in [d]^{\lvert \wh S_t \rvert}} \ind_{(\forall j \in \wh S_t, \mu_i(\bar d^\sigma_j)\geq \mu_j(\bar d^\sigma_j))} \bP_{Y \sim \cN(0_d, \Sigma)}\lp Y > \sqrt{\frac{N_{t, i}}{c(t-1, \delta)}} (\mu_i - \wh \mu_{t, i})\rp  \nonumber & 
	\\ \cdot \prod_{j \in \wh S_t} \bP_{\bm \nu}((\theta_{t, j} - \mu_{j})(\bar d^\sigma_j)>0 \mid \cF_t) \label{eq:ab-pm-lm}.
\end{eqnarray}
Further noting that by Cauchy-Schwartz inequality,  \begin{eqnarray*}
 (\wh \mu_{t, j} - \mu_j)(\bar d^\sigma_j)	&\leq&  \| \wt e_{\bar d^\sigma_j}\|_{\Sigma}  \| \mu_{j} - \wh \mu_{t,j}\|_{\Sigma^{-1}}, 
\end{eqnarray*}
it follows for any $j\in \wh S_t$ that 
\begin{eqnarray*}
\bP_{\bm \nu}((\theta_{t, j} - \mu_{j})(\bar d^\sigma_j)>0 \mid \cF_t) &= & \bP_{\bm \nu}((\theta_{t, j} - \wh \mu_{t, j})(\bar d^\sigma_j)>(\mu_j - \wh \mu_{t, j})(\bar d^\sigma_j) \mid \cF_t) 	\\
&\geq& \bP_{\bm \nu}((\theta_{t, j} - \wh \mu_{t, j})(\bar d^\sigma_j)>  \| \wt e_{\bar d^\sigma_j}\|_{\Sigma}  \| \mu_{j} - \wh \mu_{t,j}\|_{\Sigma^{-1}} \mid \cF_t) \\
&=& \bP_{X \sim \cN(0, 1)}\lp X>  \sqrt{\frac{N_{t, j}}{c(t-1, \delta)}}  \| \mu_{j} - \wh \mu_{t,j}\|_{\Sigma^{-1}} \rp, 
\end{eqnarray*}
which follows from $(\theta_{t, j} - \wh \mu_{t, j})(\bar d^\sigma_j) \mid \cF_t \sim \cN(0, c(t-1, \delta)\| \wt e_{\bar d^\sigma_j}\|/N_{t, \delta})$. 
Therefore, combining the last display with \eqref{eq:ab-pm-lm} yields 
\begin{align*}
	\ind_{\wh S_t \neq S^\star} \bP_{\wh \Pi_t \mid \cF_t}( \alt^+(\wh S_t) ) \geq &\max_{i \notin \wh S_t, \bar d \in [d]^{\lvert \wh S_t \rvert}} \ind_{(\forall j \in \wh S_t, \mu_i(\bar d^\sigma_j)\geq \mu_j(\bar d^\sigma_j))} \bP_{Y \sim \cN(0_d, \Sigma)}\lp Y > \sqrt{\frac{N_{t, i}}{c(t-1, \delta)}} (\mu_i - \wh \mu_{t, i}) \rp \\& \cdot \prod_{j \in \wh S_t} \bP_{X \sim \cN(0, 1)}\lp X>  \sqrt{\frac{N_{t, j}}{c(t-1, \delta)}}  \| \mu_{j} - \wh \mu_{t,j}\|_{\Sigma^{-1}} \rp, 
\end{align*}
which achieves the proof. 
\end{proof}

\subsubsection{Special Case: PSI with Independent Covariate}
\label{ssec_ind}
In this section, we specialize Lemma~\ref{lem:anti_conc} to the case where $\Sigma$ is diagonal. In this case, $\Sigma = \text{diag}(\{\sigma^2_c\}_{c \in [d]})$ and 
\begin{align*}
	(*) = \bP_{Y\sim \cN(0_d, \Sigma)}&\lp Y \geq \sqrt{\lp\frac1{N_{t, i}} + \frac1{N_{t, j}}\rp^{-1}\frac 1{c(t-1, \delta)}} \lp (\wh \mu_{t, i} - \wh \mu_{t, j}) - (\mu_i - \mu_j)\rp \rp = \\ & \prod_{c \in [d]} \bP_{X \sim \cN(0, 1)} \lp X > \underbrace{\sqrt{\lp\frac1{N_{t, i}} + \frac1{N_{t, j}}\rp^{-1}\frac 1{\sigma_c^2 c(t-1, \delta)}} \lp (\wh \mu_{t, i} - \wh \mu_{t, j}) - (\mu_i - \mu_j)\rp(c)}_{Z_{i,j}(c) }\rp, 
\end{align*}
which we may rewrite with Mills ratio as  $ R(x) := \frac{\bP(X >x)}{f_X(x)}$ so $$\bP(X >x) = R(x) \exp(-\frac12x^2) \frac{1}{\sqrt{2\pi}} = \wt R(x) \exp(-\frac12x^2),$$
and \begin{align*}
	(*) &= \exp\lp -\frac12 \sum_{c\in [d] } Z_{i,j}(c) \rp\prod_{c \in [d]}\wt R(Z_{i,j}(c))  \\
	& = \exp\lp -\lp\frac1{N_{t, i}} + \frac1{N_{t, j}}\rp^{-1}\frac 1{c(t-1, \delta)}\sum_{c\in [d] } \frac{\lp(\wh \mu_{t, i} - \wh \mu_{t, j}) - (\mu_i - \mu_j)\rp(c)^2}{2\sigma_c^2} \rp\prod_{c \in [d]}\wt R(Z_{i,j}(c)) 
\end{align*}
and by Cauchy-Swchartz inequality, 
$$ \lvert \lp(\wh \mu_{t, i} - \wh \mu_{t, j}) - (\mu_i - \mu_j)\rp(c) \rvert \leq \sqrt{\frac{1}{N_{t, i}} + \frac{1}{N_{t, j}}} \sqrt{N_{t, i}(\wh \mu_{t, i} - \mu_i)(c)^2 + N_{t, j} (\mu_j - \wh \mu_{t, j})(c)^2}.$$
Therefore 
\begin{align}
\label{eq:sz-az-op}
	(*) \geq \exp\lp - \frac 1{c(t-1, \delta)} \sum_{k \in \{i,j\}} \frac{N_{t, k}}{2} \|\wh  \mu_{t, k} - \mu_k\|^2_{\Sigma^{-1}}\rp \prod_{c \in [d]}\wt R(Z_{i,j}(c)), 
\end{align}
and using the following lemma,
\begin{lemma}
\label{lem:mills_properties}
The Mills ratio $R$ is decreasing, log-convex and satisfies for all $(x_1, \dots, x_p) \in \bR^p$, 
$$R\lp \frac1p \sum_{i=1}^p x_i \rp^p \leq \prod_{i=1}^p R(x_i).$$	
\end{lemma}
we obtain  
\begin{eqnarray*}
\prod_{c \in [d]}\wt R(Z_{i,j}(c)) &\geq& \wt R\lp \frac1d\sum_{c \in [d]}Z_{i,j}(c) \rp^d \\
&\overset{(a)}{\geq}& \wt R\lp \frac1d\sum_{c \in [d]} \sqrt{\frac 1{\sigma_c^2 c(t-1, \delta)} } \sqrt{N_{t, i}(\wh \mu_{t, i} - \mu_i)(c)^2 + N_{t, j} (\mu_j - \wh \mu_{t, j})(c)^2} \rp^d \\
&\overset{(b)}{\geq}& \wt R\lp  \frac{\sqrt{2}}{\sqrt{c(t-1, \delta)d}} \sqrt{\sum_{k \in \{ i, j\}} \frac{N_{t,k}}{2} \|\wh  \mu_{t, k} - \mu_k\|^2_{\Sigma^{-1}}}\rp^d, 
\end{eqnarray*}
where $(a)$ follows since $\wt R$ is decreasing (Lemma~\ref{lem:mills_properties}) and $(b)$ follows from this monotonicity and Cauchy-Schwartz. 
Combining the last display with \eqref{eq:sz-az-op}, we prove that on the event $\cE_{t, \delta/2}$, we have 
$$ (*) \geq \exp\lp - \frac{\beta(t-1, \delta/2)}{c(t-1, \delta)}\rp \wt R\lp \sqrt{\frac{{2 \beta(t-1, \delta/2)}}{{c(t-1, \delta)d}}}\rp^d$$
Choosing  $$ c(t, \delta) := \frac{\beta(t, \delta/2)}{ \log(1/\delta)}$$ yields, 
$$ (*) \geq \delta \wt R\lp\sqrt{\frac{2\log(1/\delta)}{d}}\rp^d  = r(\delta, d) \quad \text{with} \quad r(\delta, n) = \wt R\lp\sqrt{\frac{2\log(1/\delta)}{n}}\rp^n$$ and proceeding identically, 
we prove that under the event $\cE_{t, \delta/2}$, 
\begin{align*}
\bP_{Y \sim \cN(0_d, \Sigma)}\lp Y > \sqrt{\frac{N_{t, i}}{c(t-1, \delta)}} (\mu_i - \wh \mu_{t, i}) \rp \prod_{j \in \wh S_t} \bP_{X \sim \cN(0, 1)}\lp X>  \sqrt{\frac{N_{t, j}}{c(t-1, \delta)}}\| \mu_{j} - \wh \mu_{t,j}\|_{\Sigma^{-1}}\rp \geq \delta  \wt R\lp\sqrt{\frac{2\log(1/\delta)}{d + \lvert \wh S_t \rvert}}\rp^{d + \lvert \wh S_t \rvert}.
\end{align*}
All put together, we proved that 
\begin{eqnarray*}
	\ind_{\cE_{t, \delta/2}}\ind_{\wh S_t \neq S^\star} \bP_{\wh \Pi_t \mid \cF_t}( \alt(\wh S_t) ) &\geq& \delta \max\lb \ind_{\bm \mu \in \alt^-(\wh S_t) } r(\delta, d), \ind_{(\bm \mu \in \alt^+(\wh S_t))} r(\delta, d+ \lvert \wh S_t \rvert)\rb \\
	&\geq& \delta \min (r(\delta, d), r(\delta, d+ \lvert \wh S_t \rvert))
\end{eqnarray*}
which follows since $ \wh S_t \neq S^\star \impl \bm \mu \in \alt(\wh S_t ) = \alt^+(\wh S_t) \cup \alt^-(\wh S_t)$. 

The conclusion is immediate by taking $M(t, \delta) = \left\lceil \frac{\log(2 t^s\zeta(s)/\delta)}{\delta q(t,\delta)} \right \rceil$ where $q(t-1,\delta) = \min(r(\delta, d), r(\delta, d + \lvert \wh S_t \rvert))$. 
Moreover, it is known~\citep{mills_inequality} that for $x\geq 0$, 
$$ R(x) \geq \frac{2}{x + \sqrt{x^2 + 4}} $$
and $R(x) \sim \frac 1x, x \rightarrow +\infty$. The Mills ratio of the standard normal is implemented in common scientific libraries. To compute $M(t, \delta)$, we just need to compute at most the values $(r(\delta, d + k))_{k \in [K] \cup \{0\}}$ at the initialisation of the algorithm. 
 Using the lower bound on the Mills ratio, one can easily check that  $$ \frac{\log(M(t, \delta))c(t, \delta)}{\log(1/\delta)} \underset{\delta \rightarrow 0}{\leq} 1.$$

\begin{proof}[Proof of Lemma~\ref{lem:mills_properties}]
	The monotonicity is well-known and simple to prove by noting that 
\begin{eqnarray*}
 	R(x) &:=& \exp\lp x^2/2 \rp {\int_x^\infty \exp\lp -{t^2}/{2}  \rp dt} \\
 	&=& \exp\lp x^2/2 \rp \int_0^\infty \exp\lp - (t +x)^2/2\rp dt \\
 	&=& \int_0^\infty \exp\lp - tx - t^2/2 \rp dt,  
 \end{eqnarray*} 
 from which we have $R(x + \alpha^2) = \int_0^\infty \exp(-t\alpha^2)\exp\lp - tx - t^2/2 \rp dt < R(x)$. The log-convexity is proven in Theorem~2.5 of~\cite{prop_mills}. From the log-convexity and using Jensen inequality, we have 
	\begin{eqnarray*}
		\log R\lp \frac1p \sum_{i=1}^p x_i \rp &\leq&  \frac1p \sum_{i=1}^p \log R\lp x_i \rp  \\
		&=& \frac1p \log\lp \prod_{i=1}^p R(x_i) \rp
	\end{eqnarray*}
	which by monotonicity of the $\log$ proves the claimed statement. 
\end{proof}

\subsubsection{PSI with a Covariance Matrix}
We specialise Lemma~\ref{lem:anti_conc} to the case where the covariance $\Sigma$ is non-diagonal. To derive this result, we will use the following lemma.
\begin{lemma}
\label{lem:mills_cov}
	Let $\Sigma$ a covariance matrix, $V$ diagonal matrix such that $(V - \Sigma^{-1})$ is psd, and define $d_{\tiny \Sigma} := \| 1_d \|_{(V^{-1/2}\Sigma^{-1} V^{-1/2})}^2$. Then, for all $x\in \bR^d$, 
	it holds that 
	$$ \bP_{X \sim \cN(0_d, \Sigma)}\lp X \geq x \rp \geq (2\pi)^{-d/2} \det(V\Sigma)^{-1/2}\exp\lp - \frac12 x^\T \Sigma^{-1} x \rp \prod_{c \in [d]}R(e_c^\T V^{-1/2}\Sigma^{-1} x),$$ and in particular,  $\prod_{c \in [d]}R(e_c^\T V^{-1/2}\Sigma^{-1} x) \geq  R\lp \bm \| x \|_{\Sigma^{-1}} \sqrt{d_{\tiny \Sigma}}/d  \rp^d$.
	\end{lemma}
Applying this result with $V = \bar \sigma I_d$ where $\bar \sigma = \|| \Sigma^{-1}|\|$, we have  
\begin{eqnarray*}
\bP_{Y \sim \cN(0_d, \Sigma)}\lp Y > \sqrt{\frac{N_{t, i}}{c(t-1, \delta)}} (\mu_i - \wh \mu_{t, i})\rp \geq 	(2\pi)^{-d/2} \det(\bar \sigma\Sigma)^{-1/2}\exp\lp - \frac12 \| \mu_i - \wh \mu_{t, i}\|_{\Sigma^{-1}}^2 \rp \prod_{c \in [d]}R(e_c^\T V^{-1/2}\Sigma^{-1} u_t), 
\end{eqnarray*}
with $u_t := \sqrt{\frac{N_{t, i}}{c(t-1, \delta)}} (\mu_i - \wh \mu_{t, i})$, therefore, letting 
$$ (**) := \bP_{Y \sim \cN(0_d, \Sigma)}\lp Y > \sqrt{\frac{N_{t, i}}{c(t-1, \delta)}} (\mu_i - \wh \mu_{t, i}) \rp \prod_{j \in \wh S_t} \bP_{X \sim \cN(0, 1)}\lp X>  \sqrt{\frac{N_{t, j}}{c(t-1, \delta)}}  \| \mu_{j} - \wh \mu_{t,j}\|_{\Sigma^{-1}}\rp $$ we have 
\begin{align*}
	(**) \geq\det(\bar \sigma\Sigma)^{-1/2}\exp\lp - \sum_{ k \in \{ i\} \cup \wh S_t}\frac12 \| \mu_k - \wh \mu_{t, k}\|_{\Sigma^{-1}}^2 \rp \prod_{j \in \wh S_t} \wt R\lp \sqrt{\frac{N_{t, j}}{c(t-1, \delta)}}  \| \mu_{j} - \wh \mu_{t,j}\|_{\Sigma^{-1}}\rp\prod_{c \in [d]}\wt R(e_c^\T V^{-1/2}\Sigma^{-1} u_t), 
\end{align*}
then note that by Lemma~\ref{lem:mills_properties}, 
we have 
\begin{align*}
	\prod_{j \in \wh S_t} \wt R\lp \sqrt{\frac{N_{t, j}}{c(t-1, \delta)}}  \| \mu_{j} - \wh \mu_{t,j}\|_{\Sigma^{-1}}\rp\prod_{c \in [d]}\wt R(e_c^\T V^{-1/2}\Sigma^{-1} u_t /d ) \geq \wt R\lp \frac{1_d^\T V^{-1/2}\Sigma^{-1} u_t  + 1_{\lvert \wh S_t \rvert}^\T h_t}{d + \lvert \wh S_t \rvert}\rp^{d + \lvert \wh S_t \rvert}
\end{align*}
with $h_t := \lp \sqrt{\frac{N_{t, j}}{c(t-1, \delta)}}  \| \mu_{j} - \wh \mu_{t,j}\|_{\Sigma^{-1}}\rp_{j \in \lvert \wh S_t \rvert}$ and by Cauchy-Swchartz inequality, replacing $V = \bar \sigma I_d$, 
\begin{eqnarray*}
	1_d^\T V^{-1/2}\Sigma^{-1} u_t  + 1_{\lvert \wh S_t \rvert}^\T h_t &\leq& \sqrt{\|1_d\|_{(\bar \sigma \Sigma)^{-1}}^2 + \| 1_{\lvert \wh S_t \rvert}\|^2} \sqrt{\|u_t\|^2_{\Sigma^{-1}} + \| h_t\|^2} \\
	&=& \sqrt{d_{\tiny \Sigma} + \lvert \wh S_t \rvert} \lp \sqrt{ \frac{1}{c(t-1, \delta)} \sum_{k \in \{ i\} \cup \wh S_t} N_{k,t} \| \mu_{k} - \wh \mu_{t,k}\|_{\Sigma^{-1}}^2}\rp \\
	&\leq& \sqrt{d_{\tiny \Sigma} + \lvert \wh S_t \rvert} \sqrt{\frac{2 \beta(t-1, \delta/2)}{c(t-1, \delta)}}
\end{eqnarray*}
where the last inequality follows on the event $\cE_{t, \delta/2}$. Comibining these displays with $ c(t, \delta) := \frac{\beta(t, \delta/2)}{ \log(1/\delta)}$, we have on the event $\cE_{t, \delta/2}$, 
\begin{eqnarray*}
	(**) &\geq& \det(\bar \sigma\Sigma)^{-1/2}\delta \wt R \lp \frac{\sqrt{d_{\tiny \Sigma} + \lvert \wh S_t \rvert}}{ d + \lvert \wh S_t \rvert} \sqrt{{2 \log(1/\delta)}} \rp \\
	&=& \det(\bar \sigma\Sigma)^{-1/2}\delta r(\delta ^{\frac{d_{\tiny \Sigma} + \lvert \wh S_t \rvert}{ d + \lvert \wh S_t \rvert}},d + \lvert \wh S_t \rvert). 
\end{eqnarray*}
Following the same reasoning, it is simple to show that under the event $\cE_{t, \delta/2}$, we have 
\begin{eqnarray*}
 \bP_{Y\sim \cN(0_d, \Sigma)} \lp Y \geq \sqrt{\lp\frac1{N_{t, i}}\! + \!\frac1{N_{t, j}}\rp^{-1}\frac 1{c(t-1, \delta)}} \lp (\wh \mu_{t, i} - \wh \mu_{t, j}) - (\mu_i - \mu_j)\rp\rp &\geq& \det(\bar \sigma\Sigma)^{-1/2}\delta \wt R(\sqrt{2\log(1/\delta)d_{\tiny \Sigma}}/d)\\
 &=& \det(\bar \sigma\Sigma)^{-1/2}\delta r(\delta^{\frac{d_{\tiny \Sigma}}{d}}, d).
\end{eqnarray*}
All put together, we have 
\begin{eqnarray*}
	\ind_{\cE_{t, \delta/2}}\ind_{\wh S_t \neq S^\star} \bP_{\wh \Pi_t \mid \cF_t}( \alt(\wh S_t) ) &\geq& \delta \det(\bar \sigma\Sigma)^{-1/2} \max\lb \ind_{\bm \mu \in \alt^-(\wh S_t) } r(\delta^{\frac{d_{\tiny \Sigma}}{d}}, d), \ind_{(\bm \mu \in \alt^+(\wh S_t))} r(\delta ^{\frac{d_{\tiny \Sigma} + \lvert \wh S_t \rvert}{ d + \lvert \wh S_t \rvert}},d + \lvert \wh S_t \rvert)\rb \\
	&\geq& \delta \det(\bar \sigma\Sigma)^{-1/2} \min (r(\delta^{\frac{d_{\tiny \Sigma}}{d}}, d), r(\delta ^{\frac{d_{\tiny \Sigma} + \lvert \wh S_t \rvert}{ d + \lvert \wh S_t \rvert}},d + \lvert \wh S_t \rvert))\end{eqnarray*}
which follows similarly to the $\Sigma$ diagonal case, since $ \wh S_t \neq S^\star \impl \bm \mu \in \alt(\wh S_t ) = \alt^+(\wh S_t) \cup \alt^-(\wh S_t)$. 
Remark that when $\Sigma = \sigma I_d$, we recover the 
results of section~\ref{ssec_ind}. 

\begin{proof}[Proof of Lemma~\ref{lem:mills_cov}]
	$\Sigma$ is a $d\times d$ covariance matrix and $f_\Sigma$ is the density function of $\cN(0_d, \Sigma)$. We let $\mathcal{R}_\Sigma$ denote the Mills ratio of the distribution $\cN(0_d, \Sigma)$, which is defined for a vector $x \in \bR^d$ as 
\begin{equation}
\label{eq:mult_mills}
	 \mathcal{R}_\Sigma(x) := \frac{\bP\lp \cN(0_d, \Sigma) \geq x\rp}{f_\Sigma(x)} 
\end{equation}
where for two vectors $x, y$, the notation $x\leq y$ should be understood component-wise.  Expanding \eqref{eq:mult_mills} gives 
\begin{eqnarray*}
	 \mathcal{R}_\Sigma(x) = \exp(\|x \|_{\Sigma^{-1}}^2/2 ) \int_{(x,  \infty)} \exp\lp - \| u \|_{\Sigma^{-1}}^2 /2 \rp du,
\end{eqnarray*}
with $(x, \infty) = (x(1), \infty) \times \dots \times (x(d), \infty)$. By a simple translation, it follows that 
$$  \mathcal{R}_\Sigma(x)  = \int_{(0_d, \infty)} \exp\lp -x^\T \Sigma^{-1} u - \| u\|^2_{\Sigma^{-1}}/2\rp du.$$
Since $(V - \Sigma^{-1})$ is psd, we have 
\begin{eqnarray*}
	 \mathcal{R}_\Sigma(x) &\geq& \int_{(0_d, \infty)} \exp\lp - x^\T \Sigma^{-1} u -  u^\T V u/2\rp du \\
	 &=& \det(V)^{-1/2} \int_{(0_d, \infty)} \exp\lp - x^\T \Sigma^{-1} V^{-1/2}u -  u^\T u/2\rp du
\end{eqnarray*}
which follows by change of variable.  Thus, letting $z = V^{-1/2}\Sigma^{-1} x$, 
we have 
\begin{eqnarray*}
	 \mathcal{R}_\Sigma(x) &\geq& \det(V)^{-1/2} \int_{(0_d, \bm \infty)} \exp\lp - u^\T z -  u^\T u/2\rp du \\
	 &=& \det(V)^{-1/2} \prod_{c\in [d]} R(z(c))
\end{eqnarray*}
then using Lemma~\ref{lem:mills_properties}, it follows that 
\begin{eqnarray*}
	 \mathcal{R}_{\Sigma}(x) &\geq& \det(V)^{-1/2} R\lp 1_d^\T V^{-1/2}\Sigma^{-1} x /d \rp^d  \\
	 &\geq& \det(V)^{-1/2} R\lp \bm \| 1_d \|_{V^{-1/2}\Sigma^{-1} V^{-1/2}}\| x \|_{\Sigma^{-1}} /d  \rp^d.
\end{eqnarray*}
\end{proof}

\subsection{Structured Setting: Transductive Linear BAI ($d = 1$ and $\Theta = \bR^{h}$)}
\label{app:ssec_correctness_structured}

When $d=1$, we have $S^\star = \{z^\star\}$ with $z^\star = z^\star(\bm \theta)$ and $\widehat S_t = \{\hat z_t\}$ where $\hat z_t = z^\star(\bm{\hat{\theta}}_t)$ where $z^\star(\bm \lambda) \eqdef \argmax_{z \in \cZ} \theta\transpose  z $.
Let $\sigma^2 = \Sigma$ be the variance of the unique objective.

Let $\Pi_t = \cN(0_{h},\bm \Sigma_t)$ where $\bm \Sigma_t = \sigma^2 (V_{\bm N_{t}} + \xi I_h)^{-1}$.
For all $m \in [M(t-1,\delta)]$, let $\bm v_{t}^{m} \sim \Pi_t$ and $\bm \theta^{m}_{t} = \bm{\hat{\theta}}_{t} + \sqrt{c(t-1,\delta)} \bm v_{t}^{m}$.
Therefore, using computation as above and $1-x \le e^{-x}$, we obtain
\begin{align*}
	&\bP_{\bm{\nu}}\lp \tau^{\textrm{PS}}_\delta < + \infty, \hat z_{\tau^{\textrm{PS}}_\delta} \ne z^\star \rp \le \delta/2 + \\
	&\qquad \bE_{\bm{\nu}}\lsb \sum_{t\geq 1} \indi{\cE_{t,\delta/2}\cap \{\hat z_t \neq z^\star\}} \exp \left( - M(t-1, \delta) \bP_{\bm v_{t}^{m} \sim \Pi_t}\lp  \exists z \ne \hat z_t , \: (\bm{\hat{\theta}}_{t} + \sqrt{c(t-1,\delta)} \bm v_{t}^{m}) \transpose  (z - \hat z_t ) > 0  \rp)\right) \rsb
\end{align*}
For all $m \in [M(t-1,\delta)]$, let $X_{t}^{m} = \frac{(\bm v_{t}^{m}) \transpose  (z^\star - \hat z_t )}{\|\hat z_t - z^\star\|_{\Sigma_{t}} }$.
Then, we have $X_{t}^{m} \sim \cN(0,1)$ since $\bm v_{t}^{m} \sim \Pi_t$.
Under $\cE_{t,\delta/2}$, we have $\|\bm{\hat{\theta}}_{t} -\bm \theta \|_{\Sigma_{t}^{-1}} \le \sqrt{2 \beta(t-1,\delta/2)}$.
Using that $z^\star \ne \hat z_t$ and $\bm \theta \transpose  (z^\star - \hat z_t ) > 0$, we obtain
\begin{align*}
	&\exists z \ne \hat z_t , \quad (\bm{\hat{\theta}}_{t} + \sqrt{c(t-1,\delta)} \bm v_{t}^{m}) \transpose  (z - \hat z_t ) > 0 \\
	\impliedby \quad &\sqrt{c(t-1,\delta)} (\bm v_{t}^{m}) \transpose  (z^\star - \hat z_t ) > -\bm \theta \transpose  (z^\star - \hat z_t ) + (\bm{\hat{\theta}}_{t} -\bm \theta ) \transpose  (\hat z_t - z^\star  ) \\
	\impliedby \quad &\sqrt{c(t-1,\delta)} \frac{(\bm v_{t}^{m}) \transpose  (z^\star - \hat z_t )}{\|\hat z_t - z^\star\|_{\Sigma_{t}} } \ge \|\bm{\hat{\theta}}_{t} -\bm \theta \|_{\Sigma_{t}^{-1}} \\
	\impliedby \quad & X_{t}^{m} \ge \sqrt{\frac{2 \beta(t-1,\delta/2)}{c(t-1,\delta)}} \: .
\end{align*}
In the following, we consider $c(t,\delta) = \frac{\beta(t, \frac{\delta}{2})}{\log \frac{1}{\delta}}$ and $M(t,\delta) = \left\lceil \frac{\log(2 t^s \zeta(s) /\delta)}{\delta q(\delta)} \right\rceil $ where $q(\delta) = R(\sqrt{2\log(1/\delta)})/\sqrt{2\pi} = \bP_{X \sim \cN(0,1)} ( X > \sqrt{2 \log(1/\delta)})/\delta$.
Therefore, we have shown that
\begin{align*}
	\bP_{\bm{\nu}}\lp \tau^{\textrm{PS}}_\delta < + \infty, \hat z_{\tau^{\textrm{PS}}_\delta} \ne z^\star \rp &\le \delta/2 + \sum_{t\ge 1}  \exp \left( - \left\lceil \frac{\log(2 (t-1)^s \zeta(s) /\delta)}{ \bP_{X \sim \cN(0,1)} ( X > \sqrt{2 \log(1/\delta)})} \right\rceil  \bP_{X \sim \cN(0,1)} ( X > \sqrt{2 \log(1/\delta)})  \right)  \: .
\end{align*}
Using that $\lceil x\rceil \ge x$ and $\sum_{t \ge 1} 1/t^s = \zeta(s)$ concludes the proof of $\delta$-correctness, i.e. $\bP_{\bm{\nu}}\lp \tau^{\textrm{PS}}_\delta < + \infty, \hat z_{\tau^{\textrm{PS}}_\delta} \ne z^\star \rp\le \delta$.


\section{SADDLE-POINT CONVERGENCE} 
\label{app:saddle_point_convergence}

In this section, we prove the following saddle-point convergence theorem. 
\begin{theorem}
\label{thm:sdp_convergence}
There exists events $(\Xi_t)_{t\geq 1}$ and $t_3 \in \bN$ such that for all $t\geq t_3$,   
\begin{equation}
	2\text{GLR}(t) \geq t \max_{\w\in \D}\inf_{\bm \lambda \in \Theta \cap \alt(S^\star)} \left \| \vv(\bm \theta - \bm \lambda) \right\|_{\wt V_{\w}}^2  - 2f(t), 
\end{equation} with $f(t) =_{\infty} o(t)$ and  $\bP_{\bm \nu}(\Xi_t) \geq 1-5/t^2$. In particular for $t\geq t_3$, $\wh S_t = S^*(\bm \theta)$. 
\end{theorem}

To derive this result, we will first define some concentration events below : 
\begin{itemize}
	\item $\Xi_{1, t} := \lb  \forall s \leq t, \left\| \vv(\wh{\bm\theta}_s - \bm\theta) \right\|_{\wt V_s}^2 \leq \beta(t, 1/t^2) =: f_1(t) \rb$  where $\beta(t, \delta)$ is defined as in  Lemma~\ref{lem:concentration_threshold}
\end{itemize}

We also define the following (high-probability) event: 
\begin{itemize}
	\item $\Xi_{2, t}$ as in Lemma~\ref{lem:forced_exp}. \end{itemize}

\subsection{Convergence of the Alternatives}
\label{ssec:conv_alt} 
We study the time after which the Pareto set is well estimated. The forced-exploration weight vector $\bm \w_{\text{exp}}$ allows to define a deterministic time after which the Pareto set is well estimated on a good event. We prove the following result. 
\begin{lemma} 
\label{lem:conv_alt}
There exists $t_1 \in \bN$ such that for all $t\geq t_1$, if the event $\Xi_{1, t} \cap \Xi_{2, t}$ holds then $\wh S_t = S^\star(\bm \theta)$. 
\end{lemma}

\begin{proof}
Let $e_1, \dots, e_d$ denote the canonical $\bR^d$ basis and let $a\in \cA$.  And assume in this section that $\Xi_{1, t} \cap \Xi_{2,t}$ holds with $t\geq t_0(\alpha)$ (as in Lemma~\ref{lem:forced_exp}). For any  $c \in [d]$  we have by Cauchy-Schwartz inequality,   
\begin{eqnarray*}
	\left \lvert (e_i \otimes a)^\T\vv(\wh{\bm\theta}_t - \bm\theta) \right \rvert &\leq& \left \|  \Sigma^{1/2} \otimes V_t^{-1/2} (e_c \otimes a) \right\|_2  \left\| \vv(\wh{\bm\theta}_t - \bm\theta)\right \|_{\Sigma^{-1} \otimes V_t} \\
	&\leq&  \left \|  \Sigma^{1/2} \otimes V_t^{-1/2} (e_i \otimes a) \right\|_2 \sqrt{f_1(t)}, 
\end{eqnarray*}
on the event $\Xi_{1, t}$. From the inequality above,  we have for all $a\in \cA$, 
\begin{eqnarray*}
	\| (I_d \otimes a)^\T \vv(\wh{\bm\theta}_t - \bm\theta) \|_2^2 f_1(t)^{-1} &\leq& \sum_{c\in [d]} (e_c \otimes a)^\T (\Sigma \otimes V_t^{-1}) (e_c \otimes a)\\
	&=& \sum_{c\in [d]}\Tr{\lp (e_c \otimes a)^\T (\Sigma \otimes V_t^{-1}) (e_c \otimes a) \rp}\\
	&=& \Tr{\lp   (\Sigma \otimes V_t^{-1}) \sum_{i\in [d]} (e_c \otimes a) (e_c \otimes a)^\T \rp}\\
	&=& \Tr{\lp   (\Sigma \otimes V_t^{-1}) \sum_{i\in [d]} (e_c e_c^T) \otimes (aa^T) \rp} \\
	&=& \Tr{\lp   (\Sigma \otimes V_t^{-1}) (I_d \otimes (aa^T)) \rp} =  \Tr{\lp   \Sigma \otimes V_t^{-1}aa^T \rp}\\
	&=& \Tr(\Sigma) a^\T V_t^{-1} a 
\end{eqnarray*}
which follows from $\Tr(A\otimes B) = \Tr(A) \Tr(B)$. Then Lemma~\ref{lem:forced_exp} on forced exploration ensures that for any $a$, $\| a \|_{V_{t}^{-1}}^2 \leq 2 t^{{\alpha-1}} \| a \|_{V(\w_\text{exp})^{-1}}^2$, thus 
$$ \| (I_d \otimes a)^\T \vv(\wh{\bm\theta}_t - \bm\theta)\|_2^2 f_1(t)^{-1} \leq 2 t^{\alpha - 1} \Tr(\Sigma) \| a \|_{V(\w_\text{exp})^{-1}}^2$$
Now, letting $\vmu_{a}$ denote the mean vector of arm $a$, remark that by writing the linear model we have 
$\vmu_{a} = (I_d \otimes a)^\T \vv(\bm\theta)$, so that letting $\wh\vmu_{t, a} = (I_d \otimes a)^\T \vv(\wh{\bm\theta}_t) $ the estimated mean of arm $a$ at time $t$, we have proved that 
\begin{equation}
	\left \| \mu_{t, a} - \mu_{a} \right\|_2 \leq \| a \|_{V(\w_\text{exp})^{-1}} \sqrt{2 t^{\alpha - 1}\Tr(\Sigma) f_1(t) }. 
\end{equation}
For any two arms $a, a'$, let $\M(a,a'):= \| (\mu_{a} - \mu_{a'})_+\|_2$ and $\M(a, a'; t) := \| (\wh\mu_{t, a} - \wh\mu_{t, a'})_+\|_2$ where $(w)_+ := \max(w, 0)$ for $w\in \bR$ and it is applied component-wise for a vector. It is simple to prove that 
\begin{eqnarray*}
	\left\lvert \M(a,a') - \M(a, a'; t) \right\rvert &\leq& \left \| \wh\mu_{t, a} - \mu_{a} \right\|_2  + \left \| \wh\mu_{t, a'} - \mu_{a'} \right\|_2 \\ &\leq& 2\sqrt{L_*  t^{\alpha - 1}\Tr(\Sigma) f_1(t) }\end{eqnarray*}
where we recall $L_*:= 2\max_{a\in \cA} \| a \|_{V(\w_\text{exp})^{-1}}^2 $. Therefore, if $t$ is such that 
$$ 2\sqrt{L_*  t^{\alpha - 1}\Tr(\Sigma) f_1(t) } < \Delta_1 := \min_{a \in S^\star(\bm\theta)}\min_{a' \in \cZ  \backslash\{ a\}}  \M(a,a') $$ then for all arms $x\in S^\star(\bm \theta)$ and $a'\neq a \in \cZ$, we will have $\M(a, a'; t)>0$, i.e. $a$ is not dominated by $a'$ : it is empirically optimal. 

Similarly defining $\m(a,a'):= \min_{c \in [d]} (\mu_{a}(c) - \mu_{a'}(c))$ and $\m(a,a'; t):= \min_{c \in [d]} \lp \wh\mu_{t, a}(c) - \wh\mu_{t, a'}(c)\rp$, one can also prove that  
\begin{eqnarray*} 
	\left\lvert \m(a,a') - \m(a, a'; t) \right\rvert &\leq& \left \| \wh \mu_{t, a} - \mu_{a} \right\|_2  + \left \| \wh \mu_{t, a'} - \mu_{a'} \right\|_2 \\ 	&\leq& 2\sqrt{L_*  t^{\alpha - 1}\Tr(\Sigma) f_1(t) }
	\end{eqnarray*} 
	and observe that, if $\m(a,a';t)>0$ then $a$ is dominated by $a'$, so that when 
$$ 2\sqrt{L_*  t^{\alpha - 1}\Tr(\Sigma) f_1(t) } < \Delta_2 := \min_{a \in \cZ \backslash S^\star(\bm \theta)}\max_{a' \in S^\star(\bm \theta)}  \m(a,a'),$$
any arm in $\cZ \backslash S^\star(\bm \theta)$ will be empirically dominated.  
Therefore, we define a gap $\Delta_{\min}$ under which the Pareto set is well estimated as 
$$ \Delta_{\min} := \min (\Delta_1, \Delta_2)$$
and we define 
\begin{equation}
	\label{eq:def-t1}
	t_1 := \inf \lb n \geq t_0(\alpha) : \forall t\geq n, \frac{\sqrt{L_*  t^{\alpha - 1}\Tr(\Sigma) f_1(t) }}{\Delta_{\min}}   < \frac12  \rb, 
\end{equation}
which is well-defined as $f_1(t)$ is  at most logarithmic. Note that the gap $\Delta_{\min}$ is related to the gaps that appear in finite-time analysis of PSI algorithms \citep{auer_pareto_2016, kone2023adaptive}. 
\end{proof}

\subsection{Learning on Unbounded Sets}
\label{ssec:learning_set}
In this section, we properly define $\Theta_t$, $\eta_t$, and the results that allow this definition. Most game-based analyses will assume the parameter space is bounded \cite{degenne2019non, degenne2020gamification, zhaoqi_peps}. In the (transductive) linear setting, this assumption is also used to control the self-normalized deviations of the least-square estimate.  

In the unstructured setting, we show that this assumption can be relaxed by learning the norm of $\bm \mu$ online. This relies on the following result, whose proof can be found in Appendix~\ref{app:technicalities}. 
\begin{restatable}{lemma}{infBounded}
\label{lem:inf_bounded}
	Let $\w \in \bR_{+}^K$. For any $\bm \mu' := (\mu'_1 \dots \mu'_K)^\T  \in \bR^{K \times d}$, the following statement holds in the unstructured setting   
	\begin{equation}
		\inf_{\bm \lambda \in \alt(S^\star(\bm \mu'))} \left\| \vv(\bm \lambda - \bm \mu') \right\|^2_{\Sigma^{-1} \otimes \diag(\w)} = \inf_{\lambda \in \alt(S^\star(\bm \mu')) \cap \{ \bm \lambda \mid\;  \max_i \| \mu'_i - \lambda_i\|_{\Sigma^{-1}} <\epsilon\}} \left\| \vv(\bm \lambda - \bm \mu') \right\|^2_{\Sigma^{-1} \otimes \diag(\w)},  
	\end{equation} 
	where $\epsilon := \max(2\max_{i\notin S} \max_{j\in S}\|\mu'_i - \mu'_j\|_{\Sigma^{-1}}, \max_{i,j \in S^2} \| \mu'_i - \mu'_j\|_{\Sigma^{-1}})$, and $S = S^\star(\bm \mu')$. 
\end{restatable}
Intuitively, it shows that for any parameter, the best response lies inside a compact region of the alternative space. We then use the procedure in Algorithm~\ref{alg:est_theta} to define $\Theta_t$ and $\eta_t$. 
\begin{algorithm}[hbt]
	\SetAlgoLined
    \label{algo:est_theta}
		{\bfseries Input: } $B_1 = \infty, C_1, p_1 = \infty$, confidence threshold $f_1$, covariance matrix $\Sigma$
		
		\For{$t = 1,\dots,$}{
		\tcp{define ucb function}
		$h^t \gets  a \mapsto \sqrt{f_1(t)} \| a\|_{V_t^{-1}} + \| \vv(\wh{\bm \mu_t})\|_{\Sigma^{-1} \otimes aa^\T}$\;		
		\tcp{compute upper confidence bound}
		$U_t \gets \max\lb \max\{h^t({a_i- a_j}) \mid (i, j) \in \wh S_t^2 \}, 2\max\{h^t({a_i- a_j}) \mid (i,j)\in (\wh S_t^c, S_t) \rb$ \; 	
			
		$u_t \gets \max_{a \in \cA} h^t(a)$
		
		\eIf{ $U_t \leq 2 p_{t}$}{
		
		$p_{t+1} \gets U_t$ 
		}
		{$p_{t+1} \gets p_{t}$}

		\eIf{$u_t \leq 2 C_{t}$}{
		 
		 $C_{t+1} \gets u_t $ \;  
		}{
		  $C_{t+1} \gets C_{t}$ \; 
		}
		
		$B_{t+1} \gets  C_{t+1} + p_{t+1}$

		$\eta_{t+1} \gets \frac{1}{8B_{t+1}^2}$
		
		$\Theta_{t+1} = \lb \bm \lambda:= (\lambda_1,\dots,\lambda_K) \mid \max_{a\in \cA} \|\lambda_a\|_{\Sigma^{-1}} < B_{t+1} \rb$}
	\caption{\protect\hypertarget{estTheta}{Estimate and Halve}}\label{alg:est_theta}
\end{algorithm}

The procedure initializes a bound $B_1$, which is updated whenever the confidence bound is halved. We then show that both $B_t$ and $\eta_t$ will stabilize on a good event. That is after a time $t_2 \in \bN$, we will have $B_{t} = B_{t_2}$ for all $t\geq t_2$. 

\begin{lemma}
\label{lem:uns_contr}
In the unstructured setting, there exists $t_2 \in \bN$ (as defined below) such that for all $t\geq t_2$, if the event $\Xi_{1, t} \cap \Xi_{2, t}$ holds then $\eta_t = \eta_{t_2}$ and $\Theta_t = \Theta_{t_2}$. 
\end{lemma}
\begin{proof}
	We will first show that when $\Xi_{1, t}$  holds, $U_t$ is an upper confidence bound on the quantity $\epsilon$ of Lemma~\ref{lem:inf_bounded}. Let $t\geq t_1$ and assume 
$\Xi_{1, t}$ holds. By Lemma~\ref{lem:conv_alt}, we have $\wh S_t = S^\star(\bm \mu)$. 	Note that by keeping the notation of the linear setting, we have for all $i,j\in [K]$, 
\begin{eqnarray*}
\| \mu_i - \mu_j \|_{\Sigma^{-1}}^2  &=& \| I_d \otimes( a_i - a_j)\vv(\bm \mu) \|_{\Sigma^{-1}}^2	\\
&=& \vv(\bm \mu)^\T (I_d \otimes( a_i - a_j)^\T) \Sigma^{-1}(I_d \otimes( a_i - a_j))\vv(\bm \mu) \\
&=& \vv(\bm \mu)^\T \Sigma^{-1} \otimes (a_i - a_j)( a_i - a_j)^\T\vv(\bm \mu),
\end{eqnarray*}
which follows by properties of Kronecker product. Let $\wt a_{i,j} = \Sigma^{-1/2} \otimes (a_i - a_j)$, so we have $\| \mu_i - \mu_j \|_{\Sigma^{-1}}^2 = \| \vv(\bm \mu) \|_{\wt a_{i,j} \wt a_{i,j}^\T}^2$. Therefore, we should derive upper bound on $\| \vv(\bm \mu) \|_{\wt a_{i,j} \wt a_{i,j}^\T}$. We have by triangle inequality 
\begin{equation}
\label{eq:al-mo-st}
	\| \vv(\bm \mu) \|_{\wt a_{i,j} \wt a_{i,j}^\T} \leq  \| \vv(\wh{\bm\mu}_t) \|_{\wt a_{i,j} \wt a_{i,j}^\T} + \| \vv(\bm \mu - \wh{\bm\mu}_t) \|_{\wt a_{i,j} \wt a_{i,j}^\T},
\end{equation}
then we rewrite the RHS.  Note that 
\begin{eqnarray}
	\| \vv(\bm \mu - \wh{\bm\mu}_t) \|_{\wt a_{i,j} \wt a_{i,j}^\T}^2  &=& \vv(\bm \mu - \wh{\bm\mu}_t)^\T \wt V_{t}^{1/2} \wt V_{t}^{-1/2}(\wt a_{i,j} \wt a_{i,j}^\T)\vv(\bm \mu - \wh{\bm\mu}_t) \nonumber\\
	&\leq& \| \vv(\bm \mu - \wh{\bm\mu}_t)\|_{\wt V_t} \| \vv(\bm \mu - \wh{\bm\mu}_t)\|_{(\wt a_{i,j} \wt a_{i,j}^\T)\wt V_{t}^{-1} (\wt a_{i,j} \wt a_{i,j}^\T)}, \label{eq:ez-ae-po}
\end{eqnarray}
which follows by Cauchy-Swchartz inequality, and we have 
\begin{eqnarray*}
	(\wt a_{i,j} \wt a_{i,j}^\T)\wt V_{t}^{-1} (\wt a_{i,j} \wt a_{i,j}^\T) &=&  (\Sigma^{-1} \otimes (a_{i,j}a_{i,j}^\T ) ) (\Sigma \otimes V_t^{-1} )(\Sigma^{-1} \otimes (a_{i,j}a_{i,j}^\T ) )\\
	&=& (a_{i,j}^\T V_t^{-1} a_{i,j})\Sigma^{-1} \otimes (a_{i,j}a_{i,j}^\T ) \\
	&=& \| a_{i,j}\|_{V_t^{-1}}^2 \wt a_{i,j}\wt a_{i,j}^\T
\end{eqnarray*}
and combining with \eqref{eq:ez-ae-po} yields 
$$ \| \vv(\bm \mu - \wh{\bm\mu}_t) \|_{\wt a_{i,j} \wt a_{i,j}^\T}^2 \leq \| a_{i,j}\|_{V_t^{-1}} \| \vv(\bm \mu - \wh{\bm\mu}_t)\|_{\wt V_t} \| \vv(\bm \mu - \wh{\bm\mu}_t) \|_{\wt a_{i,j} \wt a_{i,j}^\T},$$ that is 
$\| \vv(\bm \mu - \wh{\bm\mu}_t) \|_{\wt a_{i,j} \wt a_{i,j}^\T} \leq \| a_{i,j}\|_{V_t^{-1}} \| \vv(\bm \mu - \wh{\bm\mu}_t)\|_{\wt V_t}$ and plugging back into \eqref{eq:al-mo-st}, yields 
$$ \| \vv(\bm \mu) \|_{\wt a_{i,j} \wt a_{i,j}^\T} \leq  h^t({a_i, a_j}) := \| \vv(\wh{\bm\mu}_t) \|_{\wt a_{i,j} \wt a_{i,j}^\T} + \| a_{i,j}\|_{V_t^{-1}} \sqrt{f_1(t)},$$
on the event $\Xi_{1, t}$. Therefore, as $\wh S_t = S^\star$, and by the line above, $U_t$ is an upper bound on $\epsilon := \max(2\max_{i\notin S^\star} \max_{j\in S^\star}\|\mu_i - \mu_j\|_{\Sigma^{-1}}, \max_{i,j \in (S^\star)^2} \| \mu_i - \mu_j\|_{\Sigma^{-1}})$ and we have (due to the terms $\| \vv(\wh{\bm\mu}_t) \|_{\wt a_{i,j} \wt a_{i,j}^\T}$ in its definition), $$U_t \geq \epsilon_t := \max\lp 2\max_{i\notin \wh S_t} \max_{j\in \wh S_t}\|\wh \mu_{t,i} - \wh \mu_{t,j}\|_{\Sigma^{-1}}, \max_{i,j \in \wh S_t ^2} \| \wh \mu_{t,i} - \wh \mu_{t,j}\|_{\Sigma^{-1}}\rp$$
We now justify that after some rounds, $B_t$ will not change anymore. For that, it suffices to show that both $p_t$ and $C_t$ will stabilize due to the estimate and halve procedure. To see this, note that 
\begin{eqnarray*}
	\| \vv(\wh{\bm\mu}_t) \|_{\wt a_{i,j} \wt a_{i,j}^\T} &\leq& \| \vv({\bm\mu}) \|_{\wt a_{i,j} \wt a_{i,j}^\T} + \| \vv(\bm \mu - \wh{\bm\mu}_t) \|_{\wt a_{i,j} \wt a_{i,j}^\T}  
	\\&\leq& \| \vv({\bm\mu}) \|_{\wt a_{i,j} \wt a_{i,j}^\T} + \| a_{i,j}\|_{V_t^{-1}} \| \vv(\bm \mu - \wh{\bm\mu}_t)\|_{\wt V_t} \quad \text{ cf above} \\
	&\leq& \| \vv({\bm\mu}) \|_{\wt a_{i,j} \wt a_{i,j}^\T}  + \| a_{i,j}\|_{V_t^{-1}} \sqrt{f_1(t)},
\end{eqnarray*}
on the event $\Xi_{1, t}$. Let $\cI := \{(i,j) \in (S^\star)^2, i\neq j\} \cup \{ (i,j) \in ((S^\star)^c, S)\}$ and define $t_2$ as 
\begin{equation}
\label{eq:def_t2}
	t_2 := \inf\lb n : \forall t\geq n : \max_{(i,j) \in \cI } \frac{\| a_{i,j}\|_{V_t^{-1}}\sqrt{f_1(t)}}{\| \vv({\bm\mu}) \|_{\wt a_{i,j} \wt a_{i,j}^\T}} < \frac12 \rb,
\end{equation}
which is well defined as the numerator is of order $\log(t)/t^{1-\alpha}$ due to forced exploration and the denominator is positive for $(i,j)\in \cI$. From $t_2$ we have $U_t \leq 2\epsilon$. 
Since on $\Xi_{1, t}$ $p_{t-1}$ is an upper bound on $\epsilon$ we have $\epsilon < p_{t-1}$, so $p_t$ is updated  at time $t_2$ as $U_{t_2}/2 \leq \epsilon$. For $t\geq t_2$ 
when $\Xi_{1, t}$ holds, since $p_{t_2} \leq 2\epsilon$, it cannot be halved again, so on $\Xi_{1, t}\cap Xi_{2, t}$, 
$$ \forall t\geq t_2, p_t = p_{t_2}\:,$$
and proceeding similarly for $C_t$ we prove that for $t\geq t_2$, 
$$ B_t = B_{t_2} \quad \eta_t = \eta_{t_2}, \quad \Theta_t = \Theta_{t_2}.$$
We recall that in the transductive linear setting $\Theta_t = \Theta$
 and \begin{equation}
 \label{eq:eta_trans}
 	\eta_t := \eta = \frac{1}{8 L_\cA^2 L_\Theta ^2}\:.
 \end{equation} 
\end{proof}

The following result derives from Lemma~\ref{lem:exp-concave} and the definition of $\Lambda_t$. 
\begin{corollary}
\label{cor:eta}
On the event $\Xi_{1, t} \cap \Xi_{2, t}$, for $t\geq t_2$, it holds that: in the unstructured setting or the transductive linear setting, for all $a\in \cA$,  $\bm \lambda \in \Lambda_t \mapsto \exp(-\eta_t \| \vv({\bm \theta} - \bm \lambda) \|_{\Sigma^{-1}\otimes a a^\T}^2)$ is concave. 
\end{corollary}

\begin{proof}
	In the (transductive) linear setting, we have $\eta_t := \eta = \frac{1}{8 L_\cA^2 L_\Theta ^2}$ then from the definition of $L_\cA,  L_\Theta$, the statement follows by Lemma~\ref{lem:exp-concave}. In the unstructured setting, we have by definition $\Lambda_t := \alt(\wh S_t) \cap \Theta_t$ with $\Theta_t := \lb \bm \lambda:= (\lambda_1,\dots,\lambda_K) \mid \max_{a} \|\lambda_a\|_{\Sigma^{-1}} < B_t \rb$ and $\eta_t := \frac{1}{8B_t^2}$. Note that in the unstructured setting we 
	have  $$ \| \vv({\bm \theta} - \bm\lambda)\|^2_{\Sigma^{-1} \otimes aa^T} = \| \mu_{a} - \lambda_a\|_{\Sigma^{-1}}^2,$$
	and, as  $2B_t$ is an upper confidence bound on $\max_{a\in \cA}\| \mu_{a} - \lambda_a\|_{\Sigma^{-1}}$, the conclusion follows again using Lemma~\ref{lem:exp-concave}. 
	
\end{proof}

\subsection{Analysis of the Sampling Rule}
\label{ssec:sampling_rule}
Letting $t_1, t_2$ as defined in \eqref{eq:def-t1}, \eqref{eq:def_t2} we let $t_3 = \max(t_1, t_2)$. When $t\geq t_3$ and $\Xi_{1, t} \cap \Xi_{2, t}$ holds, $\wh S_t = S^\star$ (Lemma~\ref{lem:conv_alt}) and $\Lambda_t := \Theta_t \cap \alt(S_t) = \Theta_t \cap \alt(S^\star)$. As $\Theta_t  \subset \Theta$ we have $\Lambda_t \subset \Theta \cap \alt(S^\star)$. In the sequel, assume $t>t_3$ and $\Xi_{1, t} \cap \Xi_{2, t}$ holds. As $t\geq t_3$, both $\eta_t$ and $\Lambda_t$ remains constant (Lemma~\ref{lem:uns_contr}) thus, in this section we may simply refer to them as  $\eta$ and $\Lambda$ respectively.
In particular, $\Lambda$ is bounded and we have 
\begin{align}
 \label{eq:def_B}
\Lambda := \Lambda_{t_3}, \eta := \eta_{t_3} \quad \text{and}\quad \max_{a \in \cA, \bm \lambda \in \Lambda} \|\vv(\bm \theta - \bm \lambda)\|_{\wt a \wt a^\T}^2 \leq B, \quad \text{with} \nonumber \\
	B := \begin{cases}
		4 \wt L_\cA^2 L_\Theta^2 & \text{(transductive linear setting)} \\
		(B_{t_3} +  \max_{a \in \cA } \| \mu_a \|_{\Sigma^{-1}})^2& \text{(unstructured setting)}. 
	\end{cases}
\end{align}

\paragraph{Iterative saddle point convergence} 
This section analyzes the regrets of the \emph{min} and \emph{max} learners. The proof of Theorem~\ref{thm:sdp_convergence} will follow. 
In this regard, we relate the regret of each learner to the saddle-point value by:  
\begin{eqnarray}
	\inf_{\bm\lambda \in \Lambda_{t-1}} \left\| \vv(\bm\lambda - \wh{\bm\theta}_{t-1}) \right\|_{\wt V_{t-1}}^2 
		&\geq& \sum_{s=t_3}^{t-1} \bE_{\bm \lambda \sim \pi_{s-1}} \lsb \left \| \vv(\wh{\bm\theta}_{s-1} - \bm\lambda )\right\|^2_{\wt a_s \wt a_s^\T} \rsb - r_1(t) \label{eq:beg_m_t}\\
	&\geq&  \sum_{s=t_3}^{t-1} \bE_{\bm \lambda \sim \pi_{s-1}, a \sim \wt \w_s} \lsb \left \| \vv(\wh{\bm\theta}_{s-1} - \bm\lambda) \right\|^2_{\wt a \wt a^\T} \rsb  - r_1(t) - m(t) \label{eq:end_m_t}\\
		&\geq& \max_{\bm \w \in \D} \sum_{s=t_3}^{t-1} \bE_{\bm \lambda \sim \pi_{s-1}, a \sim \w} \lsb \left \|\vv( \wh{\bm\theta}_{s-1} - \bm\lambda) \right\|^2_{\wt a \wt a^\T} \rsb  - r_1(t) - r_2(t) - m(t) \label{eq:beg_z_t}\\
	&\geq& \max_{\bm \w \in \D} \sum_{s=t_3}^{t-1} \bE_{\bm \lambda \sim \pi_{s-1}, a \sim \w} \lsb \left \| \vv(\bm \theta - \bm \lambda )\right\|^2_{\wt a \wt a^\T} \rsb  - r_1(t) - r_2(t) - m(t) - z(t) \label{eq:end_z_t}\\
	&\geq& \max_{\bm \w \in \D} \sum_{s=t_3}^{t-1} \inf_{p \in \cP(\Lambda_s)}\bE_{\bm \lambda \sim p, a \sim \w} \lsb \left \| \vv(\bm\theta - \bm\lambda) \right\|^2_{\wt a \wt a^\T} \rsb  - r_1(t) - r_2(t) - m(t) - z(t) \nonumber\\
	&=& (t-t_3) \max_{\w \in \D} \inf_{\bm \lambda \in \Theta \cap \alt(S^\star)} \left\| \vv(\bm\theta - \bm\lambda) \right\|_{\wt V_{\bm \w}}^2 - r_1(t) - r_2(t) - m(t)  - z(t), \nonumber
\end{eqnarray}
where the last inequality follows since $\Lambda_s \subset \Theta \cap \alt(S^\star)$ for $s\geq t_3$ when $\Xi_{1, t} \cap \Xi_{2, t}$ holds. In the decomposition above, $r_1(t)$ (resp. $r_2(t)$) is the regret of the \emph{min} (resp. \emph{max}) learner, $m(t)$ is a martingale concentration term and $z(t)$ is an approximation error. Each of these terms will be controlled in the following subsections.  In particular, we justify in the next paragraphs that they are sub-linear terms (i.e. $o_{t\rightarrow \infty}(t)$).  
\subsubsection{Bounding the approximation error : step \eqref{eq:beg_z_t} to \eqref{eq:end_z_t}}
 We note that by convexity of the squared norm we have 
\begin{eqnarray*} 
	\sum_{s=t_3}^{t-1}  \| \vv(\wh{\bm\theta}_{s-1} - \bm\lambda_s) \|^2_{\wt a \wt a ^\T} \geq \sum_{s=t_3}^{t-1}  \| \vv(\bm\theta - \bm\lambda_s) \|^2_{\wt a \wt a^\T} + 2\sum_{s=t_3}^{t-1} \vv(\wh{\bm\theta}_{s-1} - \bm\theta)^\T (\wt a \wt a ^\T) \vv(\bm\theta - \bm\lambda_s) 
\end{eqnarray*}
then, let us define $y_s := (\wt a \wt a^\T) \vv(\bm\theta - \bm\lambda_s)$ and observe that by Cauchy-Schwartz inequality, we have for $\bm \lambda_s \in \Lambda$, 
\begin{eqnarray*}
	\vv(\bm\theta - \wh{\bm\theta}_{s-1})^\T y_s &\leq& \| \vv(\bm\theta - \wh{\bm\theta}_{s-1})\|_{\wt V_{s-1}} \| y_s \|_{\wt V_{s-1}^{-1}} \\
	&=& \| \vv(\bm\theta - \wh{\bm\theta}_{s-1})\|_{\wt V_{s-1}} \sqrt{\vv(\bm\theta - \bm\lambda_s)^\T  (\wt a \wt a^\T) \wt V_{s-1}^{-1} (\wt a \wt a^\T) \vv(\bm\theta - \bm\lambda_s)}  \\ 
	&=& \| \vv(\bm\theta - \wh{\bm\theta}_{s-1})\|_{\wt V_{s-1}} \| a\|_{V_{s-1}^{-1}} \|\vv(\bm\theta - \bm\lambda_s) \|_{\wt a \wt a^\T} \quad \text{(Kronecker product)}\\ 
	&\leq& \| \vv(\bm \theta - \wh{\bm\theta}_{s-1})\|_{\wt V_{s-1}} \| a\| _{V_{s-1}^{-1}} \sqrt{B}\end{eqnarray*}
where we recall that $\wt a := \Sigma^{-1/2} \otimes a$ and $\wt V_s := \Sigma^{-1} \otimes V_s$. From the previous inequality, it follows that for any $\w \in \D$ fixed, we have 
\begin{eqnarray*}
	\sum_{s=t_3}^{t-1} \bE_{\bm \lambda \sim \pi_{s-1}, a \sim \w} \!\!\lsb  \| \vv(\wh{\bm\theta}_{s-1} - \bm\lambda) \|^2_{\wt a \wt a^\T}\!\! \rsb \!\!\!\!&\geq& \!\!\!\!\sum_{s=t_3}^{t-1} \bE_{\bm \lambda \sim \pi_{s-1}, a \sim \w} \lsb\!  \| \vv(\bm\theta - \bm \lambda) \|^2_{\wt a \wt a^\T} \!\rsb - 2\sqrt{B} \bE_{a \sim \w} \!\!\lsb \sum_{s=t_3}^{t-1} \| \vv(\bm\theta - \wh{\bm\theta}_{s-1})\|_{\wt V_{s-1}} \| a\|_{V_{s-1}^{-1}}\rsb.
\end{eqnarray*}

 We will bound the RHS of the equation and prove that it is sub-linear due to the tiny forced exploration. We define 
\begin{equation}
	z(t) := 2\sqrt{B}  \max_{\w \in \D} \bE_{a \sim \w} \lsb \sum_{s=t_3}^{t-1} \| \vv(\bm\theta - \wh{\bm\theta}_{s-1})\|_{\wt V_{s-1}} \| a\|_{V_{s-1}^{-1}}\rsb
\end{equation}
which we claim to be sub-linear as from Lemma~\ref{lem:forced_exp}, 
\begin{eqnarray*}
	\sum_{s=t_3}^{t-1} \| \vv(\bm\theta - \wh{\bm\theta}_{s-1})\|_{\wt V_{s-1}} \| a\|_{V_{s-1}^{-1}} &\leq& \sqrt{t f_1(t)} \lp L* \sum_{s=1}^t s^{\alpha - 1}\rp^{1/2}\\
	&\leq& \sqrt{L* t^{1+\alpha} f_1(t)/\alpha}, 
\end{eqnarray*}
which is sub-linear for $\alpha \in (0, 1)$ so 
\begin{equation}
\label{eq:def_ztt}
z(t) \leq 2 \sqrt{B L_* t^{1+\alpha} f_1(t)/\alpha}	
\end{equation}
is sub-linear. 
The next paragraph shows that $m(t)$ is a sub-linear term. 
\subsubsection{Martingale Concentration: step \eqref{eq:beg_m_t} to \eqref{eq:end_m_t}}
Let us define the $s$-indexed stochastic process $(U_s)$ as 
$$ U_s := \bE_{\bm \lambda \sim \pi_{s-1}} \lsb \| \vv(\wh{\bm\theta}_{s-1} - \bm\lambda)\|^2_{\wt a_s \wt a_s^\T} \rsb - \bE_{\bm \lambda \sim \pi_{s-1}, a \sim \wt \w_s} \lsb  \| \vv(\wh{\bm\theta}_{s-1} - \bm\lambda) \|^2_{\wt a \wt a^\T} \rsb.$$
We recall that both $\bm \w_s$ and $ \bE_{\bm\lambda \sim \pi_{s-1}} \lsb  \| \vv(\wh{\bm\theta}_{s-1} - \bm\lambda)\|^2_{\wt a \wt a^\T} \rsb$ (for fixed $\wt a$) are $\cF_{s-1}$-measurable, and $U_s$ is adapted to the filtration $(\cF_s)_{s\geq 1}$. Moroever, simple calculation yields 
$$ \bE[U_s \mid \cF_{s-1}] = 0,$$
i.e $(U_s)_s$ is $\cF$-martingale difference sequence and $\lvert U_s \rvert \leq f_2(s,t)$ on the event $\Xi_{1, t} \cap \Xi_{2, t}$ (bounded below).  Let us define the event 
\begin{equation}
\label{eq:evt_3}
\Xi_{3, t}:= \lp \sum_{s=1}^t U_s \leq m(t) := \sqrt{2 \log(t^2) \sum_{s=1}^t f_2(s, t)^2 }\rp	
\end{equation}
and observe that by Azuma's inequality it holds that $\bP(\Xi_{3, t}) \geq 1- 1/t^2$. It remains to show that on $\Xi_{3, t} \cap \Xi_{1, t}$, $m(t)$ can be bounded by a sub-linear term. Indeed, we have \begin{eqnarray*}
	\| \vv(\wh{\bm\theta}_s - \bm\lambda)\|_{\wt a \wt a^\T} &\leq&   \| \vv(\wh{\bm\theta}_s - \bm\theta) \|_{\wt a \wt a^\T} + \| \vv(\bm\theta - \bm\lambda) \|_{\wt a \wt a^\T} \\
	&\leq& \| \vv(\wh{\bm\theta}_s - \bm\theta )\|_{\wt V_s} \|a \|_{V_s^{-1}} + \sqrt{B}
\end{eqnarray*}
so that on $\Xi_{1, t}\cap \Xi_{2, t}$, $f_2(s,t) \leq f_1(t) \|a \|_{V_s^{-1}}^2 + B + 2\sqrt{B} \| \vv(\wh{\bm\theta}_s - \bm\theta)\|_{\wt V_s} \|a\|_{V_s^{-1}} $ then 
\begin{eqnarray*}
	f_2(s,t)^2 \leq \lp 2L_*f_1(t) s^{\alpha - 1} + B + 2\sqrt{L_*B f_1(t) s ^{\alpha - 1}} \rp ^2   
\end{eqnarray*}
so that 
$$ \sum_{s=1}^t f_2(s,t)^2 \leq \cO(f_1(t)^2 t^{2\alpha - 1} + f_1(t) t^{\alpha} + \sqrt{f_1(t)} t^{(\alpha+1)/2} + f_1(t) \sqrt{f_1(t)} t^{(3\alpha -1)/2}) $$
which is sub-linear, so that 
\begin{equation}
\label{eq:def_mt}
m(t) = \cO(\sqrt{f_1(t)^2 t \log(t^2)})\:.
\end{equation}
\subsubsection{Regret of the \emph{min} learner} We analyze the regret of the min player which akin to continuous exponential weights. The goal is to show that we can have a sub-linear regret with a constant learning rate for the min learner. In this section we show the following guarantees for the min learner. 
\begin{lemma}
\label{lem:min_learner}
	For any $t\geq t_3$, the following holds on $\Xi_{t}$: 
	$$ \sum_{s=t_3+1}^t \bE_{\bm\lambda \sim \pi_{s-1}}\lsb \| \vv(\wh {\bm\theta}_{s-1} - \bm\lambda)\|_{\wt a_s \wt a_s^\T}^2 \rsb  - \inf_{\bm\lambda \in \Lambda } \|\vv(\bm\lambda - \wh{\bm\theta}_{t}) \|^2_{\wt V_t} \leq \cO\lp \sqrt{t \log^2(t)}\rp.$$
\end{lemma}

\begin{proof}
At round $s$, the strategy $\bm \lambda_s$ of the \emph{min} player is  drawn from the distribution $\pi_{s-1} := \cN(\vv(\bm \theta_s), \eta_s^{-1/2}\bm \Sigma_s; \vv(\Lambda_s))$ for which the density at $\lambda$ is proportional to 
$\exp( - \frac12 \eta  { \|\vv(\wh{\bm\theta}_{s}) - \lambda \|^2_{\wt V_s }} )$, with a normalizing constant $$ W_s := \int_{\vv(\Lambda_s)} \exp\lp - \frac12 \eta  { \| \vv(\wh{\bm \theta}_{s}) - \lambda \|^2_{\wt V_s }} \rp d\lambda.$$
We have for $s\geq t_3$, 
\begin{eqnarray*}
	\log\frac{W_{s}}{W_{s-1}}  &=& \log \frac{\int_{\vv(\Lambda)}\exp\lp - \frac{\eta \left \| \vv(\wh{\bm\theta}_{s}) - \lambda \right\|^2_{\wt V_s }}{2}  \rp d\lambda}{W_{s-1}}\\
	&=& \log \frac{\int_{\vv(\Lambda)}\exp\lp - \frac{\eta  \| \vv(\wh{\bm\theta}_{s})  - \lambda \|^2_{\wt V_{s} }}{2}  + \frac{\eta  \| \vv(\wh {\bm\theta}_{s-1})  -  \lambda \|^2_{\wt V_{s-1} }}{2}    - \frac{\eta \|\vv( \wh{\bm\theta}_{s-1}) - \lambda\|^2_{\wt V_{s-1} } }{2}   \rp d\lambda}{W_{s-1}} \\ 
	&=& \log \bE_{\bm\lambda \sim \pi_{s-1}} \lsb \exp\lp -\frac{\eta \| \vv(\wh{\bm\theta}_{s-1}  - \bm\lambda) \|^2_{\wt a_s \wt a_s^\T}}{2}   + \underbrace{\frac{\eta \| \vv(\wh {\bm\theta}_{s-1}  - \bm\lambda) \|^2_{\wt V_{s} }}{2}  - \frac{\eta \| \vv(\wh{\bm\theta}_{s}  - \bm\lambda) \|^2_{\wt V_{s}}}{2}}_{\Gamma_s(\bm\lambda)/2}\rp \rsb \\
	&\leq& \frac12\lp  \log\bE_{\bm\lambda \sim \pi_{s-1}} \lsb  \exp\lp -\eta  \| \vv(\wh{\bm\theta}_{s-1} - \bm\lambda) \|^2_{\wt  a_s \wt a_s^\T}  \rp \rsb  + \log\bE_{\bm\lambda \sim \pi_{s-1}} \lsb  \exp(\eta \Gamma_s(\bm\lambda)\rsb \rp \quad \text{(Cauchy-Schwartz)}
\end{eqnarray*}
At this step, the goal is to use the exp-concavity of the squared norm on bounded domains.  However we will not use it directly with $ \bm \lambda \mapsto  \| \vv(\wh {\bm\theta}_{s-1} - \bm\lambda) \|^2_{\wt  a_s \wt a_s^\T}$, but using concentration, we relate the latter to $\bm \lambda \mapsto \| \vv(\bm\theta - \bm\lambda)\|^2_{\wt  a_s \wt a_s^\T}$. For this purpose, observe that  by convex inequality on the squared norm,  
\begin{eqnarray*}
	 -{\|\vv(\wh{\bm\theta}_{s-1} - \bm\lambda) \|^2_{\wt  a_s \wt a_s^\T}}   &\leq&  -  \| \vv(\bm\theta - \bm\lambda)\|^2_{\wt  a_s \wt a_s^\T} + 2\vv(\bm\theta - \wh{\bm\theta}_{s-1})^\T (\wt a_s \wt a_s^\T)\vv(\bm\theta - \bm\lambda) \\
	 &\overset{(i)}{\leq}& -  \| \vv(\bm\theta - \bm\lambda)\|^2_{\wt  a_s \wt a_s^\T}  + 2\| \vv(\bm\theta - \wh{\bm\theta}_{s-1})\|_{\wt V_{s-1}} \| a_s \|_{V_{s-1}^{-1}} \| \vv(\bm\theta - \bm\lambda)\|_{\wt a_s \wt a_s^\T} \\
	 &\overset{(ii)}{\leq}& -  \|\vv(\bm\theta - \bm\lambda)\|^2_{\wt  a_s \wt a_s^\T}  + 2\| \vv(\bm\theta - \wh{\bm\theta}_{s-1})\|_{\wt V_{s-1}} \| a_s\|_{V_{s-1}^{-1}} \sqrt{B} \\
	 &=& -  \| \vv(\bm\theta - \bm\lambda) \|^2_{\wt  a_s \wt a_s^\T}  + \alpha_s
\end{eqnarray*}
where $(i)$ follows from Cauchy-Schwartz, and $(ii)$ follows from the definition of $B$ and, we define 
\begin{equation}
	\alpha_s := 2\|\vv(\bm\theta - \wh{\bm\theta}_{s-1})\|_{\wt V_{s-1}} \| a_s\|_{V_{s-1}^{-1}} \sqrt{B}.
\end{equation}
Thus by taking the expectation on both sides, we have 
\begin{eqnarray*}
	\bE_{\bm\lambda \sim \pi_{s-1}}\lsb  \exp\lp -{\eta \| \vv(\wh{\bm\theta}_{s-1} - \bm\lambda)\|^2_{\wt  a_s \wt a_s^\T}}  \rp \rsb &\leq& \bE_{\bm\lambda \sim \pi_{s-1}}\lsb  \exp\lp -{\eta   \| \vv(\bm\theta - \bm\lambda) \|^2_{\wt  a_s \wt a_s^\T}}\rp   \rsb \exp(\eta\alpha_s) \\
	&\leq& \exp\lp  \bE_{\bm \lambda \sim \pi_{s-1}}\lsb   -{\eta  ) \|\vv(\bm\theta - \bm\lambda) )\|^2_{\wt  a_s \wt a_s^\T}}   \rsb  \rp \exp(\eta\alpha_s)
\end{eqnarray*}
where $\eta$ ensures that the square loss ($\bm\lambda \mapsto \| \vv(\bm\lambda - \bm\theta) \|_{\wt a_t \wt a_t^\T}^2$) is $\eta$-exp-concave on $\Lambda$ (Corollary~\ref{cor:eta}). Therefore, 
\begin{eqnarray}
\label{eq:eq-swx}
\log\bE_{\bm\lambda \sim \pi_{s-1}}\lsb  \exp\lp -{\eta \| \vv(\wh{\bm\theta}_{s-1} - \bm\lambda) \|^2_{\wt  a_s \wt a_s^\T}}  \rp \rsb  &\leq& -\eta\bE_{\bm\lambda \sim \pi_{s-1}}\lsb \| \vv(\bm \theta - \bm\lambda)\|_{\wt a_s \wt a_s^\T}^2 \rsb + \eta \alpha_s. 
\end{eqnarray}
Now we will return to the loss  with $\wh{ \bm\theta}_{s-1}$ on the RHS of the above inequality. Similarly to the previous derivation, we have by convexity 
\begin{eqnarray*}
	 -{\left \|\vv(\bm\theta - \bm\lambda)\right\|^2_{\wt  a_s \wt a_s^\T}}   &\leq&  - \| \vv(\wh{\bm\theta}_{s-1} - \bm\lambda)\|^2_{\wt  a_s \wt a_s^\T} - 2\vv(\bm\theta - \wh{\bm\theta}_{t-1})^\T (\wt a_s \wt a_s^\T)\vv(\wh{\bm\theta}_{s-1} - \bm\lambda) \\
	 &\leq& - \| \vv(\wh{\bm\theta}_{s-1} - \bm\lambda) \|^2_{\wt  a_s \wt a_s^\T}  + 2\| \vv(\bm\theta - \wh{\bm\theta}_{s-1})\|_{\wt V_{s-1}} \| a_s\|_{V_{s-1}^{-1}} \| \vv(\wh{\bm\theta}_{s-1} - \bm\lambda)\|_{\wt a_s \wt a_s^\T},
\end{eqnarray*} then observe that 
\begin{eqnarray*}
	\|\vv(\wh{\bm\theta}_{s-1} - \bm\lambda)\|_{\wt a_s \wt a_s^\T} &\leq& \| \vv(\bm\theta - \wh{\bm\theta}_{s-1})\|_{\wt a_s \wt a_s^\T} + \| \vv(\bm\theta - \bm\lambda) \|_{\wt a_s \wt a_s^\T} \\
	&\leq& \|\vv(\bm\theta - \wh{\bm\theta}_{s-1})\|_{\wt a_s \wt a_s^\T} + \sqrt{B} \\
	&\leq& \|\vv(\bm\theta - \wh{\bm\theta}_{s-1})\|_{\wt V_{s-1}} \| a_s\|_{V_{s-1}^{-1}} + \sqrt{B}.
\end{eqnarray*}
Putting these displays together yields 
$$ -\eta\bE_{\bm\lambda \sim \pi_{s-1}}\lsb \| \vv(\bm\theta - \bm\lambda)\|_{\wt a_s \wt a_s^\T}^2 \rsb \leq -\eta\bE_{\bm\lambda \sim \pi_{s-1}}\lsb \| \vv(\wh{\bm\theta}_{s-1} - \bm\lambda)\|_{\wt a_s \wt a_s^\T}^2 \rsb + \eta (2\alpha_s + \alpha_s^2/(2B))  $$
which combined with \eqref{eq:eq-swx} yields 
\begin{equation}
	\log\frac{W_{s+1}}{W_s} \leq  -\frac\eta2 \bE_{\bm\lambda \sim \pi_{s-1}}\lsb \|\vv(\wh{\bm\theta}_{s-1} - \bm\lambda)\|_{\wt a_s \wt a_s^\T}^2 \rsb + \eta (2\alpha_s + \alpha_s^2/(2B)) + \frac12 \log\bE_{\bm \lambda \sim \pi_{s-1}} \lsb  \exp(\eta \Gamma_s(\bm \lambda)\rsb. 
\end{equation}
By telescoping, we have 
\begin{eqnarray}
	\log\frac{W_{t}}{W_{t_3}} &=& \sum_{s=t_3+1}^t \log\frac{W_{s}}{W_{s-1}} \\ &\leq& \sum_{s=t_3+1}^t \lp - \frac{\eta}2 \bE_{\bm \lambda \sim \pi_{s-1}}\lsb \|\vv(\wh{\bm\theta}_{s-1} - \bm\lambda)\|_{\wt a_s \wt a_s^\T}^2 \rsb + \frac12 \log\bE_{\bm \lambda \sim \pi_{s-1}} \lsb  \exp(\eta \Gamma_s(\bm \lambda)\rsb  \rp + \eta d_t\\
	&=& \!\!\!-\frac{\eta}2  \sum_{s=t_3+1}^t \bE_{\bm\lambda \sim \pi_{s-1}}\lsb \| \vv(\wh{\bm\theta}_{s-1} - \bm\lambda)\|_{\wt a_s \wt a_s^\T}^2 \rsb + \frac12 \sum_{s=t_3+1}^t \log\bE_{\bm\lambda \sim \pi_{s-1}} \lsb  \exp(\eta \Gamma_s(\bm\lambda)\rsb + \eta d_t \label{eq:eq-wws}
\end{eqnarray}
where $d_t :=\sum_{s=t_3+1}^t\eta (2\alpha_s + \alpha_s^2/(2B))$. We also define 
\begin{equation}
\label{eq:def_vartheta}
	\vartheta_t := \sum_{s=t_3+1}^t \log\bE_{\bm\lambda \sim \pi_{s-1}} \lsb  \exp(\eta \Gamma_s(\bm\lambda)\rsb.  
\end{equation}

On the other side, let $\gamma>0, \wt{\bm\lambda}_t \in \argmin_{\bm\lambda \in \overline{\Lambda}} \|\vv(\bm\lambda - \wh{\bm\theta}_{t})\|^2_{\wt V_t}$. Since $\Lambda$ is a union of convex sets, there exists a convex set $\cC \C \overline{\Lambda}$ such that $\wt{\bm\lambda}_t \in \cC$. 
 Then letting $\cN_\gamma := \{ (1-\gamma)\wt{\bm\lambda}_t + \gamma  \bm\lambda, \bm\lambda \in \cC \} = (1-\gamma)\wt {\bm\lambda}_t + \gamma \cC$, it follows 
\begin{eqnarray*}
	\log\frac{W_{t}}{W_{t_3}} &\geq& \log \frac{\int_{\vv(\cC)} \exp\lp - \frac{\eta  \| \vv(\wh{\bm\theta}_{t}) -\lambda \|^2_{\wt V_t }}{2}  \rp d\lambda}{W_{t_3}} \\
	\\&\geq& 
	 \log \frac{\int_{\vv(\cN_\gamma)} \exp\lp - \frac{\eta \|\vv(\wh{\bm\theta}_{t}) - \lambda \|^2_{\wt V_t }}{2}  \rp d\lambda}{W_{t_3}} \quad \text{(convexity of $\cC$)}\\
	 &=& \log \frac{\int_{\gamma \vv(\cC)} \exp\lp - \frac{\eta \|\vv(\wh{\bm\theta}_{t} -  (1-\gamma)\wt{\bm\lambda}_t ) - \lambda \|^2_{\wt V_t }}{2}  \rp d\lambda}{W_{t_3}}\\
	 &=&  \log \frac{\int_{\vv(\cC)} \gamma^{d h}\exp\lp - \frac{\eta  \| (1-\gamma)\vv(\wh{\bm\theta}_{t} - \wt{\bm\lambda}_t )  + \gamma \vv(\wh{\bm\theta}_t)  - \gamma \lambda \|^2_{\wt V_t }}{2}  \rp d\lambda}{W_{t_3}}
	 \\&\geq&  
	 \log \frac{\int_{\vv(\cC)} \gamma^{d h}\exp\lp - \frac{\eta \lp (1-\gamma)\|\vv(\wh{\bm\theta}_{t} - \wt{\bm\lambda}_t) \|_{\wt V_t}^2 + \gamma \|\lambda  - \vv(\wh{\bm\theta}_{t}) \|^2_{\wt V_t } \rp}{2}   \rp d\lambda}{W_{t_3}} \quad \text{ (convexity of squared norm)} \\
	 &=& dh \log(\gamma) - \frac{\eta(1-\gamma)}{2} \|\vv(\wh {\bm\theta}_{t} - \wt{\bm\lambda}_t) \|_{\wt V_t}^2 + \log \frac{\int_{\vv(\cC)} \exp\lp - \frac{\eta \gamma \|\lambda - \vv(\wh{\bm\theta}_{t}) \|^2_{\wt V_t }}{2}       \rp d\lambda}{W_{t_3}} \\
	 &\geq& d h \log(\gamma) - \frac{\eta(1-\gamma)}{2} \|\vv(\wh{\bm\theta}_{t} - \wt{\bm\lambda}_t) \|_{\wt V_t}^2 -  \frac{\eta \gamma \bE_{\bm \lambda \sim \cU(\cC) } \lsb \|\vv(\bm\lambda - \wh{\bm\theta}_{t}) \|^2_{\wt V_t }\rsb }{2} + \log\frac{\vol(\cC_*)}{W_{t_3}} 
\end{eqnarray*}
with $\Lambda = \bigcup_{i\in \cI} \cC_i$  where $(\cC_i)_{i\in \cI}$ are convex sets with non-empty interior and $\cC_*$ is the set with minimum general volume. Combining the last inequality with \eqref{eq:eq-wws} and letting $\gamma = 1/t$ yields 
$$  \sum_{s=t_3}^t \bE_{\bm \lambda \sim \pi_{s-1}}\lsb \| \vv(\wh{\bm\theta}_{s-1} -\bm \lambda)\|_{\wt a_s \wt a_s^\T}^2 \rsb  - \inf_{\bm\lambda \in \Lambda} \|\vv(\bm\lambda - \wh{\bm\theta}_{t}) \|^2_{\wt V_t} \leq \varrho_t, $$
where 
\begin{equation}
\varrho_t = \frac{2}{\eta}\log(W_{t_3}) + {2d_t} + \frac{\vartheta_t}{\eta} + \frac{2 dh \log(t)}{\eta} - \frac{2}{\eta}\log{\vol(\cC_*)} + \frac{\bE_{\bm \lambda \sim \cU(\cC) } [ \|\vv(\bm\lambda - \wh{\bm\theta}_{t}) \|^2_{\wt V_t }] }{t}. 	
\end{equation}
Then, note that $\vol(\cC_*)>0$, since $\cC_*$ has non-empty interior ($\alt(S^\star)$ is a union of convex sets with non-empty interior and $\Theta_t$ is a  ball). Further remark that $W_{t_3} \leq \vol(\Lambda)$,
 which is finite by definition of $\Lambda$ (cf \eqref{eq:def_B}). Moreover, for $\bm \lambda \in \cC \subset \Lambda$, 
 \begin{eqnarray*}
	\|\vv(\bm\lambda - \wh{\bm\theta}_{t}) \|_{\wt{V}_t } &\leq& \| \vv({\bm\theta} - \bm\lambda) \|_{\wt V_t}
 + \| \vv(\wh{\bm\theta}_t - \bm\theta) \|_{\wt V_t} \\
 &\leq& \sqrt{t B} + \sqrt{f_1(t)},
\end{eqnarray*}
which holds on the event $\Xi_{1, t}$. Summerizing the displays above, we have 
\begin{eqnarray}
\varrho_t &\leq& \frac{2}{\eta}\log\frac{\vol(\Lambda)}{\vol(\cC_*)} +  {2d_t} + \frac{\vartheta_t}{\eta} + \frac{2 dh \log(t)}{\eta} + (\sqrt{t B} + \sqrt{f_1(t)})^2/t	 \\
&=& \frac{\vartheta_t}{\eta}   +  {2d_t} + \lp B + f_1(t)/t + \sqrt{Bf_1(t)/t} + \frac{2}{\eta}\log\frac{\vol(\Lambda)}{\vol(\cC_*)} + \frac{2 dh \log(t)}{\eta} \rp, \label{eq:eq-varrho}
\end{eqnarray}
and the leftmost term is $\cO(1)$ (recall that $f_1(t)$ is logarithmic). Bounding $d_t$ and $\vartheta_t$ in Lemma~\ref{lem:lem_vartheta_dt} will conclude the proof by showing that $\varrho_t + d_t = \cO(\sqrt{t \log^2t})$. 
\end{proof}

 \begin{lemma}
 \label{lem:lem_vartheta_dt}
On the event $\Xi_{1, t} \cap \Xi_{4, t}$ (cf \eqref{eq:evt_4}) , it holds that 
 $$ d_t \leq 4hf_1(t) \log\lp  \frac{\Tr(V_0) + t L_{\cA}^2}{d \det(V_0)^{1/d}} \rp + 2 \sqrt{2 h Bt f_1(t) \log\lp  \frac{\Tr(V_0) + t L_{\cA}^2}{d \det(V_0)^{1/d}} \rp} = \cO\lp \sqrt{t\log^2(t)}\rp, \text{and }$$
$$ \vartheta_t = \cO\lp \sqrt{t\log^2(t)}\rp.$$
  \end{lemma}
  \begin{proof}
  	The first statement follows by  elliptic potential (Lemma~\ref{lem:elliptic_potential}). We have 
$d_t := \sum_{s=t_3+1}^t(\alpha_s + \alpha_s^2/(2B))$
with $\alpha_s := 2\|\vv(\bm\theta - \wh{\bm\theta}_{s-1})\|_{\wt V_{s-1}} \| a_s\|_{V_{s-1}^{-1}} \sqrt{B}.$ We have on the event $\Xi_{1, t}$
\begin{eqnarray*}
	d_t &\leq& 2 f_1(t) \sum_{s=t_3+1}^t \| a_s\|_{V_{s-1}^{-1}}^2 + 2 \sqrt{B t f_1(t)} \lp \sum_{s=t_3+1}^t \| a_s\|_{V_{s-1}^{-1}}^2 \rp^{1/2} \\
	&\leq&  4hf_1(t) \log\lp  \frac{\Tr(V_0) + t L_{\cA}^2}{d \det(V_0)^{1/d}} \rp + 2 \sqrt{2 h Bt f_1(t) \log\lp  \frac{\Tr(V_0) + t L_{\cA}^2}{d \det(V_0)^{1/d}} \rp}\\ 
	&=& \cO\lp \sqrt{t\log^2(t)}\rp. 
\end{eqnarray*}

We now prove the second part of the statement. We recall that 

\begin{equation*}
	\vartheta_t := \sum_{s=t_3+1}^t \log\bE_{\bm\lambda \sim \pi_{s-1}} \lsb  \exp(\eta \Gamma_s(\bm\lambda)\rsb.  
\end{equation*}

Letting $\bm\lambda_s \sim \pi_{s}$, and using Fubini's lemma,  $\vartheta_t$ rewrites as 
\begin{eqnarray*}
	\vartheta_t &=& \log\lp \bE_{\bm\lambda_{t_3+1}, \dots, \bm\lambda_t} \lsb \exp\lp \sum_{s=t_3+1}^t \Gamma_s(\bm\lambda_{s-1}) \rp\rsb \rp 
\end{eqnarray*}
We will first bound $\sum_{s=t_3+1}^t \Gamma_s(\bm\lambda_{s-1})$ and derive a bound on $\vartheta_t$. We recall that for any $\bm\lambda$, 
\begin{eqnarray*}
	\Gamma_s(\bm\lambda) &=& { \left \|\vv(\wh{\bm\theta}_{s-1}  - \bm\lambda) \right\|^2_{\wt V_{s} }}  - {\left \|\vv(\wh {\bm\theta}_{s}  - \bm\lambda) \right\|^2_{\wt V_{s} }}  
	\\ &\overset{(i)}{\leq}& 2  \vv(\wh{\bm\theta}_{s-1} - \bm\lambda)^\T  {\wt V_s}\vv(\wh{\bm\theta}_{s-1} - \wh{\bm\theta}_s)\\
	&\overset{(ii)}{=}& -2(\wh{\bm\theta}_{s-1} - \bm\lambda)^\T [\Sigma^{-1} \otimes a_s] (\veps_s   + (I_d \otimes a_s^\T) \vv(\bm\theta - \wh {\bm\theta}_{s-1})) \\
	&=& 2  \vv(\bm\lambda - \wh{\bm\theta}_{s-1})^\T(\wt a_s \wt a_s^\T)\vv(\bm\theta - \wh{\bm\theta}_{s-1}) + 2 \vv(\wh{\bm\theta}_{s-1} - \bm\lambda)^\T [\Sigma^{-1} \otimes a_s] \veps_s
\end{eqnarray*}
where $(i)$ follows by convexity of the squared norm and $(ii)$ is due to Lemma~\ref{lem:est_error}. We recall that $\veps_s$ is the centered (sub)Gaussian noise in the observations at time $s$, $\wt a_s = \Sigma^{-1/2} \otimes a_s$ with $a_s$ the covariate of the arm pulled at time $s$.  In the sequel, for any $\bm\lambda$, and at any time $s$, we define 
\begin{equation}
\label{eq:def_phi_psi}
	\psi_s(\bm\lambda) := \vv(\wh{\bm\theta}_{s-1} - \bm\lambda)^\T [\Sigma^{-1} \otimes a_s] \veps_s \quad \text{ and } \quad \phi_s(\bm\lambda) := \vv( \bm\lambda - \wh{\bm\theta}_{s-1})^\T(\wt a_s \wt a_s^\T)\vv(\bm\theta - \wh {\bm\theta}_{s-1}), 
\end{equation}
then, it follows that 
\begin{eqnarray*}
	\frac12 \sum_{s=t_3+1}^t \Gamma_s(\bm\lambda_{s-1}) \leq \sum_{s=t_3+1}^t \phi_s(\bm\lambda_{s-1})  + \sum_{s=t_3+1}^t \psi_s(\bm\lambda_{s-1}).
\end{eqnarray*}
We further define 
 \begin{equation}
 	\varPsi_t := \sum_{s=t_3+1}^t \psi_s(\bm\lambda_{s-1}) \quad \text{and} \quad \varPhi_t := \sum_{s=t_3+1}^t \phi_s(\bm\lambda_{s-1}).
 \end{equation}
 In the next step, we will bound each of these terms separately. In Lemma~\ref{lem:bound_sum_phi}, we bound $\varPhi_t$ and Lemma~\ref{lem:bound_sum_psi} uses martingale concentration to bound $\varPsi_t$. 
  \end{proof}
   
\begin{lemma}
\label{lem:bound_sum_phi}
The following statement holds on $\Xi_{1, t}$ 
$$ \varPhi_t \leq \sqrt{2 hB t f_1(t) \log\lp  \frac{\Tr(V_0) + t L_{\cA}^2}{d \det(V_0)^{1/d}} \rp} +  2 hf_1(t) \log\lp  \frac{\Tr(V_0) + t L_{\cA}^2}{d \det(V_0)^{1/d}} \rp = \cO\lp\sqrt{t\log^2(t)}\rp.$$
	\end{lemma}
	\begin{proof}
	First observe that for any $\bm \lambda$ and any time $s$ 
  	 \begin{equation}
  	\phi_s(\bm\lambda) \leq  \|\vv(\bm\theta - \bm\lambda) \|_{\wt a_s \wt a_s^\T} \| \vv(\wh{\bm\theta}_{s-1} - \bm\theta)\|_{\wt V_{s-1}} \| a_s \|_{V_{s-1}^{-1}} +
  	\| \vv(\wh{\bm\theta}_{s-1} - \bm\theta)\|_{\wt V_{s-1}}^2 \| a_s \|_{V_{s-1}^{-1}}^2, 
  	\end{equation}
  	which follows from \begin{eqnarray*}
			\phi_s(\bm\lambda) &=& \vv(\bm\lambda - \wh{\bm\theta}_{s-1})^\T(\wt a_s \wt a_s^\T)\vv(\bm\theta - \wh{\bm\theta}_{s-1}) \\
			&=& \|\vv(\wh{\bm\theta}_{s-1} - \bm\lambda) \|_{\wt a_s \wt a_s^\T} \|\vv(\wh{\bm\theta}_{s-1} - \bm\theta) \|_{\wt a_s \wt a_s^\T} \\
			&\leq& \|\vv(\bm\theta - \bm\lambda) \|_{\wt a_s \wt a_s^\T} \|\vv(\wh{\bm\theta}_{s-1} - \bm\theta) \|_{\wt a_s \wt a_s^\T} + \|\vv(\wh {\bm\theta}_{s-1} - \bm\theta) \|_{\wt a_s \wt a_s^\T}^2, 
		\end{eqnarray*}
		and note that $$ \|\vv(\wh{\bm\theta}_{s-1} - \bm\theta) \|_{\wt a_s \wt a_s^\T} \leq \| \vv(\wh{ \bm\theta}_{s-1} - \bm\theta)\|_{\wt V_{s-1}}\| a_s \|_{V_{s-1}^{-1}}.$$
Further, we have 
	\begin{eqnarray*}
		\varPhi_t &:=& \sum_{s=t_3+1}^t \phi_s(\bm\lambda_{s-1}) \\&\leq& \sum_{s=t_3+1}^t \lp \|\vv(\bm\theta - \bm\lambda_{s-1}) \|_{\wt a_s \wt a_s^\T} \| \vv(\wh{\bm\theta}_{s-1} - \bm\theta)\|_{\wt V_{s-1}} \| a_s \|_{V_{s-1}^{-1}} +
  	\| \vv(\wh{\bm\theta}_{s-1} - \bm\theta)\|_{\wt V_{s-1}}^2 \| a_s \|_{V_{s-1}^{-1}}^2 \rp  \\&\leq& \sqrt{B \sum_{s=t_3+1}^t \| \vv(\wh{\bm\theta}_{s-1} - \bm\theta)\|_{\wt V_{s-1}}^2}\sqrt{\sum_{s=t_3+1}^t \| a_s \|_{V_{s-1}^{-1}}^2 }  + f_1(t)\sum_{s=t_3+1}^t \| a_s \|_{V_{s-1}^{-1}}^2
	\end{eqnarray*}
	where the first term is due to Cauchy-Schwartz inequality, and the second term uses concentration on the event $\Xi_{1, t}$. Further applying elliptic potential (Lemma~\ref{lem:elliptic_potential}), we have 
	\begin{eqnarray*}
		\varPhi_t  &\leq& \sqrt{2 hB t f_1(t) \log\lp  \frac{\Tr(V_0) + t L_{\cA}^2}{d \det(V_0)^{1/d}} \rp} +  2 hf_1(t) \log\lp  \frac{\Tr(V_0) + t L_{\cA}^2}{d \det(V_0)^{1/d}} \rp.
	\end{eqnarray*}
	\end{proof}
	Introducing the event 
\begin{equation}
\label{eq:evt_4}
	\Xi_{4, t} := \lp \sum_{s=t_3 + 1}^t  \psi_s  \leq  f_4(t) := \sqrt{2  \sum_{s=t_3+1}^t v_s  \log(t^2) } \rp, 
\end{equation}
we prove the following result. 
\begin{lemma}
\label{lem:bound_sum_psi}
The following statement holds on the event $\Xi_{1, t} \cap \Xi_{4, t}$ 
$$\varPhi_t \leq  \sqrt{2\log(t^2)\sum_{s=t_3+1}^t v_s}, $$
with 
$$ \sum_{s=t_3+1}^t v_s \leq tB + 2 h f_1(t)\log\lp  \frac{\Tr(V_0) + t L_{\cA}^2}{d \det(V_0)^{1/d}} \rp + 2\sqrt{2 hB t f_1(t) \log\lp  \frac{\Tr(V_0) + t L_{\cA}^2}{d \det(V_0)^{1/d}} \rp}, $$
so that $ \varPhi_t \leq  \cO\lp \sqrt{t\log(t)}\rp$. 
\end{lemma}
\begin{proof}
To bound $\varPsi_t$, we use martingale arguments by defining a slightly richer filtration. In the sequel, we denote by $\wt \cF_s:= \sigma(\cF_{s}, \bm \lambda_{s})$ the information accumulated up to time $s$ including $\bm \lambda_{s}$, i.e $\bm \lambda_s$ is $\wt \cF_{s}$-measurable. We simplify notation and let $\psi_s := \psi_s(\bm \lambda_{s-1}) = \vv(\wh{\bm\theta}_{s-1} - \bm \lambda_{s-1})^\T [\Sigma^{-1} \otimes a_s] \veps_s $, where we recall that $\bm\lambda_s$ is the strategy of the \emph{min} player at time $s$. Then observe that $\psi_s$ is $\wt \cF_s$-measurable and we have 
\begin{eqnarray*}
\bE[\psi_s \mid \wt \cF_{s-1}] &=& \vv(\wh{\bm\theta}_{s-1} - \bm\lambda_{s-1})^\T \bE[[\Sigma^{-1} \otimes a_s ] \veps_s \mid \wt \cF_{s-1}] \\
&=& 0 
\end{eqnarray*}
that is $(\psi_s)_s$ is a $\wt \cF$-martingale difference sequence with conditionally (sub)Gaussian increment of variance (proxy)
$$ v_s := \| [\Sigma^{-1} \otimes a_s^\T] \vv(\wh{\bm\theta}_{s-1} - \bm\lambda_{s-1})\|_{\Sigma}^2.$$ 
Using Azuma's inequality, we have that $\bP(\Xi_{4, t}) \geq 1- 1/t^2$. 
We argue that $f_4(t)$ is a sub-linear term on $\Xi_{1, t}$. Indeed, observe that 
\begin{eqnarray*}
	v_s &=& \vv(\wh{\bm\theta}_{s-1} - \bm\lambda_{s-1})^\T [\Sigma^{-1} \otimes a_s] \Sigma [\Sigma^{-1} \otimes a_s^\T]\vv(\wh{\bm\theta}_{s-1} - \bm\lambda_{s-1}) \\
	&=& (1 / \| a_s\|^2) \vv(\wh{\bm\theta}_{s-1} - \bm\lambda_{s-1})^\T [\Sigma^{-1} \otimes a_s] [\Sigma \otimes a_s^\T] [I_d \otimes a_s] [\Sigma^{-1} \otimes a_s^\T]\vv(\wh{\bm\theta}_{s-1} - \bm\lambda_{s-1}) \\
	&=& (1 / \| a_s\|^2) \vv(\wh{\bm\theta}_{s-1} - \bm\lambda_{s-1})^\T [I_d \otimes a_s a_s^\T] [\Sigma^{-1} \otimes a_s a_s^\T] \vv(\wh{\bm\theta}_{s-1} - \bm\lambda_{s-1}) \\
	&=& \vv(\wh{\bm\theta}_{s-1} - \bm\lambda_{s-1})^\T [\Sigma^{-1} \otimes a_s a_s^\T] \vv(\wh{\bm\theta}_{s-1} - \bm\lambda_{s-1}) \\
	&=& \| \vv(\wh{\bm\theta}_{s-1} - \bm\lambda_{s-1})\|_{\wt a_s \wt a_s^\T}^2.  
\end{eqnarray*} 
We will now bound the right-hand side as :
\begin{eqnarray*}
	\| \vv(\wh{\bm\theta}_{s-1} - \bm\lambda_{s-1})\|_{\wt a_s \wt a_s^\T} &\leq& \|\vv(\wh{\bm\theta}_{s-1} - \bm\theta)\|_{\wt a_s \wt a_s^\T} + \| \vv(\bm\theta - \bm\lambda_{s-1})\|_{\wt a_s \wt a_s^\T} \\
	&\leq& \| \vv(\wh{\bm\theta}_{s-1} - \bm\theta)\|_{\wt V_{s-1}} \| a_s\|_{V_{s-1}^{-1}} + \sqrt{B},
\end{eqnarray*}
thus 
\begin{eqnarray*}
	v_s \leq B + \|\vv(\wh{\bm\theta}_{s-1} - \bm\theta)\|_{\wt V_{s-1}}^2 \| a_s\|_{V_{s-1}^{-1}}^2 + 2\sqrt{B}\| \vv(\wh{\bm\theta}_{s-1} - \bm\theta) \|_{\wt V_{s-1}} \| a_s\|_{V_{s-1}^{-1}},
\end{eqnarray*}
then, using Cauchy-Schwartz and the elliptic potential lemma, it follows that on the event $\Xi_{1,t}$, 
\begin{eqnarray*}
\sum_{s=t_3+1}^t v_s &\leq& t B + f_1(t) \sum_{s=t_3+1}^t \| a_s\|_{V_{s-1}^{-1}}^2 + 2\sqrt{B} \lp \sum_{s=t_3+1}^t \| \vv(\wh{\bm\theta}_{s-1} - \bm\theta)\|_{\wt V_{s-1}}^2 \rp^{1/2}\lp \sum_{s=t_3+1}^t \| a_s\|_{V_{s-1}^{-1}}^2\rp^{1/2}, \\
&\leq& tB + 2 h f_1(t)\log\lp  \frac{\Tr(V_0) + t L_{\cA}^2}{d \det(V_0)^{1/d}} \rp + 2\sqrt{2 hB t f_1(t) \log\lp  \frac{\Tr(V_0) + t L_{\cA}^2}{d \det(V_0)^{1/d}} \rp},
\end{eqnarray*}
where the last inequality follows from elliptic potential lemma, which achieves the proof. 
\end{proof}

Combining Lemma~\ref{lem:lem_vartheta_dt} and Equation~\eqref{eq:eq-varrho}, we have proved that on $\Xi_{1, t}  \cap \Xi_{4, t}$ for $t>t_3$, we have $ r_1(t) = \cO\lp \sqrt{t \log^2(t)}\rp$. 

and this achieves the proof of the \emph{min} player's regret, as we have bounded $r_1(t)$.
\subsubsection{Regret of the \emph{max}-learner} We study the \emph{max}-leaner's regret, aiming to prove that $r_2(t)$ is a sub-linear term. Formally, introducing 
 the event (with $Z_s(w)$ as in \eqref{eq:def_zt} and $w_t^s$ defined in \eqref{eq:def_wt})
 \begin{equation}
 \label{eq:evt_5}
 	\Xi_{5, t}  :=  \lp \sup_{\w \in \cV(\veps)} \sum_{s=t_3+1}^{t} Z_s(\w) \leq f_5(t) \rp \quad \text{with} \quad f_5(t) := \sqrt{{2 \lp \sum_{s=t_3+1}^{t} (w_t^s(a))^2 \rp } \log(\lvert \cV(\veps)\rvert t^2)},  
 \end{equation}
in this section, we prove that 
\begin{lemma}
	\label{lem:max_learner}
	On the event $\Xi_{1, t} \cap \Xi_{5, t}$, it holds that 
$$\sup_{\w \in \D} \sum_{s=t_3+1}^{t-1} (\wt g_s(\w) - \wt g_s(\w_s)) \leq \cO\lp \sqrt{t \log^2(t)} + \log(t)t^{1-\alpha}\rp,$$
with $ \wt g_s(\w) := \bE_{\bm \lambda \sim \pi_{s-1}, a \sim \w} \lsb  \| \vv(\wh{\bm\theta}_{s-1} - \bm\lambda)\|^2_{\wt a \wt a^\T}\rsb$.
\end{lemma}

\begin{proof}
We recall that $\w \in \D$ is considered a distribution over $\cA$ and for any $a\in \cA$, $\wt a := \Sigma^{-1/2} \otimes a$. 
Let us define for any $\w \in \D$, 
	\begin{align} 
		Z_s(\w) := &\bE_{\bm \lambda \sim \pi_{s-1}, a \sim \w} \lsb \left \| \vv(\wh{\bm\theta}_{s-1} - \bm \lambda) \right\|^2_{\wt a \wt a^\T}\rsb   -  \bE_{a \sim \w }\lsb \left\| \vv(\wh{\bm\theta}_{s-1}- \bm\lambda_{s-1}) \right\|_{\wt a \wt a^\T}^2 \rsb - \label{eq:def_zt}\\
	 &\bE_{\bm \lambda \sim \pi_{s-1}, a \sim \wt \w_s} \lsb \left\| \vv(\wh{\bm\theta}_{s-1} - \bm\lambda) \right\|^2_{\wt a \wt a^\T}\rsb + \bE_{a \sim \wt \w_s} \lsb \left\| \vv(\wh{\bm\theta}_{s-1} - \bm\lambda_{s-1}) \right\|^2_{\wt a \wt a^\T}\rsb.  \nonumber
	\end{align}
	Then observe that $(Z_s(\w))_s$ is a $\wt \cF$-adapted process and and it is a $\wt \cF$-martingale difference sequence for any fixed $\w$ since by direct algebra, $\bE[Z_s(\w) \mid \wt \cF_{s-1}] = 0$. 
	Further note that, for $\bm \lambda \in \Lambda_s$, 
\begin{eqnarray*}
	\|\vv(\wh{\bm\theta}_{s-1} - \bm\lambda)  \|_{\wt a \wt a^\T} &\leq& \|\vv(\wh{\bm\theta}_{s-1} - \bm\theta)  \|_{\wt a \wt a^\T} + \|\vv({\bm\theta} - \bm\lambda)  \|_{\wt a \wt a^\T}, \\
	&\leq & \| \vv(\wh{\bm\theta}_{s-1} - \bm\theta)\|_{\wt V_{s-1}} \| a\|_{V_{s-1}^{-1}}  + \sqrt{B_s}. 
\end{eqnarray*}
Therefore, on the event $\Xi_{1,t}$, we have $\lvert Z_s(\w) \lvert  \leq 2 \max_{a\in \cA} (w_t^s(a))^2,$ where
	\begin{equation}
	\label{eq:def_wt}
		w_t^s(a) := 2 \sqrt{f_1(t)} \| a\|_{V_{s-1}^{-1}} + \sqrt{B_2}
	\end{equation}
	and remark that $B_2 <\infty$. Let $\veps >0$ and $\cV(\veps)$ be an $\veps$-cover of $\D$, $\cN(\veps) = \lvert \cV(\veps) \rvert$. Note that by Azuma-Hoeffding's inequality, it is simple to see that $\bP(\Xi_{5, t})\geq 1- 1/t^2$.  Direct development yields 
\begin{align*}
	(*) := \sum_{s=t_3+1}^{t-1} (\wt g_s(\w) - \wt g_s(\w_s)) &= \sum_{s=t_3+1}^{t-1} Z_s(\omega)  + \sum_{s=t_3+1}^{t-1}  \bE_{a \sim \w }\lsb \|\vv(\wh{\bm\theta}_{s-1} - \bm\lambda_{s-1})\|_{\wt a \wt a^\T}^2 \rsb -\bE_{a \sim \wt \w_s} \lsb \| \vv(\wh{\bm\theta}_{s-1} - \bm\lambda_{s-1}) \|^2_{\wt a \wt a^\T}\rsb \\
	=&   \sum_{s=t_3+1}^{t-1}  \bE_{a \sim \w }\lsb \| \vv(\wh {\bm\theta}_{s-1} - \bm\lambda_{s-1}) \|_{\wt a \wt a^\T}^2 \rsb -\bE_{a \sim \w_s} \lsb \|\vv(\wh{\bm\theta}_{s-1} - \bm\lambda_{s-1}) \|^2_{\wt a \wt a^\T}\rsb 
	+  \\ &\qquad \sum_{s=t_3+1}^{t-1} Z_s(\omega)  + \sum_{s=t_3+1}^{t-1} \varphi_s, 
\end{align*}
where 
\begin{eqnarray*}
	\varphi_s  &=&  \bE_{a \sim \w_s} \lsb \left\| \vv(\wh{\bm\theta}_{s-1} - \bm\lambda_{s-1}) \right\|^2_{\wt a \wt a^\T}\rsb - \bE_{a \sim \wt \w_s} \lsb \left\| \vv(\wh{\bm\theta}_{s-1} - \bm\lambda_{s-1})\right\|^2_{\wt a \wt a^\T}\rsb \\
	&=& \gamma_s \lp \bE_{a \sim \w_s} \lsb \left\| \vv(\wh {\bm\theta}_{s-1} - \bm\lambda_{s-1}) \right\|^2_{\wt a \wt a^\T}\rsb - \bE_{a \sim \w_\text{exp}} \lsb \left\| \vv(\wh{\bm\theta}_{s-1} - \bm\lambda_{s-1}) \right\|^2_{\wt a \wt a^\T}\rsb \rp, \\
	&\leq& \gamma_s \max_{a\in \cA} (w_t^s(a))^2 \end{eqnarray*}
where the last inequality holds on the event $\Xi_{1, t}$. Therefore, on the event $\Xi_{5, t} \cap \Xi_{1, t}$, for all $\w \in \cV(\veps)$, it holds that 
\begin{eqnarray} 
	 (*)  \leq h(t) + f_5(t) + \sum_{s=2}^{t-1} \gamma_s \max_{a\in \cA} (w_t^s(a))^2 + \sum_{s=2}^{t-1}  \bE_{a \sim \w }\lsb \| \vv(\wh{\bm\theta}_{s-1} - \bm\lambda_{s-1}) \|_{\wt a \wt a^\T}^2 \rsb -\bE_{a \sim \w_s} \lsb \| \vv(\wh {\bm\theta}_{s-1} - \bm\lambda_{s-1}) \|^2_{\wt a \wt a^\T}\rsb, \label{eq:eq-ws-dfg}
\end{eqnarray}
where the rightmost term of Equation~\eqref{eq:eq-ws-dfg} is related to the regret of AdaHedge and 
\begin{eqnarray}	
	h(t) &:=& \sum_{s=2}^{t_3}\bE_{a \sim \w_s} \lsb \| \vv(\wh {\bm\theta}_{s-1} - \bm\lambda_{s-1}) \|^2_{\wt a \wt a^\T}\rsb - \bE_{a \sim \w }\lsb \| \vv(\wh{\bm\theta}_{s-1} - \bm\lambda_{s-1})\|_{\wt a \wt a^\T}^2 \rsb, \nonumber \\
	&\leq& \sum_{s=2}^{t_3} \max_{a\in \cA} (w_t^s(a))^2, \nonumber \\
	&\leq& \cO(f_1(t)) \label{eq:bound_h}, 
\end{eqnarray} which follows by expanding the sum, and since $t_3$ fixed.  

\begin{lemma}[\cite{rooij_ada_hedge}]
AdaHedge run with gains $g_{s}(\w) := \bE_{a \sim \w }[\|\vv(\wh{\bm\theta}_{s} - \bm\lambda_{s})\|_{\wt a \wt a^\T}^2]$ satisfies the following regret bound 
		$$ \max_{\w \in \D} \sum_{s=2}^{t-1}  \bE_{x \sim \w }\lsb \left\| \vv(\wh{\bm\theta}_{s-1} - \bm\lambda_{s-1}) \right\|_{\wt a \wt a^\T}^2 \rsb -\bE_{a \sim \w_s} \lsb \left\|\vv(\wh{\bm\theta}_{s-1} - \bm\lambda_{s-1}) \right\|^2_{\wt a \wt a^\T}\rsb  \leq 2\sigma_t \sqrt{\log(\lvert \cA \lvert t)} + 16\sigma_t (2 + \log(\lvert \cA \rvert)/3), $$ where  $\sigma_t := \max_{s \leq t, a \in \cA} \|\vv(\wh{\bm\theta}_{s-1} - \bm\lambda_{s-1})  \|_{\wt a \wt a^\T}^2.$
\end{lemma}
We can bound $\sigma_t$ as 
\begin{eqnarray}
	\sigma_t &\leq& \max_{a\in \cA, s\leq t} (w_t^s(a))^2 \nonumber\\
	&\leq& \max_{a\in \cA, s\leq t} (2 \sqrt{f_1(t)} \| a\|_{V_{0}^{-1}} + \sqrt{B_2})^2 \nonumber \\&=&
	(2 \sqrt{f_1(t)} \max_{a\in \cA} \| a\|_{V_{0}^{-1}} + \sqrt{B_2})^2 \label{eq:bound_sig_t}.
\end{eqnarray}
Thus, we obtain for all $\w \in \cV(\veps)$, 
\begin{equation}
	(*) \leq f_5(t) + h(t) + \sum_{s=1}^{t-1} \gamma_s (w_t^s(a))^2  + 2\sigma_t \sqrt{\log(\lvert \cA \lvert t)} + 16\sigma_t (2 + \log(\lvert \cA \rvert)/3) .
\end{equation}
In the next step, we relate the supremum over $\cV(\veps)$ to the supremum over $\D$ by using the covering argument and a Lipschitz condition. 
Letting $\w, \w' \in \D$, observe that 
\begin{eqnarray*}
\wt g_s(\omega) - \wt g_s(\omega')  &=& \sum_{a \in \cA} (\w_a - \w'_a)\bE_{\bm \lambda \sim \pi_{s-1}} \lsb \left \| \vv(\wh {\bm\theta}_{s-1} - \bm\lambda )\right\|^2_{\wt a \wt a^\T}\rsb\\
 &\leq& \| \w - \w'\|_1 (w_t^s(a))^2
\end{eqnarray*}
thus, as $\cV(\veps)$ is an $\veps$-cover of $\D$, for any $\w \in \D$, there exists $\w' \in \cV(\veps)$ such that $\| \w - \w'\|_1 \leq \veps$ and $\wt g_s(\w) \leq \wt g_s(\w') + \veps (w_t^s(a))^2$. All put together, we have 
\begin{eqnarray*}
\sup_{\w \in \D} \sum_{s=1}^{t-1} (\wt g_s(\w) - \wt g_s(\w_s)) &=& \sup_{\w \in \D} \min_{\w' \in \cV(\veps)} \sum_{s=1}^{t-1} (\wt g_s(\w)  - \wt g_s(\w') + \wt g_s(\w')- \wt g_s(\w_s))	\\
&\leq& \sup_{\w' \in \cV(\veps)} \sum_{s=1}^{t-1} (\wt g_s(\w') - \wt g_s(\w_s)) + \veps \sum_{s=1}^{t-1} (w_t^s(a))^2 \\
&\leq& f_5(t) + h(t) + \sum_{s=1}^{t-1} \gamma_s (w_t^s(a))^2  + \sqrt{1/t} \sum_{s=1}^{t-1} (w_t^s(a))^2  \\
& &\qquad + 2\sigma_t \sqrt{\log(\lvert \cA \lvert t)} + 16\sigma_t (2 + \log(\lvert \cA \rvert)/3), 
\end{eqnarray*}
where we take $\veps = 1/\sqrt{t}$ and recall that $\cN(\veps) \leq (3/\veps)^K$.  Replacing with the expression of each term we have 
\begin{eqnarray}
	f_5(t) &\leq & \sqrt{{4 \lp \sum_{s=t_3+1}^{t} \max_{a\in \cA}(w_t^s(a))^2 \rp } \log(\lvert \cV(\veps)\rvert t^2)}\nonumber\\
	&\leq& \sqrt{{4K \lp \sum_{s=t_3+1}^{t} \max_{a\in \cA}(w_t^s(a))^2 \rp } \log(3t^2\sqrt{t})} \nonumber \\
	&\leq& \sqrt{{10K t(2 \sqrt{f_1(t)} \max_{a\in \cA} \| a\|_{V_{0}^{-1}} + \sqrt{B_2} )^2 } \log(3t)} = \cO(\sqrt{t\log^2(t)})\label{eq:bound_f5}
\end{eqnarray}
similarly, we have 
\begin{eqnarray}
	\sum_{s=1}^{t-1} \gamma_s (w_t^s(a))^2 &\leq &  (2 \sqrt{f_1(t)} \max_{a\in \cA} \| a\|_{V_{0}^{-1}} + \sqrt{B_2} )^2 \sum_{s=1}^{t-1}s^{-\alpha}  \nonumber \\
	&\leq& (2 \sqrt{f_1(t)} \max_{a\in \cA} \| a\|_{V_{0}^{-1}} + \sqrt{B_2} )^2 \frac{t^{1- \alpha}}{1-\alpha} \label{eq:bound_int} = \cO(\log(t) t^{1-\alpha}).
\end{eqnarray} 

Combining Equations \eqref{eq:bound_h}, \eqref{eq:bound_sig_t}, \eqref{eq:bound_f5}   and \eqref{eq:bound_int} yields the claimed statement. 
\end{proof}
This achieves the proof of the sublinear regret for the \emph{max} player. 
 \subsection{Proof of Theorem~\ref{thm:sdp_convergence}}
 Introducing the event 
 \begin{equation}
 \label{eq:def_xi}
 	\Xi_t = \bigcap_{i=1}^5 \Xi_{i, t} \:,
 \end{equation}
 we have proven in the sections above that there exists a time $t_3\in \bN$ such that when $\Xi_t $ holds,  \begin{eqnarray}
	\inf_{\bm\lambda \in \Lambda_{t-1}} \| \vv(\bm\lambda - \wh{\bm\theta}_{t-1}) \|_{\wt V_{t-1}}^2 
		&\geq&  (t-t_3) \max_{\w \in \D} \inf_{\bm \lambda \in \Theta \cap \alt(S^\star)} \left\| \vv(\bm\theta - \bm\lambda) \right\|_{\wt V_{\bm \w}}^2 - r_1(t) - r_2(t) - m(t)  - z(t), \nonumber \\
		& = & (t-1)T^\star(\bm \theta)^{-1} - 2 \underbrace{\lp  (1+t_3) T^\star(\bm \theta)^{-1} + r_1(t) + r_2(t) + m(t)  + z(t) \rp/2}_{f(t-1)}, \label{eq:eq-op-rt}
\end{eqnarray}
and we further proved that $f(t)$ is sub-linear (Equation~\eqref{eq:def_ztt}, Equation~\eqref{eq:def_mt}, Lemma~\ref{lem:min_learner}, Lemma~\ref{lem:max_learner}). To conclude the proof of Theorem~\ref{thm:sdp_convergence}, it remains to prove that the LHS of Equation~\eqref{eq:eq-op-rt} is the GLR at time $t-1$. Indeed, in the transductive setting, as $\Lambda_{t-1} = \Theta \cap \alt(\wh S_{t-1})$, the result is immediate. In the unstructured setting, the result follows by Lemma~\ref{lem:uns_contr} applied  to $\wh{\bm\mu}_{t}$, for $t\geq t_3$ on the event $\Xi_{1,t}\cap \Xi_{2,t} \subset \Xi_t$.
Further observe that by their definition we have for each $i\in [5]$, $\bP_{\bm \nu}(\Xi_{i,t}) \leq 1/t^2$. Therefore, 
$$ \bP_{\bm \nu}(\Xi_t) \geq 1- 5/t^2$$ and the conclusion follows.


\section{LIKELIHOOD RATIO AND POSTERIOR PROBABILITY OF ERROR}
\label{app:lr_poe}
In this section, we prove some results related to the generalized likelihood ratio, the lower bound and the posterior probability of error. 
\subsection{Lower Bound}
We discuss the proof of the lower bound of PSI in this section.  
\LowerBound*
\begin{proof}
	The proof of this lemma follows the same lines as Theorem~1 of~\cite{garivier_optimal_2016} from which it is simple to prove that the stopping time $\tau_\delta$ of any $\delta$-correct algorithm for PSI satisfies 
	$$\bE_{\bm \nu}[\tau_\delta] \geq   T^\star(\bm\theta)\log(1/(2.4\delta)),$$ where, for the problems in $\cD^K$,  
	\begin{eqnarray*}
		T^\star(\bm \theta)^{-1} &:=&  \sup_{\bm \w \in \D} \inf_{\bm \lambda \in \Theta \cap \alt(S^\star(\bm\theta))}  \sum_{a\in \cA} \frac12 \omega_a \left\| (\bm \theta- \bm\lambda)^\T a \right\|^2_{\Sigma^{-1}}. 
	\end{eqnarray*}
Then, observe that by the properties of vectorization and Kronecker product, 
\begin{eqnarray*}
	\vv(\bm \theta^\T a ) &=& \vv(a ^\T \bm \theta),\\
	&=&  (I_d \otimes a^\T)\vv(\bm \theta),
		\end{eqnarray*}
which follows from $\vv(AB) = (I_d\otimes A) \vv(B)$ for $A\in \bR^{p, q}, B\in \bR^{q, d}$. Therefore, 
\begin{eqnarray*}
	\left\| (\bm \theta- \bm\lambda)^\T a\right\|^2_{\Sigma^{-1}}  &=& \vv(\bm \theta - \bm \lambda)^\T (I_d \otimes a^\T)^\T \Sigma^{-1} (I_d \otimes a^\T)\vv(\bm \theta - \bm\lambda) \\
	&=& \vv(\bm \theta - \bm \lambda)^\T \lp (I_d \otimes a) \Sigma^{-1} (I_d \otimes a^\T) \rp \vv(\bm \theta - \bm \lambda),
\end{eqnarray*}
further note that 
\begin{eqnarray*}
	(I_d \otimes a) \Sigma^{-1} (I_d \otimes a^\T) &=& \frac{1}{a^\T a} (I_d \otimes a) (\Sigma^{-1} \otimes a^\T a) (I_d \otimes a_i^\T) \\
	&\overset{(i)}{=}& \frac{1}{a^\T a} \Sigma^{-1} \otimes (a a^\T a a_i^\T) \\
	&=& \Sigma^{-1} \otimes (a a^\T), 
\end{eqnarray*}
where $(i)$ follows from $(A\otimes B) (C\otimes D) = (AC) \otimes (B D)$. 
All put together, we have 
\begin{eqnarray*}
	T^\star(\bm \theta)^{-1} &:=& \sup_{\w \in \D} \inf_{\bm \lambda \in \Theta \cap \alt(S^\star(\bm\theta))}  \sum_{a\in \cA} \frac12 \omega_a \left\| \vv(\bm \theta - \bm \lambda) \right\|^2_{\Sigma^{-1} \otimes a a^\T} \\
&=& \frac12 \sup_{\w \in \D} \inf_{\bm \lambda \in \Theta \cap \alt(S^\star(\bm\theta))} \left\| \vv(\bm \theta - \bm \lambda)\right\|_{\Sigma^{-1} \otimes V_{\w}}^2, 
\end{eqnarray*} with $V_{\w} := \sum_{a \in \cA} \omega_a a a^\T$. 
\end{proof}

\subsection{Posterior Error Probability}
	\begin{lemma}
\label{lem:prob_to_glr}
	Under the unstructured assumption or the linear transductive setting, and for any convex set $\Theta$, it holds at each round that 
	$$ \bP_{\wh \Pi_t \mid \cF_t}( \Theta \cap \alt(\wh S_t)) \leq \alpha_t \exp\lp -\frac{\textrm{GLR}(t)}{c(t-1, \delta)}\rp,$$
	with  $2 \alpha_t = p_t(p_t -1) + (\lvert \cZ \rvert - p_t)d^{p_t}$, with $p_t := \lvert \wh S_t \rvert$, the size of the empirical Pareto set at time $t$. 
\end{lemma}
 \begin{proof}
 	The proof of this lemma relies on the properties of  $\Theta \cap \alt(\wh S_t)$ and Gaussian concentration. Lemma~\ref{lem:alt_convexity} proves that $\alt(\wh S_t) = \cup_{i \in [n_t]} C_i$ where $C_1,\dots,C_n$ are convex sets and $n_t := p_t(p_t -1) + (\lvert \cZ \rvert - p_t)d^{p_t}$. Since, $\wh\Pi_t \mid \cF_t = \cN(\vv(\bm\theta_t), c(t-1, \delta)\Sigma \otimes V_t^{-1})$, letting $\vv(\bm\lambda_t) \mid \cF_{t} \sim  \cN(\vv(\bm\theta_t), c(t-1, \delta)\Sigma \otimes V_t^{-1})$, we have 
 	$$\bP(\bm \lambda_t \in C_i \mid \cF_t) = \bP(\vv(\bm \lambda_t) \in \vv(C_i) \mid \cF_t)$$ 
 	and by Lemma~\ref{lem:mvn-upb}, 
 and convexity of $\vv(C_i)$, it follows that 
  \begin{eqnarray*}
  	\bP(\bm \lambda_t \in C_i \mid \cF_t) \leq \frac12 \exp\lp - \inf_{\bm \lambda \in C_i} \frac{\| \vv(\wh{\bm \theta}_t - \bm\lambda)\|_{\Sigma \otimes V_t^{-1}}^2}{2c(t-1, \delta)}\rp, 
  \end{eqnarray*}
  therefore, by union bound and convexity of $\Theta \cap C_i$ (since $\Theta$ is convex), it follows that 
 \begin{eqnarray*}
 	\bP_{\wh \Pi_t \mid \cF_t}( \Theta \cap \alt(\wh S_t)) &:=& \bP\lp \bm \lambda_t \in \cup_{i \in [n_t]} (\Theta \cap C_i) \mid \cF_t\rp \\
 	&\leq& \frac12 \sum_{i \in [n_t]} \exp\lp - \inf_{\bm \lambda \in \Theta \cap C_i}  \frac{\| \vv(\wh{\bm \theta}_t - \bm\lambda)\|_{\Sigma \otimes V_t^{-1}}^2}{2c(t-1, \delta)} \rp\\
 	&\leq& \frac{n_t}2 \exp\lp - \inf_{\bm \lambda \in \Theta \cap \alt(\wh S_t)}  \frac{\| \vv(\wh{\bm \theta}_t - \bm\lambda)\|_{\Sigma \otimes V_t^{-1}}^2}{2c(t-1, \delta)} \rp \\
 	&=& \frac{n_t}2 \exp\lp - \frac{\text{GLR}(t)}{c(t-1, \delta)}\rp
 \end{eqnarray*}
 and from Lemma~\ref{lem:alt_convexity},  $n_t = p_t(p_t -1) + (\lvert \cZ  \rvert - p_t)d^{p_t}$, with $p_t := \lvert \wh S_t \rvert$,
 \end{proof}


\section{ASYMPTOTIC EXPECTED SAMPLE COMPLEXITY}
\label{app:sample_complexity}

We prove the asymptotic optimality of our algorithm. 
\upperBoundSc*
\begin{proof}
In this section, $(\Xi_t)_{t\geq 1}$ denotes the sequence of events introduced in Equation~\eqref{eq:def_xi} and $t_3 \in \bN$. 
We have  $\tau_\delta^\text{PS} = 1+ \sum_{t\geq 1} \ind_{\lp \tau_\delta^\text{PS} > t\rp}$ then 
\begin{eqnarray*}
	\bE[\tau_\delta^\text{PS}] &=& 1 + \bE\left[\sum_{t\geq 1}\bP\lp \tau_\delta^\text{PS} > t \mid \cF_{t}\rp \right] \\
	&=& 1+ \bE\left[ \sum_{t\geq 1}\bP\lp \exists m \in [M(t,\delta)]: \bm \theta_t^m  \in \Theta \cap \alt(\wh S_t) \mid \cF_{t}\rp \right] \\
	&\leq& \bE\left[ \sum_{t\geq t_3} \ind_{\Xi_t}M(t, \delta)\bP\lp \bm \theta_t  \in \Theta \cap \alt(\wh S_t) \mid \cF_{t}\rp  \right] + \lsb \sum_{t\geq 1} \bP(\Xi_t^c) \rsb + (t_3+1),  
	\end{eqnarray*}
	where  $\bm \theta_t, \bm \theta_t^1 \dots, \dots, \bm \theta_t^m$ are \iid given $\cF_{t}$. Moreover, by Lemma~\ref{lem:prob_to_glr}, we have 
	$$ \bP\lp \bm \theta_t  \in \Theta \cap \alt(\wh S_t) \mid \cF_{t}\rp \leq \alpha_t \exp\lp - \frac{\text{GLR}(t)}{c(t-1, \delta)}\rp$$
	with $2 \alpha_t = p_t(p_t -1) + (\lvert \cZ \rvert - p_t)d^{p_t} \leq \lvert \cZ \rvert (\lvert \cZ \rvert + d^{\lvert \cZ \rvert}) := \alpha_0$, 
	therefore, 
\begin{eqnarray*}
	\bE[\tau_\delta^\text{PS}] 
	&\leq& \underbrace{\bE\left[ \sum_{t\geq t_3} \ind_{\Xi_t}\alpha_0 M(t, \delta)\exp\lp - \frac{\text{GLR}(t)}{c(t-1, \delta)}\rp \right]}_{L_1(\delta)} + \underbrace{\lsb \sum_{t\geq 1} \bP(\Xi_t^c) \rsb}_{L_2} + (t_3+1),  
	\end{eqnarray*}
	then, since $\bP(\Xi_t^c) \leq 5/t^2$ we immediately have  
\begin{equation}
\label{eq:zz-mm}
	L_2 \leq 5\pi^2/6.
\end{equation}
It remains to bound $L_1(\delta)$. Using the saddle-point convergence property on the event $\Xi_t$ ensures that (Theorem~\ref{thm:sdp_convergence}) for $t\geq t_3$, 
\begin{eqnarray}
	2\text{GLR}(t) := \inf_{\bm \lambda \in \Theta \cap \alt(\wh S_t)} \left\|\vv( \bm \lambda - \wh{\bm\theta}_{t}) \right\|_{\Sigma^{-1} \otimes V_{t}}^2 
	&\geq & t \max_{\bm\w \in \Delta} \inf_{\bm \lambda \in \Theta \cap \alt(S^\star(\bm\theta))} \left\|\vv(\bm\theta - \bm\lambda)\right\|_{\Sigma^{-1} \otimes V_{\bm\w}}^2 - 2f(t), \end{eqnarray}
which further results in  
\begin{eqnarray*}
	 L_1(\delta) &\leq& \sum_{t\geq t_3}  \alpha_0 M(t,\delta) \exp\lp - \frac{t}{T^\star(\bm \theta) c(t, \delta)} + f(t)/c(t, \delta)\rp  \\
	 &=& \sum_{t\geq t_3}   \exp\lp - \frac{t}{T^\star(\bm \theta) c(t, \delta)} + f(t)/c(t, \delta) + \log\lp \alpha_0M(t,\delta)\rp\rp.
\end{eqnarray*}
To bound the above quantity, let us introduce  
\begin{equation}
	T(\delta) := \sup \lb t \left. \mid \right. \frac{t}{T^\star(\bm \theta) c(t, \delta)}    - \log(M(t, \delta)) - f(t) / c(t, \delta) - \log(\alpha_0) \leq \log(t\log(t) ) \rb, 
\end{equation}
then it follows that 
\begin{equation}
\label{eq:zz-oo}
	L_1(\delta) \leq T(\delta) + \sum_{t\geq t_3} (t \log(t))^{-1}.
\end{equation}
Further observe that $T(\delta)$ can be rewritten as 
\begin{eqnarray*}
		T(\delta) := \sup\lb t\mid  t \leq  {T^\star(\bm \theta)} c(t, \delta) \lp \log(M(t, \delta)) +  f(t)/c(t, \delta) + \log(\alpha_0) + \log(t\log(t))  \rp \rb. 
	\end{eqnarray*}
	
\paragraph{Bounding $T(\delta)$}
Since $f$ is sub-linear in $t$, there exists $\veps \in (0, 1)$ such that $f(t) = o_{t\rightarrow \infty}(t^{\veps})$ then $f(\log(1/\delta)^{1/\veps}) = o(\log(1/\delta))$. Further observe that for $t_\delta = \log(1/\delta)^{1/\alpha}$, 
\begin{eqnarray*}
	\underbrace{c(t_\delta, \delta) \log(M(t_\delta, \delta)) + f(\log(1/\delta)^{1/\veps}) + c(t_\delta, \delta) \lp \log(t_\delta\log(t_\delta)) + \alpha_0  \rp}_{b(t_\delta)} \underset{\delta \rightarrow 0}{\sim} \log(1/\delta).\end{eqnarray*} 
	
Let $\delta_{\min} \in (0, 1)$ be defined as 
$$ \delta_{\min} := \inf \left\{ \delta \in (0, 1) \mid b(t_\delta) > \log(1/\delta)^{1/\veps} T^\star(\bm \theta)^{-1} \right\}$$
which is well defined as $b(t_\delta) \underset{\delta \rightarrow 0}{\sim } \log(1/\delta)$ and $\veps \in (0, 1)$. Letting $T_{\max} = \log(1/\delta_{\min})^{1/\veps}$, remark that for all $t\geq T_{\max}$, there is $(0, 1) \ni \delta' \leq \delta_{\min}$  $t_\delta' = t$ and $b(t_{\delta'}) < t_{\delta'}$. Therefore, for all $\delta\leq \delta_{\min}$
\begin{equation}
	T(\delta) \leq \log(1/\delta)^{1/\veps}
\end{equation}
and further noting that by definition $T(\delta) \leq T^\star(\bm\theta) b(T(\delta))$ and $b$ is increasing, it follows that 
\begin{eqnarray}
\label{eq:zz-pp}
	T(\delta) \leq T^\star(\bm \theta) b(\log(1/\delta)^{1/\veps}).
\end{eqnarray}

Combining \eqref{eq:zz-mm},  \eqref{eq:zz-oo} and \eqref{eq:zz-pp}, it follows that for $\delta \leq \delta_{\min}$, 
\begin{eqnarray}
	\bE[\tau_\delta^{\text{PS}}] &\leq& T^\star(\bm\theta) b(\log(1/\delta)^{1/\veps}) + 5\pi^2/6 + (t_0 +1) + \sum_{t\geq t_3} (t\log(t))^{-1}. 
\end{eqnarray}
Finally, noting that $b(\log(1/\delta)^{1/\veps}) \underset{\delta \rightarrow 0}{\sim} \log(1/\delta)$, we have proved that 
\begin{equation}
	\limsup_{\delta \rightarrow 0}\frac{\bE[\tau_\delta^\text{PS}]}{\log(1/\delta)} \leq T^\star(\bm \theta). 
\end{equation}
\end{proof}


\section{POSTERIOR CONVERGENCE}
\label{app:posterior_convergence}
In this section, we prove the following result. 
\postConvergence*
We prove that the posterior contraction rate of our algorithm is un-improvable for any adaptive algorithm. 
In BAI, \cite{russo_simple_2016} and \cite{zhaoqi_peps} proved a similar result for truncated Gaussian (restricted to a bounded domain). We prove it more generally in the unbounded setting for Gaussian distribution. 

We recall that when $\bm \lambda$ is a matrix and $\pi$ is a distribution supported on vectors, $\bm \lambda \sim \pi$ denotes $\vv(\bm \lambda) \sim \pi$.
\begin{proof}[Proof of Theorem~\ref{thm:asymptotic_rate_convergence}]
	We first  prove the upper bound. Similarly to  Lemma~\ref{lem:prob_to_glr}, we can derive (using Lemma~\ref{lem:alt_convexity} and Lemma~\ref{lem:mvn-upb}), 
	$$ \bP_{\wt \Pi_t \mid \cF_t}\lp\Theta \cap \alt(S^\star)\rp \leq \alpha_0 \exp\lp - \frac{\inf_{\bm \lambda \in \Theta \cap \alt(S^\star)}\| \vv(\wh{\bm\theta}_t - \bm \lambda)\|^2_{\bm \Sigma_t}}{2}\rp$$
	with $\alpha_0 : = \lvert \cZ \rvert (\lvert \cZ \rvert + d^{\lvert \cZ \rvert}) $, 
	
From Theorem~\ref{thm:sdp_convergence}, there exists events  $(\Xi_t)_{t\geq }$ and $t_3 \in \bN$ such that for $t\geq t_3$, $\wh S_t = S^\star$, and  
\begin{eqnarray}
	2\text{GLR}(t) := \inf_{\bm \lambda \in \Theta \cap \alt(\wh S_t)} \left\|\vv( \bm \lambda - \wh{\bm\theta}_{t}) \right\|_{\Sigma^{-1} \otimes V_{t}}^2 
	&\geq & t \max_{\w \in \Delta} \inf_{\bm \lambda \in \Theta \cap \alt(S^\star(\bm\theta))} \left\|\vv(\bm\theta - \bm\lambda)\right\|_{\Sigma^{-1} \otimes V_{\w}}^2 - 2f(t), \end{eqnarray} with $f(t) =_{\infty} o(t)$ and $\bP_{\bm \nu}(\Xi_t) \geq 1 - 5/t^2$. Since $\sum_{t\geq 1} \bP(\Xi_t^c) < \infty$, by Borel-Cantelli's lemma, with probability 1, there exists a finite time $\wt t_3$ (possibly) random such that for $t\geq \wt t_3$, $\Xi_t$ holds. So for $t\geq \max(t_3, \wt t_3)$, we have $\wh S_t = S^\star$ and  
$$ \frac12 \inf_{\bm \lambda \in \Theta \cap \alt(S^*)} \left\|\vv( \bm \lambda - \wh{\bm\theta}_{t}) \right\|_{\Sigma^{-1} \otimes V_{t}}^2 = \text{GLR}(t)
	\geq t T^\star(\bm \theta)^{-1} - f(t), $$
	then \begin{eqnarray*}
		\bP_{\wt \Pi_t\mid \cF_t}(\Theta \cap \alt(S^\star)) \leq \alpha_0 \exp\lp - tT^\star(\bm \theta)^{-1}  + f(t)\rp, 
	\end{eqnarray*}
	so that 
	\begin{eqnarray*}
		- \frac{1}{t} \log(\bP_{\wt \Pi_t\mid \cF_t}(\Theta \cap \alt(S^\star))) \geq T^\star(\bm \theta)^{-1} -(1/t)\log(\alpha_0) - f(t)/t, 
	\end{eqnarray*}
	then, since $f(t) =_{\infty} o(t)$, put together, the above displays show that with probability 1, 
	$$ \liminf_{t\rightarrow \infty} -\frac{1}{t} \log(\bP_{\wt \Pi_t\mid \cF_t}(\Theta \cap \alt(S^\star))) \geq T^\star(\bm \theta)^{-1}.$$ 
	
The proof of the lower bound uses Lemma~\ref{lem:asymptotic_posterior_concentration}. 
		In the transductive linear setting, when $\Theta$ is convex $\alt(S^*) \cap \Theta$ is countably convex (by Lemma~\ref{lem:alt_convexity}) and bounded. Then taking $O =  \alt(S^*) \cap \Theta$ and applying Lemma~\ref{lem:asymptotic_posterior_concentration} yields that with probability 1, 
		\begin{eqnarray*}
			\limsup_{t\rightarrow \infty}-\frac1t \log\bP_{\wt \Pi_t\mid \cF_t}(\alt(S^*) \cap \Theta) &\leq& \frac12 \sup_{\w \in \D}\inf_{\bm \lambda \in \alt(S^*) \cap \Theta} \|\vv({\bm\theta} - \bm\lambda) \|_{\Sigma^{-1} \otimes V_{\w }}^2 \\
			&=& T^*(\bm \theta)^{-1}. 
		\end{eqnarray*}

In the unstrucutured setting, let $\bB_\epsilon$ be the ball centered on $\bm \mu$ and with radius $\veps$ as in Lemma~\ref{lem:inf_bounded}. We have 
$-\bP(\alt(S^*)) \leq -\bP(\alt(S^*) \cap \bB_\epsilon)$, then $O = \alt(S^*) \cap \bB_\epsilon$ is countably convex (Lemma~\ref{lem:alt_convexity} and $\bB_\epsilon$ is convex) and bounded. Applying Lemma~\ref{lem:asymptotic_posterior_concentration}, we obtain 
\begin{eqnarray*}
	\limsup_{t\rightarrow \infty}-\frac1t \log\bP_{\wt \Pi_t\mid \cF_t}(\alt(S^*)) &\leq& \limsup_{t\rightarrow \infty}-\frac1t \log\bP_{\wt \Pi_t\mid \cF_t}(\alt(S^*) \cap \bB_\epsilon) \\
	&\leq& \frac12 \sup_{\w \in \D} \inf_{\bm \lambda \in \alt(S^*) \cap \bB_\epsilon} \|\vv({\bm\theta} - \bm\lambda) \|_{\Sigma^{-1} \otimes V_{\w }}^2 \\
	&=& \frac12 \sup_{\w \in \D} \inf_{\bm \lambda \in \alt(S^*)} \|\vv({\bm\theta} - \bm\lambda) \|_{\Sigma^{-1} \otimes V_{\w }}^2 \quad \text{ (Lemma~\ref{lem:inf_bounded})} \\
	&=& T^*(\bm \theta)^{-1}, 
\end{eqnarray*}
which concludes the proof. 
	\end{proof}

\begin{lemma}\label{lem:asymptotic_posterior_concentration}
	Let $O \subset \bR^{h \times d}$ be countably convex bounded set, with non-empty interior. With probability 1, it holds that $$ 
	\limsup_{t\rightarrow \infty}-\frac1t \log\bP_{\wt \Pi_t\mid \cF_t}(O) \leq \frac12 \sup_{\w \in \D}\inf_{\bm \lambda \in O} \|\vv({\bm\theta} - \bm\lambda) \|_{\Sigma^{-1} \otimes V_{\w }}^2.$$
\end{lemma}

\begin{proof}
Since $O$ is countably there exists convex sets $C_1 \dots C_n$
 such that $O = \cup_{i\in [n]} C_i$. 
 
 We have $$ \bP_{\wt \Pi_t\mid \cF_t}(O)  =  (2\pi)^{-dh/2} \det(\bm \Sigma_t)^{-1/2} \int_{\vv(O)} \exp\lp - \frac{\| \vv(\wh{\bm\theta}_{t}) -\lambda \|^2_{\wt V_t }}{2}   \rp d\lambda.$$
 
 Similarly to the proof of the regret of the min learner, let $\gamma>0$ (to be defined) and $\wt{\bm\lambda}_t \in \argmin_{\bm\lambda \in \overline{O}} \|\vv(\bm\lambda - \wh{\bm\theta}_{t})\|^2_{\wt V_t}$. Since $O$ is a union of convex sets, there exists a convex set $\cC \C {O}$ such that $\wt{\bm\lambda}_t \in \cC$. Then letting $\cN_\gamma := \{ (1-\gamma)\wt{\bm\lambda}_t + \gamma  \bm\lambda, \bm\lambda \in \cC \} = (1-\gamma)\wt {\bm\lambda}_t + \gamma \cC$, it follows 
\begin{eqnarray*}
	w_t := \int_{\vv(O)} \exp\lp - \frac{\| \vv(\wh{\bm\theta}_{t}) -\lambda \|^2_{\wt V_t }}{2}   \rp d\lambda  &\geq&{\int_{\vv(\cC)} \exp\lp - \frac{\| \vv(\wh{\bm\theta}_{t}) -\lambda \|^2_{\wt V_t }}{2}  \rp d\lambda} \\
	\\&\geq& 
	 {\int_{\vv(\cN_\gamma)} \exp\lp - \frac{\eta \|\vv(\wh{\bm\theta}_{t}) - \lambda \|^2_{\wt V_t }}{2}  \rp d\lambda} \quad \text{(convexity of $\cC$)}\\
	 &=& {\int_{\gamma \vv(\cC)} \exp\lp - \frac{\|\vv(\wh{\bm\theta}_{t} -  (1-\gamma)\wt{\bm\lambda}_t ) - \lambda \|^2_{\wt V_t }}{2}  \rp d\lambda}\\
	 &=&  {\int_{\vv(\cC)} \gamma^{d h}\exp\lp - \frac{\| (1-\gamma)\vv(\wh{\bm\theta}_{t} - \wt{\bm\lambda}_t )  + \gamma \vv(\wh{\bm\theta}_t)  - \gamma \lambda \|^2_{\wt V_t }}{2}  \rp d\lambda}
	 \\&\geq&  
	{\int_{\vv(\cC)} \gamma^{d h}\exp\lp - \frac{\lp (1-\gamma)\|\vv(\wh{\bm\theta}_{t} - \wt{\bm\lambda}_t) \|_{\wt V_t}^2 + \gamma \|\lambda  - \vv(\wh{\bm\theta}_{t}) \|^2_{\wt V_t } \rp}{2}   \rp d\lambda} 
\end{eqnarray*} therefore, 
\begin{eqnarray}
\label{eq:sh-ert}
\log w_t \geq d h \log(\gamma) - \frac{(1-\gamma)}{2} \|\vv(\wh{\bm\theta}_{t} - \wt{\bm\lambda}_t) \|_{\wt V_t}^2 -  \frac{\eta \gamma \bE_{\bm \lambda \sim \cU(\cC) } \lsb \|\vv(\bm\lambda - \wh{\bm\theta}_{t}) \|^2_{\wt V_t }\rsb }{2} + {\vol(\vv(\cC_*))}\end{eqnarray} where $\cC_*$ is the set of minimum volume among $C_1,\dots C_n$. We recall 
$$ \Xi_{1, t} := \lb  \forall s \leq t, \left\| \vv(\bm \theta - \wh{\bm\theta}_s) \right\|_{\wt V_s}^2 \leq \beta(t, 1/t^2) =: f_1(t) \rb, $$ where $\beta(t, \delta)$ is defined as in  Lemma~\ref{lem:concentration_threshold}, and $\bP_{\bm \nu}(\Xi_{1, t}) \geq 1- 1/t^2$. Moreover, $f_1(t)$ is logarithmic in $t$. Since $\sum_{t\geq 1} \bP(\Xi_t^c) < \infty$, by Borel-Cantelli's lemma, with probability 1, there exists a finite time $\wt t$ (possibly) random such that for $t\geq \wt t$, $\Xi_{1, t}$ holds. 

Taking $\gamma = 1/t$ in Equation~\eqref{eq:sh-ert}, and for $t\geq \wt t$, we get (after simplification) 
	\begin{eqnarray*}
		\log w_t &\geq& - dh \log(t)  - \frac{(1-\gamma)}{2} \|\vv(\wh{\bm\theta}_{t} - \wt{\bm\lambda}_t) \|_{\wt V_t}^2 - (\sqrt{t L(O)} + \sqrt{f_1(t)})^2/t  + {\vol(\vv(\cC_*))},  
	\end{eqnarray*}
	where $L(O) = \max_{a \in \cA, \bm \lambda \in O} \| \vv(\bm \lambda - \bm \theta)\|^2_{\Sigma^{-1} \otimes a a^\T}$. Now by convex inequlity, we have for any $\bm \lambda \in O$, 

\begin{eqnarray*}
	\|\vv(\wh{\bm\theta}_{t} - \bm\lambda) \|_{\wt V_t}^2 &\leq&  \|\vv({\bm\theta} - \bm\lambda) \|_{\wt V_t}^2  + 2\vv(\bm \lambda - \wh{\bm\theta}_t)^T\wt V_t\vv(\bm \theta - \wh{\bm\theta}_t) \\
	&\leq& \|\vv({\bm\theta} - \bm\lambda) \|_{\wt V_t}^2 + 2\|\vv(\wh {\bm\theta}_t - \bm\lambda) \|_{\wt V_t} \|\vv({\bm\theta} - \wh{\bm\theta}_{t}) \|_{\wt V_t} \\
	&\leq& \|\vv({\bm\theta} - \bm\lambda) \|_{\wt V_t}^2 + 2 f_1(t) + 2\sqrt{f_1(t)}\|\vv({\bm\theta} - \bm\lambda) \|_{\wt V_t} \\
	&\leq& \|\vv({\bm\theta} - \bm\lambda) \|_{\wt V_t}^2 + 2 f_1(t) + 2\sqrt{f_1(t)}\sqrt{tL(O)}, 
\end{eqnarray*}
thus we have 
\begin{eqnarray*}
\|\vv(\wh{\bm\theta}_{t} - \wt{\bm\lambda}_t) \|_{\wt V_t}^2&\leq& \inf_{\bm \lambda \in O} \|\vv({\bm\theta} - \bm\lambda) \|_{\wt V_t}^2  + 2 f_1(t) + 2\sqrt{f_1(t)}\sqrt{tL(O)}.
\end{eqnarray*}
Put together, we have proved that 
\begin{eqnarray*}
	- \log \bP_{\wt \Pi_t\mid \cF_t}(O) \leq \frac12 \inf_{\bm \lambda \in O}  \|\vv({\bm\theta} - \bm\lambda) \|_{\wt V_t}^2 + 2 f_1(t) + 2\sqrt{f_1(t)}\sqrt{tL(O)} +  (dh/2)\log(2\pi) +\frac12 \log\det(\bm \Sigma_t) +  dh \log(t) \\
	+ (\sqrt{t L(O)} + \sqrt{f_1(t)})^2/t  - {\vol(\vv(\cC_*))}.
\end{eqnarray*}
We recall that $\bm \Sigma_t = \Sigma \otimes V_t^{-1}$ and $\wt V_t = \bm \Sigma_t^{-1}$ with $V_t = V_{\bm N_t} + V_0$  so $\log\det(\bm \Sigma_t) \leq \log(\det(\Sigma \otimes V_0)) = cste$ (due to initialization or regularization). Combining the displays above yield, 
$$ - \log \bP_{\wt \Pi_t\mid \cF_t}(O)  \leq \frac12 \inf_{\bm \lambda \in O}  \|\vv({\bm\theta} - \bm\lambda) \|_{\wt V_t}^2 + o(t),$$
which finally yields, 
\begin{eqnarray*}
-1/t\log\bP_{\wt \Pi_t\mid \cF_t}(O) &\leq& \frac12 \inf_{\bm \lambda \in O} \|\vv({\bm\theta} - \bm\lambda) \|_{\wt V_t/t}^2 + o(t) / t, 	\\
&=& \frac12 \inf_{\bm \lambda \in O} \|\vv({\bm\theta} - \bm\lambda) \|_{\Sigma^{-1} \otimes V_{\bm N_t/t}}^2 + o(t) / t,\\
&\leq& \frac12 \sup_{\w \in \D}\inf_{\bm \lambda \in O} \|\vv({\bm\theta} - \bm\lambda) \|_{\Sigma^{-1} \otimes V_{\w }}^2 + o(t) / t,
\end{eqnarray*}
 and taking the limit yields the claimed statement. 
\end{proof}


\section{CONCENTRATION RESULTS}
\label{app:concentration}

\subsection{Concentration of Good Event}

\concentrationThreshold*

\subsubsection{Unstructured Setting}
\label{app:ssec_concentration_unstructured}

In the unstructured setting, we take $\xi = 0$ and denote the empirical mean by $\bm{\widehat{\mu}}_{t}$ instead of $\bm{\widehat{\theta}}_{t}$ with $\widehat{\mu}_{t,i} = \frac{1}{N_{t,i}} \sum_{s \in [t-1]} \indi{I_s = i} X_s$ where $I_t$ denote the index associated to the arm $a_t$ pulled at time $t$.

The set of multi-variate Gaussian distributions with known covariance matrix $\Sigma$ is a $d$-dimensional canonical exponential family admitting $\bR^d$ as set of possible natural parameters.
Let us denote by $\theta_{i} \eqdef \Sigma^{-1} \mu_{i}$ the natural parameter associated to $\mu_{i}$.
The associated log-partition function is defined as $\phi_i(\theta_i) \eqdef \frac{1}{2}\theta_{i}\transpose \Sigma \theta_{i}$ and satisfies that $\nabla \phi_i (\theta_i) =  \Sigma \mu_i$ and $\nabla^2 \phi_i(\mu_i) =  \Sigma$. 
We note that $\widehat{\mu}_{t,i}$ is the mean parameter associated with the natural parameter $(\nabla \phi_i)^{-1}(\frac{1}{N_{t,i}} \sum_{s \in [t-1]} \indi{I_s = i} X_s) = \Sigma^{-1}\widehat{\mu}_{t,i}$.
Therefore, we can use the concentration results derived in Chapter 4 of~\citet{degen2019}.
In particular, we will be using the generalization to $K$ arms derived in~\citet{pmlr-v201-jourdan23a}.
\begin{lemma}[Lemma 39 in~\citet{pmlr-v201-jourdan23a} based on~\citet{degen2019}] \label{lem:kl_bound_with_prior_all_n_multi_arm}
    Let $\{\rho_{0,i}\}_{i \in [K]} \subseteq \cP(\bR^d)$. 
    With probability $1 - \delta$, for all $t \in \mathbb{N}$, 
    \begin{align*}
    \sum_{i \in [K]} \frac{N_{t,i}}{2} \|\widehat{\mu}_{t,i} - \mu_i\|^2_{\Sigma^{-1}}
    \le \log(1/\delta) - \sum_{i \in [K]} \ln \mathbb{E}_{y \sim \rho_{0,i}}\exp\left( - \frac{N_{t,i}}{2}  \|\widehat{\mu}_{t,i} - y\|^2_{\Sigma^{-1}}\right) 
    \: .
    \end{align*}
\end{lemma}

Lemma~\ref{lem:concentration_unstructured} gives the concentration threshold for the unstructured setting.
The proof is obtained by using the method from~\citet{degen2019}: peeling argument with a sequence of Gaussian priors.
\begin{lemma}\label{lem:concentration_unstructured}
    Let $s > 1$, $\overline{W}_{-1}$ as in Lemma~\ref{lem:property_W_lambert} and $\zeta$ be the Riemann $\zeta$ function.
    With probability $1 - \delta$, for all $t \in \bN$
    \begin{align*}
    \sum_{i \in [K]} \frac{N_{t,i}}{2}  \|\widehat{\mu}_{t,i} - \mu_i\|^2_{\Sigma^{-1}}
    \le \frac{dK}{2} \overline{W}_{-1}\left(\frac{2}{dK} \log \frac{e^{Ks} \zeta(s)^K}{\delta}  + \frac{2 s}{d} \log \left( 1 + \frac{d}{2s}\log \frac{t-1}{K} \right) + 1 \right) \: .
    \end{align*}
\end{lemma}
\begin{proof}
Let $\gamma > 1$, $\eta > 0$ to be defined later.
let $(k_i)_{i \in [K]} \in \mathbb{N}^{K}$ and $n_k \eqdef \gamma^k$ for $k \in \mathbb{N}$.
Let $(M_{0,i})_{i \in [K]}$ be positive definite matrices to be defined later.
Let us define the prior $\rho_{0,i} = \cN(\mu_i, M_{0,i}^{-1})$ for all $i \in [K]$.
Using the computations in the proof of Lemma 4.23 in~\citet{degen2019}, we obtain, for all $i \in [K]$,
\begin{align*}
    &\mathbb{E}_{y \sim \rho_{0,i}} \exp\left( - \frac{N_{t,i}}{2}    \|\widehat{\mu}_{t,i} - y\|^2_{\Sigma^{-1}}\right) = \sqrt{\frac{\textrm{det}(M_{0,i})}{\textrm{det}(N_{t,i} \Sigma^{-1} + M_{0,i})}} \exp\left( - \frac{1}{2}\|\widehat{\mu}_{t,i} - \mu_i\|^2_{(\Sigma/N_{t,i} +  M_{0,i}^{-1})^{-1}}\right) \: .
\end{align*}
For all $t \in \bN$, for all $i$, there exists $k_i$ such $N_{t,i} \in [n_{k_i}, n_{k_i + 1})$.
Let $(k_i)_{i \in [K]}$ be such vector of indices.
In the following, we consider $M_{0,i}^{-1} \eqdef \frac{1}{n_{k_i} \eta} \Sigma$, hence we have
\begin{align*}
    &\sqrt{\frac{\textrm{det}(M_{0,i})}{\textrm{det}(N_{t,i} \Sigma^{-1} + M_{0,i})}} = \left( 1+\frac{N_{t,i}}{n_{k_i} \eta}\right)^{-d/2}  \ge \left(1  + \frac{1}{\eta} \right)^{-d/2} \gamma^{-d/2}\: ,\\
    &(\Sigma/N_{t,i} +  M_{0,i}^{-1})^{-1} = N_{t,i}(1 +  \frac{N_{t,i}}{n_{k_i} \eta})^{-1} \Sigma^{-1} \preccurlyeq N_{t,i} \left( 1 -  \frac{1}{\eta + 1}  \right) \Sigma^{-1} \: .
\end{align*}
where we used that $N_{t,i} \ge n_{k_i}$, $N_{t,i} \le n_{k_i} \gamma$ and $\gamma > 1$.
Therefore, we obtain
\begin{align*}
    &\mathbb{E}_{y \sim \rho_{0,i}} \exp\left( - \frac{N_{t,i}}{2}    \|\widehat{\mu}_{t,i} - y\|^2_{\Sigma^{-1}}\right) \ge \left(1  + \frac{1}{\eta} \right)^{-d/2} \gamma^{-d/2} \exp\left( - \frac{N_{t,i}}{2}  \left( 1 -  \frac{1}{\eta + 1}  \right)  \|\widehat{\mu}_{t,i} - \mu_i\|^2_{\Sigma^{-1}}\right) \: .
\end{align*}
We use Lemma~\ref{lem:kl_bound_with_prior_all_n_multi_arm} with prior $\rho_{0,i} = \cN(\mu_i, M_{0,i}^{-1})$.
By taking the logarithm, summing over $i \in [K]$ and re-ordering terms, we obtain that, with probability $1 - \delta$, for all $t \in \mathbb{N}$, if $\bm N_{t} \in \bigotimes_{i \in [K]} [n_{k_i}, n_{k_i + 1})$, then
\begin{align*}
    &\sum_{i \in [K]} \frac{N_{t,i}}{2} \|\widehat{\mu}_{t,i} - \mu_i\|^2_{\Sigma^{-1}}
    \le (\eta + 1)\left( \log(1/\delta)  + \frac{dK}{2} \log \gamma + \frac{dK}{2} \log \left( 1 + \frac{1}{\eta}\right)\right) 
\end{align*}
To choose an optimal value for $\eta$, we rely on Lemma~\ref{lem:lemma_A_3_of_Remy} taken from~\citet{degen2019}. 
\begin{lemma}[Lemma A.3 in \citet{degen2019}] \label{lem:lemma_A_3_of_Remy}
    For $a,b\geq 1$, the minimal value of $f(\eta)=(1+\eta)(a+\ln(b+\frac{1}{\eta}))$ is attained at $\eta^\star$ such that $f(\eta^\star) \leq 1-b+\overline{W}_{-1}(a+b)$.    If $b=1$, then there is equality.
\end{lemma}
Then, we have that, with probability $1 - \delta$, for all $t \in \mathbb{N}$, if $\bm N_{t} \in \bigotimes_{i \in [K]} [n_{k_i}, n_{k_i + 1})$, then
\begin{align*}
    &\sum_{i \in [K]} \frac{N_{t,i}}{2} \|\widehat{\mu}_{t,i} - \mu_i\|^2_{\Sigma^{-1}}
    \le \frac{dK}{2} \overline{W}_{-1}\left( 1 + \frac{2}{dK} \log(1/\delta)  + \log \gamma \right) 
\end{align*}
Let us define $w_{k} = \frac{1}{\zeta(s)(k+1)^s}$. Then, we have 
\[
    \sum_{(k_i)_{i \in [K]} \in \bN^{K}} \prod_{i \in [K]} w_{k_i} = \left(\sum_{ k \in \bN} w_{k}\right)^K = \left(\frac{1}{\zeta(s)}\sum_{ k \in \bN} \frac{1}{(k+1)^s}\right)^K= 1 \: .
    \]

Using that $\bigcup_{(k_i)_{i \in [K]} \in \bN^{K}}  \bigotimes_{i \in [K]} [n_{k_i}, n_{k_i + 1}) = \bN^K$, we can apply the above results for each $(k_i)_{i \in [K]} \in \bN^{K}$ with probability $1 - \delta \prod_{i \in [K]} w_{k_i}$.
When $N_{t,i} \in  [n_{k_i}, n_{k_i + 1})$, we have $k_i \le \log(N_{t,i})/\log(\gamma)$.
Therefore, a direct union bound yields that, with probability $1 - \delta$, for all $t \in \mathbb{N}$,
\begin{align*}
    &\sum_{i \in [K]} \frac{N_{t,i}}{2} \|\widehat{\mu}_{t,i} - \mu_i\|^2_{\Sigma^{-1}}
    \\
    &\le \frac{dK}{2} \overline{W}_{-1}\left( 1 + \frac{2}{dK} \log(1/\delta) + \frac{2}{d} \log \zeta(s) + \frac{2 s}{dK} \sum_{i \in [K]} \log (\log (\gamma) +\log N_{t,i}) - \frac{2 s}{d} \log\log(\gamma) + \log \gamma \right) \: .
\end{align*}
The function $f: \gamma \to \log \gamma - \frac{2 s}{d} \log\log \gamma  $ is minimized at $\gamma^\star = e^{2s/d}$ with value $f(\gamma^\star) = 2s(1- \log (2s/d))/d$.
Taking $\gamma^\star$ and using the concavity of $x \to \log \left(\frac{2s}{d} + \log x \right)$ concludes the proof since we have
\[
    \frac{2 s}{dK} \sum_{i \in [K]} \log \left(\frac{2s}{d} +\log N_{t,i} \right) \le \frac{2 s}{d} \log \left(\frac{2s}{d} +\log \frac{t-1}{K} \right) \: .
\]
\end{proof}

\subsubsection{Structured Setting}
\label{app:ssec_concentration_structured}
 We prove the following result. 
\begin{lemma}
\label{lem:cineq_strcutured} Letting 
\begin{equation*} \label{eq:concentration_threshold_structured2}
    \sqrt{\beta(t, \delta)}  = \sqrt{\log\left(\frac{1}{\delta}\left(\frac{L_{\cA}^2}{h \xi}t + 1 \right)^{\frac{dh}{2}}\right)} + \sqrt{\frac{d L_{\cM}^2}{2 \lambda_{\min}(\Sigma) \xi}} \:, 
\end{equation*}
where $L_{\cA} \eqdef \max_{a \in \cA} \|a\|_2$ and $L_{\cM} \eqdef \max_{\bm \lambda \in \Theta} \max_{c \in [d]} \|\bm \lambda e_c \|_2$,  the event $$ \cE_{\delta} := \left\{ \forall t\geq 1,\; \frac12 \left\| \vv(\bm\theta - \wh{\bm\theta_t}) \right\|_{\Sigma^{-1} \otimes (V_{\bm N_t} + \xi I_h)}^2 \leq \beta(t-1, \delta)\right\} $$ holds with probability at least $1-\delta$.
\end{lemma}
\begin{proof}
At time $t$, $a_t$ is chosen from $[K]$ (adaptively) and a random vector $X_t := (I_d \otimes a_t)\vv(\bm\theta) + \veps_t$ is observed and $\veps, \veps_1, \dots \veps_t$ are  centered \iid $\Sigma$ subgaussian random vectors, i.e. for all $u \in \bR^d$, 
\begin{equation}
\label{eq:eq-subg}
\bE\lsb\exp(u^\T\veps)\rsb \leq  \exp \lp \frac12 \left\|  u\right \|^2_{\Sigma}\rp. 
\end{equation}
Let us define $S_t := \sum_{s=1}^t \vv{(a_s \veps_s^\T)}$, $\bm A_t := (a_1 \dots a_t)^\T$, $\bm X_t := (X_1 \dots X_t)$ and $\bm H_t := (\veps_1 \dots \veps_t)^\T$. The regularized least-squares estimator of $\bm \theta$ (for the Frobenius norm ) is given by $$\wh{\bm \theta}_t  =  (V_{\bm N_t} + \xi I_h)^{-1} \bm A_t ^\T \bm X_t.$$
We introduce sligthly modified filtration that make $a_t$ precitbible (as $\wt \w$ is and an agent can pull $a_{t+1}$ at the end of round $t$). We define $\cF_t = \sigma(\cF_t, a_{t+1})$. We define the following $t$-indexed stochastic process for any $\lambda \in \bR^{d\cdot h}$
 $$ M^\lambda(t) := \exp\left(\lambda^\T S_t - \frac12 \left\|\lambda \right\|^2_{\Sigma \otimes V_{\bm N_t}}\right).$$ 
 
 We first justify that $M^\lambda(t)$ is a super-martingale. Indeed, $M^\lambda(t)$ is an $\{ \wt \cF_t\}_{t\geq 1}$-adapted process and
 \begin{equation}
 \label{eq:eqq-opd}
 \bE\lsb M^\lambda(t+1) \mid \wt \cF_t \rsb = M^\lambda(t) \bE\lsb \exp\lp \lambda^\T \vv({a_{t+1}} \veps_{t+1}^\T) - \frac12 \left\| \lambda \right \|_{\Sigma \otimes {a_{t+1}} {a_{t+1}}^\T}^2\rp \left. \mid \right. \wt \cF_t \rsb,	
 \end{equation}
then, note that by properties of Kronecker product, $\vv(x\veps^\T)  =  (I_d \otimes x) \veps $, so 
 \begin{equation}
 \lambda^\T \vv({a_{t+1}} \veps_{t+1}^\T) = \lambda^\T (I_d \otimes {a_{t+1}}) \veps_{t+1}, 	
 \end{equation} and \begin{eqnarray*}
 	\lambda^\T (I_d \otimes a) \Sigma (\lambda^\T (I_d \otimes a))^\T &=& \lambda^\T (I_d \otimes a) \Sigma (I_d \otimes a^\T) \lambda   \\
 	&=& \frac{1}{a a^\T} \lambda^\T (I_d \otimes a) (\Sigma \otimes a^\T a) (I_d \otimes a^\T) \lambda \\
 	&=& \frac{1}{a a^\T} \lambda^\T (\Sigma \otimes (a a^\T x a^\T)) \lambda = \lambda^\T \lp\Sigma \otimes aa^\T \rp \lambda 
 \end{eqnarray*}
 which follows from the property $(A\otimes B) (C\otimes D) = (AC) \otimes (BD)$, proving that  
$$ \left \|(I_d \otimes {a_{t+1}}^\T)  \lambda \right\|_{\Sigma}^2 = \left \| \lambda \right\|_{\Sigma \otimes ({a_{t+1}} {a_{t+1}}^\T)}^2.$$

 Thus, combining the above display with \eqref{eq:eqq-opd} and the sub-gaussian property of $\veps_{t+1}$ shows that 
 $$\bE\lsb M^\lambda(t+1) \mid \wt \cF_t \rsb \leq  M^\lambda(t), $$ 
 so $(M^\lambda(t))_t$ is a super-martingale. For $\xi>0$ and $\Lambda \sim \cN(0_{dh},\Sigma^{-1} \otimes (\xi I_h)^{-1})$, we define   
 $$ M(t) := \bE\lsb M^\Lambda(t) \lvert \wt \cF_\infty \rsb.$$
Letting $U_t := \Sigma \otimes V_{\bm N_t}$ and defining $c(P) := \sqrt{(2\pi)^{d\cdot h})/det(P)} = \int_{\bR^{d\cdot h}}\exp(-\frac12 \lambda^\T P \lambda)d\lambda$, we have 
 \begin{eqnarray*}
 	 M(t) &=& \frac{1}{c(\Sigma \otimes \xi I_h)} \int \exp\left(\lambda^\T S_t - \frac12 \lambda^\T U_t \lambda - \frac{1}2 \lambda^\T (\Sigma \otimes (\xi I_h) )\lambda \right) d\lambda \\
 	 &=& \frac{1}{c(\Sigma \otimes \xi I_h)} \exp\left(\frac12S_t^\T W_t^{-1} S_t \right) \int \exp\left(-\frac12 \left(\lambda  -W_t^{-1}S_t\right)^\T  W_t \left(\lambda  -W_t^{-1}S_t\right)\right) d\lambda 
 \end{eqnarray*}
 where $W_t := U_t + \Sigma \otimes (\xi  I_h)  = \Sigma \otimes (V_{\bm N_t} + \xi I_h)$ then direct algebra yields 
 \begin{eqnarray*}
 	 M(t)  &=& \frac{c(\Sigma \otimes (V_{\bm N_t} +\xi I_h))}{c(\Sigma \otimes \xi I_h)} \exp\left(\frac12\|S_t\|^2_{W_t^{-1}}\right)\\
 	 &=&  \left(\frac{\det(\Sigma \otimes \xi I_h)}{\det(\Sigma \otimes (V_{\bm N_t} +\xi I_h))}\right)^{1/2} \exp\left(\frac12\|S_t\|^2_{W_t^{-1}}\right), 
 \end{eqnarray*}
 then it is known that $M_t$ is also a super-martingale (by Fubini's theorem) with $\bE_{\bm \nu}[M_t] \leq 1$ and using classical technique similar to section 20.1 of \cite{lattimore_bandit_2020} on super-martingale, it holds that with probability at least $1-\delta$ we have for all $t\geq 1$, 
 \begin{eqnarray}
 	\|S_t\|^2_{W_t^{-1}} &\leq& 2\log\left(\frac{det(\Sigma \otimes (V_{\bm N_t} +\xi I_h))^{1/2}}{\delta \det(\Sigma \otimes \xi I_h)^{1/2}} \right) \label{eq:norm_st}\\
 	&=& 2 \log\left(\frac{\det(V_{\bm N_t} +\xi I_h)^{d/2}}{\delta \xi^{d \cdot h/2}} \right),\nonumber 
 \end{eqnarray}
 which follows as for $A\in \bR^{p\times p }$ and $B\in \mathbb{R}^{q\times q}$
 $$ \det(A\otimes B) = \det(A)^q \det (B)^p.$$  
 We further have $\wh{\bm \theta}_t := (V_{\bm N_t} + \xi I_h)^{-1} \bm A_t^\T \bm X_t$ then 
 \begin{eqnarray*}
 	\wh{\bm \theta}_t  &=&  (V_{\bm N_t} + \xi I_h)^{-1} \bm A_t^\T(\bm A_t\bm\theta + \bm H_t)\\
 	&=& (V_{\bm N_t} + \xi I_h)^{-1} V_{\bm N_t}\bm\theta + (V_{\bm N_t} + \xi I_h)^{-1} \bm A_t^\T \bm H_t\\
 	&=& \bm\theta - \xi (V_{\bm N_t} + \xi I_h)^{-1}\bm\theta + (V_{\bm N_t} + \xi I_h)^{-1} \bm A_t^\T \bm H_t.
 \end{eqnarray*}
 We recall that for any $B,C$ well defined and $C\in \bR^{p\times q}$
 \begin{eqnarray*}
 	\vv{(BC)} &=& (I_q \otimes B) \vv{(C)}.
 \end{eqnarray*}
 Therefore we can derive 
 \begin{eqnarray*}
 	\vv(\wh{\bm \theta}_t - \bm \theta)
 	&=& -\xi (I_d \otimes (V_{\bm N_t} + \xi I_h)^{-1})\vv(\bm \theta)  + (I_d \otimes (V_{\bm N_t} + \xi I_h)^{-1})S_t.
 \end{eqnarray*}

We recall that for any two invertible matrices $A, B$, $A\otimes B$ is invertible and $$(A\otimes B)^{-1} = A^{-1} \otimes B^{-1}$$ and for any matrices $A, B, C, D$ when the product is possible 
 $$ (A\otimes B) (C\otimes D) = (AC)\otimes (BD)$$
 and $(A\otimes B)^\T = (A^\T \otimes B^\T)$. Letting $Z_t := I_d \otimes (V_{\bm N_t} + \xi I_h)$ and for any $x\in \bR^{d h}$, 
 \begin{eqnarray*}
 	x^\T\vv(\wh{\bm \theta_t}- \bm \theta)&=& - \xi x^\T Z_t^{-1}\vv(\bm \theta) + x^\T Z_t^{-1}S_t\\
 	&=& \xi x^\T (\Sigma \otimes I_h)W_t^{-1}\vv(\bm \theta) + x^\T (\Sigma \otimes I_h)W_t^{-1} S_t
 \end{eqnarray*}
  where $W_t = \Sigma \otimes (V_{\bm N_t} + \xi I_h)$, then thnaks to Cauchy-Schwartz inequality  
 \begin{eqnarray*}
 	\lvert x^\T\vv(\widehat{\bm \theta_t} - \bm \theta) \rvert &\leq&\left\| W_t^{-1/2} (\Sigma\otimes I_h)x \right\| \left(\xi \left\| \vv(\bm \theta) \right\|_{W_t^{-1}}  + \left\| S_t \right\|_{W_t^{-1}}\right) \\
 	&=& \left \| \Sigma^{1/2}\otimes (V_{\bm N_t} + \xi I_h)^{-1/2} x \right\| \left( \xi \left\| \theta \right\|_{W_t^{-1}}  + \left\| S_t \right\|_{W_t^{-1}}\right).
 \end{eqnarray*}
 Letting $$ x := (\Sigma^{-1}\otimes (V_{\bm N_t} + \xi I_h))\vv(\widehat{\bm \theta_t} - \bm \theta),$$
 and plugging back into the last equation yields 
 \begin{eqnarray*}
 	\left\| \vv(\widehat{\bm\theta_t} - \bm \theta ) \right\|^2_{\Sigma^{-1}\otimes (V_{\bm N_t} + \xi I_h)} &\leq& \left\| \Sigma^{-1/2} \otimes (V_{\bm N_t} + \xi I_h)^{1/2} \vv(\wh{\bm \theta_t} - \bm \theta)\right\| \left( \xi \left \| \vv(\bm \theta) \right\|_{W_t^{-1}}  + \left\| S_t \right\|_{W_t^{-1}}\right), \\
 	&=& \left \|\vv(\widehat{\bm \theta_t} - \bm \theta )\right\|_{\Sigma^{-1}\otimes (V_{N_n} + \xi I_h)} \left( \xi \left \| \vv(\bm \theta) \right \|_{W_t^{-1}}  + \left \| S_t \right \|_{W_t^{-1}}\right),
 \end{eqnarray*}
 therefore $$ \left\|\vv(\wh{\bm \theta_t} - \bm \theta)\right\|_{\Sigma^{-1}\otimes (V_{\bm N_t} + \xi I_h)} \leq \xi \left\| \vv(\bm \theta) \right\|_{W_t^{-1}}  + \left \| S_t \right \|_{W_t^{-1}}, $$ 
 with $W_t := \Sigma \otimes (V_{\bm N_t} + \xi I_h)$
 and finally using the bound on \eqref{eq:norm_st}, we prove that with probability at least $1-\delta$, for all $t\geq 1$, 
 \begin{equation}
 	\left \|\vv(\wh{\bm \theta_t} - \bm \theta) \right \|_{\Sigma^{-1}\otimes (V_{\bm N_t} + \xi I_h)} \leq  \xi \left \| \vv(\bm \theta) \right\|_{\Sigma^{-1} \otimes (V_{\bm N_t} + \xi I_h)^{-1}}  + \sqrt{2 \log\left(\frac{det(V_{\bm N_t} +\xi I_h)^{d/2}}{\delta \xi^{d \cdot h/2}} \right)}. 
 \end{equation} that is \begin{equation}
 	\frac12 \left \|\vv(\wh{\bm \theta_t} - \bm \theta) \right \|_{\Sigma^{-1}\otimes (V_{\bm N_t} + \xi I_h)}^2  \leq  \lp \frac{\xi}{\sqrt{2}} \left \| \vv(\bm \theta) \right\|_{\Sigma^{-1} \otimes (V_{\bm N_t} + \xi I_h)^{-1}}  + \sqrt{\log\left(\frac{det(V_{\bm N_t} +\xi I_h)^{d/2}}{\delta \xi^{d \cdot h/2}} \right)}\rp^2.  
 \end{equation}
 The result then follows by simple algebra combined with the determinant-trace inequality (cf Lemma~19.4 of \cite{lattimore_bandit_2020}). 
  \end{proof}
 
 \subsection{Guarantees on the Sampling}
 The following result is well-known, and its proof can be found in \cite{lattimore_bandit_2020, degenne2019non}. 
  \begin{lemma}[Elliptic potential lemma]
  \label{lem:elliptic_potential} For $V_0=\xi I_h, \xi \geq 1$, it holds that $$ \sum_{s=1}^t \| a_s \|_{V_{s-1}^{-1}}^2 \leq 2 h \log\lp  \frac{\Tr(V_0) + t L_{\cA}^2}{d \det(V_0)^{1/d}} \rp.$$ 
\end{lemma}

In the following lemma, for two matrices $A,B$, we write $A\geq B$ iff $A-B$ is positive semi-definite. 
\begin{lemma}
\label{lem:forced_exp}
	Let $\w_{\text{exp}}$ be a distribution on $\cA$ supported on some atoms (not necessarily full support) and define the forced-exploration weights as $\wt \w_t:= (1-\gamma_t) \w_t + \gamma_t \w_{\text{exp}}$. For all $\veps \in (0, 1/2)$ there exists $t_0(\alpha) \in \bN$ and events $\Xi_{2, t}$ (cf \eqref{eq:def_event_2}) such that for $t\geq t_0(\alpha)$, if $\Xi_{2, t}$ holds then $V_{\bm N_t} \geq \frac{t^{1-\alpha}}{2(1-\alpha)} V_{\w_{\text{exp}}}$. Moreover, if $V_{\w_\text{exp}}$ is non-singular then $V_{\bm N_t}^{-1} \leq 2(1-\alpha) t^{\alpha -1} V_{\w_\text{exp}}^{-1}$. 
\end{lemma}
\begin{proof}
	Let us introduce  $D_{t-1}^a := N_{t, a} - \sum_{s=1}^{t-1}\wt \omega_s(a) = \sum_{s=1}^{t-1} (\ind_{(a_s = a) } - \wt \omega_s (a))$ and the following event
	$$ \Xi_{2, t}^a := \lp \lvert D_t \rvert \leq \sqrt{2t \log(2 \lvert \cA \rvert t^2)} \rp.$$
	By simple algebra we $D_t$ is  $\cF_t$ measurable and recalling that $\wt \w_t$ is $\cF_{t-1}$ measurable we have 
\begin{eqnarray*}
	\bE[D_t^a \mid \cF_{t-1}] &=& D_{t-1} + \bE[\ind_{(a_t = a)}- \wt\omega_t(a) \mid \cF_{t-1} ]\\
	&=& 0, 
\end{eqnarray*}
so $(D_t)_{t\geq 1}$ is a $\cF$ martingale which satisfies $\lvert D_t - D_{t-1}\rvert \leq 1$ almost surely. By Azuma's inequality, $\bP_{\bm \nu}(\Xi_{2, t}^a) \geq 1 - \frac{1}{\lvert \cA \rvert t^2}$ and defining 
\begin{equation}
\label{eq:def_event_2}
\Xi_{2, t} := \bigcap_{a \in \cA} \Xi_{2, t}^a, 	
\end{equation}
we have $\bP_{\bm \nu}(\Xi_{2, t}) \geq 1 - 1/t^2$. When $\Xi_{2, t}$ holds, we have for any arm $a\in \cA$, 
\begin{eqnarray*}
	N_{t, a} &\geq& \sum_{s=1}^{t-1}\wt \omega_s(a) - \sqrt{2t \log(2 \lvert \cA \rvert t^2)} \\
	&\geq& \sum_{s=1}^{t-1}\gamma_s \omega_\text{exp}(a) - \sqrt{2t \log(2 \lvert \cA \rvert t^2)},  \\
	&\geq& \frac{(t-1)^{1 - \alpha}}{1 - \alpha}\omega_\text{exp}(a) - \sqrt{2t \log(2 \lvert \cA \rvert t^2)}
\end{eqnarray*}
 so that introducing  \begin{equation}
 	t_0(\alpha) := \inf\lb n: \forall t\geq n, \forall a \in \cA \mid \omega(a)>0, \frac{(t-1)^{1 - \alpha}}{(1 - \alpha)}\omega_\text{exp}(a) - \sqrt{2t \log(2 \lvert \cA \rvert t^2)}  \geq  \frac{t^{1 - \alpha}}{2(1 - \alpha)}\omega_\text{exp}(a) \rb, 
 \end{equation}
 which is  well defined for $\alpha \in (0, 1/2)$,
 we have for $t\geq t_0(\alpha)$, $N_{t, a} \geq \frac{t^{1 - \alpha}}{2(1 - \alpha)}\omega_\text{exp}(a) $, so that   
 \begin{eqnarray*}
 	V_{\bm N_t} &\geq& \frac{t^{1 - \alpha}}{2(1 - \alpha)}\sum_{a \in \cA } \omega_\text{exp}(a) a a^\T \\
 	&=& \frac{t^{1 - \alpha}}{2(1 - \alpha)} V_{\w_\text{exp}}, 
 \end{eqnarray*}
 which is the claimed result and the second part of the statement follows from Lemma~\ref{lem:loewner}. 
\end{proof}


\section{TECHNICALITIES}
\label{app:technicalities}

Appendix~\ref{app:technicalities} gathers existing and new technical results used for our proofs.

Lemma~\ref{lem:property_W_lambert} gathers properties on the function $\overline{W}_{-1}$, which is used in the literature to obtain concentration results.
\begin{lemma}[\citet{pmlr-v201-jourdan23a}] \label{lem:property_W_lambert}
    Let $\overline{W}_{-1}(x)  \eqdef - W_{-1}(-e^{-x})$ for all $x \ge 1$, where $W_{-1}$ is the negative branch of the Lambert $W$ function.
    The function $\overline{W}_{-1}$ is increasing on $(1, +\infty)$ and strictly concave on $(1, + \infty)$.
    In particular, $\overline{W}_{-1}'(x) = \left(1-\frac{1}{\overline{W}_{-1}(x)} \right)^{-1}$ for all $x > 1$.
    Then, for all $y \ge 1$ and $x \ge 1$, $\overline{W}_{-1}(y) \le x$ if and only if $ y \le x - \ln(x)$.
    Moreover, for all $x > 1$, $x + \log(x) \le \overline{W}_{-1}(x) \le x + \log(x) + \min \left\{ \frac{1}{2}, \frac{1}{\sqrt{x}} \right\}$.
\end{lemma}

In general, $\alt(\bm\theta)$ is not convex, but the result below shows that it is countably convex i.e. union of convex sets, and their number depends on the size of the Pareto set, the dimension, and the number of arms.

\begin{lemma} 
\label{lem:alt_convexity} 
For all parameter $\bm \theta$ letting $p := \lvert S^*(\bm \theta)\rvert$, $\alt(\bm \theta)$ can be written as the union  $(p(p-1) + (\lvert \cZ \rvert - p)d^p)$ convex subsets. 
\end{lemma}
\begin{proof}
We let $E_{z,x}(c) = e_c \otimes (z-x) $ and $e_c = (\indi{c'=c})_{c' \in [d]}$ for $z,x\in \cZ$. To further ease notation, we let $\bm E_z = I_d \otimes z^\T$, and $\bm E_{z, x} = \bm E_z - \bm E_x$, which corresponds to row-wise stacking of vectors $(E_{z,x}(c))_c$. It yields  $\bm \theta^\T z  = \bm E_z \vv(\bm \theta)$ and $\bm E_z \vv(\bm \theta) \prec \bm E_x \vv(\bm \theta)$ iff $\bm E_{z-x} \vv(\bm \theta) \prec 0_d$. 

To have a Pareto set different from $S$, either an arm of $S$ should be made sub-optimal or an arm from $\cZ \backslash S$ should be Pareto-optimal. We have 
\begin{eqnarray}
	\bm \lambda \in \alt(\bm \theta) &\equi& S^*(\bm \lambda) \neq S^*(\bm \theta) \\
		&\equi& \exists z \in S^*(\bm \theta) : z \notin S^*(\bm \lambda) \quad \text{ or } \quad \exists z \in  \cZ \backslash S^*(\bm \theta): z \in S^*(\bm \lambda) \\
		&\equi&\lp \exists z,x \in S^*(\bm \theta) : \bm E_{z-x} \vv(\bm \lambda)  \prec 0_d \rp \text{ or } \lp \exists z \in \cZ \backslash S^*(\bm \theta): \bm E_{z-x} \vv(\bm \lambda) \nprec 0_d \; \forall x \in S^*(\bm \theta) \rp \label{eq:azz-wx}.
\end{eqnarray}

To see the last equivalence, assume $z \in S^*(\bm \theta)\backslash S^*(\bm\lambda)$, then there exists $x \in \cZ$ such that $\bm E_z \vv(\bm \lambda) \prec \bm E_x \vv(\bm \lambda)$, i.e. $\bm E_{z-x} \vv(\bm \lambda) \prec 0_d$.
If $x \in S^*(\bm \theta)$  then \eqref{eq:azz-wx} follows. 

 Further assume $x \notin S^*(\bm \theta)$, then, either there exists $x' \in S^*(\bm \theta)$ such that $\bm E_x \vv(\bm \lambda) \prec \bm E_{x'} \vv(\bm \lambda)$ then by transitivity $\bm E_z \vv(\bm \lambda) \prec \bm E_{x'} \vv(\bm \lambda)$ and $z,x' \in  S^*(\bm \theta)$  or for all $x' \in S^*(\bm \theta)$, $\bm E_x \vv(\bm \lambda) \nprec \bm E_{x'} \vv(\bm \lambda)$. Put together we have
 $(z \in S^*(\bm \theta)\backslash S^*(\bm\lambda))$ implies \eqref{eq:azz-wx}. Similarly, for  $z \in S^*(\bm\lambda)\backslash S^*(\bm \theta)$, it is direct by Pareto-optimality (in the  instance $\bm \lambda$) that for all $x\in S^*(\bm \theta), \bm E_z \vv(\bm \lambda) \nprec \bm E_x \vv(\bm \lambda)$. 
 
 Reversely, if \eqref{eq:azz-wx} holds, then in the case $a)$, there exists $z \in S^*(\bm \theta) \backslash S^*(\bm \lambda)$. If $b)$ holds, then one cannot have $S^*(\bm \theta) = S^*(\bm \lambda)$. Indeed if we had $S^*(\bm \theta) = S^*(\bm \lambda)$, then it would follow that 
	there exists $z \in \cZ \backslash S^*(\bm \theta) = \cZ \backslash S^*(\bm \lambda)$ such that for all $z \in S^*(\bm \lambda)$, $\bm E_z \vv(\bm \lambda) \nprec \bm E_z \vv(\bm \lambda)$, that in the instance $\bm \lambda$, there would be a sub-optimal arm that is not dominated by any other Pareto-optimal arm, which is not possible. This concludes the reverse inclusion. 
	Therefore, 
	\begin{equation}
		 \alt(\bm \theta) := \underbrace{\lb \bm \lambda \mid z,x \in S^*(\bm \theta): \bm E_{z-x} \vv(\bm \lambda) \prec 0_d \rb}_{A_{\bm\theta}}\cup \underbrace{\lb \bm \lambda \mid \exists z \in \cZ \backslash S^*(\bm \theta): \forall x \in S^*(\bm \theta), \bm E_{z-x} \vv(\bm \lambda) \nprec 0_d \rb}_{B_{\bm \theta}},
	\end{equation}
	then observe that 
	\begin{equation}
	\label{eq:xxw-cc}
	A_{\bm\theta} := \bigcup_{z\in S^*(\bm \theta)} \bigcup_{x \in S^*(\bm \theta)\backslash\{z\}} \lb \bm \lambda : \bm E_{z-x} \vv(\bm \lambda) \prec 0_d \rb, 	
	\end{equation}
and for $z,x$ fixed, $\lb \bm \lambda \mid  \bm E_{z-x} \vv(\bm \lambda) \prec 0_d \rb$ is a convex set. On the other side, for $z$ fixed, if $(\bm E_{z-x} \vv(\bm \lambda) \nprec 0_d , \forall x \in S^*(\bm \theta))$ holds then there exists  $\bar d^z \in \{ 1, \dots, d\}^{S^*(\bm \theta)}$ such that for all $x\in S^*(\bm \theta)$,  $\langle E_{z,x}(d^z(x)),  \vv(\bm \lambda) \rangle \geq 0 $. 
Putting these displays together yields 
		\begin{equation}
		\label{eq:xxw-gg}
			B_{\bm \theta} = \bigcup_{z \in \cZ \backslash S^*(\bm \theta)} \bigcup_{\bar d^z\in [d]^{S^*(\bm \theta)}} \lb \bm \lambda : \forall x \in S^*(\bm \theta): \langle E_{z,x}(d^z(x)), \vv(\bm \lambda) \rangle \geq 0 \rb, 
		\end{equation}
		then remark that for $z, \bar d^z$ fixed, the set $\lb \bm \lambda : \forall x \in S^*(\bm \theta): \langle E_{z,x}(d^z(x)), \vv(\bm \lambda) \rangle \geq 0 \rb$ is convex.  Finally, combining \eqref{eq:xxw-gg} with \eqref{eq:xxw-cc} yields the claimed result. 
\end{proof}

The following lemma allows to upper bound the probability that a multi-variate normal vector belongs to a convex set. This is proven  in \cite{lu_mvn_upb} where the authors attributed it to \cite{kuelbs_mvn_balls}. 
\begin{lemma} 
\label{lem:mvn-upb}
	Let $\Sigma$ be a variance-covariance matrix, $\mu \in \bR^n$, let $X \sim \cN(\mu, \Sigma)$ and $C \subset \bR^n$, a convex set. Then, it holds that  
	$$ \bP(X \in C) \leq \frac12 \exp\lp - {\inf_{x \in C} \frac12 {\left \| x- \mu \right\|_{\Sigma^{-1}}^2 }} \rp.$$  
\end{lemma}

The followins lemma can be found  in algebra books, we prove it for the sake of self-containedness. 
 For two $n\times n$ matrices $A, B$ we write $A\geq B$ iff $A-B$ is positive semi-definite i.e for all $x\in \bR^n$, $x^\T (A-B)x \geq 0$. 
 
\begin{lemma}
 \label{lem:loewner}
	Let $A, B$ be two symmetric order $n$ non-singular matrices. 
	If $A - B \leq 0$ then $B^{-1} - A^{-1} \leq 0 $. 
\end{lemma}
\begin{proof} 
As  $A$ is symmetric, there exits $P$ such that $A= P^2$ which is denoted by $A^{1/2} = P$. We have 
\begin{eqnarray*}
	A-B \leq 0 &\equi& \forall x \in \bR^n, x^\T A x -x^\T B x \leq 0 \\
	&\equi& \forall x \in \bR^n, (A^{1/2} x )^\T (A^{1/2}x) - (A^{1/2}x)^\T A^{-1/2} B A^{-1/2} (A^{1/2}x)  \leq 0 \\
	&\equi& I_n - \underbrace{A^{-1/2} B A^{-1/2}}_{M} \leq 0 \quad \text{(as $A^{1/2}$ is invertible)} \\  
	&\equi& \forall x \in \bR^n, (M^{1/2}x)^\T M^{-1} (M^{1/2}x) - (M^{1/2}x)^\T  (M^{1/2}x) \leq 0 \quad \text{(as $M$ is symmetric and invertible)} \\
	&\equi& M^{-1} - I_n \leq 0 , 
	\end{eqnarray*}	
then by plugging in the expression of $M$, we have 
\begin{eqnarray*}
	A-B \leq 0 &\equi& (A^{-1/2} B A^{-1/2})^{-1/2})^{-1} -  I_n \leq 0 \\
	&\equi& \forall x \in \bR^n, x^\T A^{1/2} B^{-1} A^{1/2}x - (A^{1/2}x)^\T A^{-1} (A^{1/2}x) \leq 0 \\
	&\equi&  B^{-1} - A^{-1} \leq 0,
\end{eqnarray*}
where the last inequality follows again as $A^{1/2}$ is invertible which follows as $A$ is symmetric non-singular. 
\end{proof}

In the result below, we show that the Euclidean norm is $\alpha$-$\exp$ concave over a bounded domain. 

\begin{lemma}
\label{lem:exp-concave}
	The function $\bm \lambda  \mapsto \exp\lp -\alpha \left\| \vv(\bm \theta' - \bm \lambda) \right\|_{\Sigma^{-1} \otimes aa^\T}^2 \rp $ defined a bounded domain $\bm D$ is concave for $\alpha \in \lp 0, \frac{1}{2 \max_{\bm \lambda \in \bm D} \| \vv( \bm\theta' - \bm \lambda)\|_{\Sigma^{-1} \otimes aa^\T}^2}\rsb$. 
	\end{lemma}
	\begin{proof}
		Let $\wt V = \Sigma^{-1} \otimes aa^\T$ and define $g(\lambda) := \exp(-\alpha \| \theta' -  \lambda \|_{\Sigma^{-1} \otimes aa^\T}^2)$ for $\lambda \in \vv(\bm D)$, with $\theta' := \vv(\bm\theta')$. Direct calculation of the Hessian of $g$ gives 
		$$ \nabla^2 g(\lambda) = \lp- 2\alpha \wt V + 4\alpha^2 \wt V(\lambda- \theta')(\theta' - \theta')^\T \wt V  \rp g(\lambda)$$ 
	further, observe that, 
\begin{eqnarray*}
	u^T \lp- 2\alpha \wt V + 4\alpha^2 \wt V(\lambda - \theta')(\lambda - \theta')^\T \wt V  \rp u & = & -2 \alpha\|u\|_{\wt V}^2 + 4\alpha^2 (u^\T \wt V (\lambda - \theta'))^2 \\
	&\leq& -2\alpha\| u\|_{\wt V}^2 + 4\alpha^2 \| u\|^2_{\wt V} \|\lambda - \theta'\|_{\wt V}^2,
\end{eqnarray*}
which proves the claimed statement as the Hessian is negative semi-definite for $\alpha \in (0, 1/({2 \max_{\bm \lambda \in \bm D} \| \vv( \bm\theta' - \bm \lambda)\|_{\Sigma^{-1} \otimes aa^\T}^2})]$. 
\end{proof}

The following lemma expresses the difference in estimation of the empirical mean between two consecutive rounds. 
 \begin{lemma}
  \label{lem:est_error}
  	Let $\wt V_t := \Sigma^{-1} \otimes V_t$, then, it holds that 
  	$$\wt V_t \vv{}(\wh{\bm\theta}_{t} - \wh{\bm\theta}_{t-1}) = \lsb \Sigma^{-1} \otimes a_t \rsb \lp \veps_t + \lsb I_d \otimes a_t^\T\rsb \vv{}\lp \bm \theta - \wh{\bm\theta}_{t-1} \rp \rp. $$
  \end{lemma}

  \begin{proof}
By Sherman-Morrison's formula, with $S_t = \sum_{s\in [t]} \vv(a_s X_s^\T)$
 \begin{eqnarray*}
 	 	\vv{(\wh{\bm\theta}_{t})} &=& \lsb I_d \otimes (V_{t-1} + a_t a_t^\T)^{-1} \rsb S_{t} \\
 	&=&  \lsb I_d \otimes \lp V_{t-1}^{-1} - \frac{V_{t-1}^{-1} a_t a_t^\T V_{t-1}^{-1}}{1+ a_t^\T V_{t-1}^{-1} a_t} \rp \rsb S_{t} \\
 	&=& \vv(\wh{\bm\theta}_{t-1}) -  \lsb I_d \otimes \lp \frac{V_{t-1}^{-1} a_t a_t^\T V_{t-1}^{-1}}{1+ a_t^\T V_{t-1}^{-1}a_t} \rp  \rsb S_{t} + [I_d \otimes V_{t-1}^{-1}] \vv(a_t X_t^\T)  \\
 	&=& \vv(\wh{\bm\theta}_{t-1}) - \frac{[ I_d \otimes (V_{t-1}^{-1} a_t a_t^\T)]  [I_d \otimes V_{t-1}^{-1}]}{1+ a_t^\T V_{t-1}^{-1}a_t} S_{t} + [I_d \otimes V_{t-1}^{-1}] \vv(a_t X_t^\T) \\
 	&=& \vv(\wh{\bm\theta}_{t-1}) - \frac{[ I_d \otimes (V_{t-1}^{-1} a_t a_t^\T)]  \vv(\wh{\bm{\theta}}_{t-1})  + [I_d \otimes ( V_{t-1}^{-1} a_t a_t^\T V_{t-1}^{-1})] \vv(a_t X_t^\T)}{1+ a_t^\T V_{t-1}^{-1}a_t} + [I_d \otimes V_{t-1}^{-1}] \vv(a_t X_t^\T) 
 \end{eqnarray*}
we recall that for two vector $a, X$, 
$$ \vv(a X^\T) = [I_d \otimes a] X \quad \text{and} \quad [A\otimes B] [C\otimes D] = (AC) \otimes (BD)$$
 then 
\begin{eqnarray*}
	[I_d \otimes ( V_{t-1}^{-1} a_t a_t^\T V_{t-1}^{-1})] \vv(a_t X_t^\T) &=& [I_d \otimes ( V_{t-1}^{-1} a_t a_t^\T V_{t-1}^{-1})] [I_d \otimes a_t] X_t \\
	&=& (a_t^\T V_{t-1}^{-1} a_t) [ I_d \otimes (V_{t-1}^{-1} a_t)] X_t.
\end{eqnarray*}
Proceeding similarly yields 
\begin{eqnarray*}
	[I_d \otimes V_{t-1}^{-1}] \vv(a_t X_t^\T) &=& [I_d \otimes V_{t-1}^{-1}] [I_d \otimes a_t] X_t \\
	&=& [I_d \otimes V_{t-1}^{-1}a_t] X_t, 
\end{eqnarray*}
therefore, 
 \begin{eqnarray*}
 	 	\vv{(\wh{\bm\theta}_{t})} &=&  \vv{(\wh{\bm\theta}_{t-1})} +   \frac{[I_d \otimes (V_{t-1}^{-1}a_t)]X_t -[I_d \otimes (V_{t-1}^{-1} a_t a_t^\T)]  \vv{(\wh{\bm\theta}_{t-1})} }{1+ a_t^\T V_{t-1}^{-1}a_t}. 
 \end{eqnarray*}
By Kronecker product on the left, we have 
 \begin{eqnarray*}
 	[\Sigma^{-1} \otimes V_{t}]\vv(\wh{\bm\theta}_{t} - \wh{\bm\theta}_{t-1}) &=&  \frac{[\Sigma^{-1} \otimes V_{t}][I_d \otimes V_{t-1}^{-1}a_t]X_t - [\Sigma^{-1} \otimes V_{t}][I_d \otimes (V_{t-1}^{-1} a_t a_t^\T)]  \vv{(\wh{\bm\theta}_{t-1})} }{1+ a_t^\T V_{t-1}^{-1}a_t} \\
 	&=& 
 	\frac{[\Sigma^{-1} \otimes (V_tV_{t-1}^{-1} a_t)]X_t - [\Sigma^{-1} \otimes (V_t V_{t-1}^{-1}a_ta_t^\T)]  \vv{(\wh{\bm\theta}_{t-1})} }{1+ a_t^\T V_{t-1}^{-1}a_t}. 
 \end{eqnarray*}
Then observe that 
 \begin{eqnarray*}
 [\Sigma^{-1} \otimes (V_tV_{t-1}^{-1} a_t)]X_t &=& [\Sigma^{-1} \otimes ((I_h + a_t a_t^\T V_{t-1}^{-1})a_t)]X_t \\
 &=& \lp [\Sigma^{-1} \otimes a_t] +  (a_t^\T V_{t-1}^{-1}a_t) [\Sigma^{-1}\otimes a_t]  \rp X_t  \\
 &=& (1 + a_t^\T V_{t-1}^{-1}a_t) [\Sigma^{-1} \otimes a_t] X_t,  
 \end{eqnarray*}
 and similar developments yields 
 \begin{eqnarray*}
 	[\Sigma^{-1} \otimes (V_t V_{t-1}^{-1}a_ta_t^\T)]  \vv{(\wh{\bm\theta}_{t-1})} &=& [\Sigma^{-1} \otimes ((I_h + a_t a_t^\T V_{t-1}^{-1})a_ta_t^\T) ] \vv{(\wh{\bm\theta}_{t-1})}\\
 	&=& \lp [\Sigma^{-1} \otimes (a_t a_t^\T)] + (a_t^\T V_{t-1}^{-1}a_t) [\Sigma^{-1} \otimes (a_t a_t^\T)] \rp \vv{(\wh{\bm\theta}_{t-1})} \\
 	&=& (1 + a_t^\T V_{t-1}^{-1}a_t) [\Sigma^{-1} \otimes (a_t a_t^\T)] \vv{(\wh{\bm\theta}_{t-1})} 
 \end{eqnarray*}
 so that we finally have 
 \begin{eqnarray*}
 	\wt V_t \vv(\wh{\bm\theta}_{t} - \wh{\bm\theta}_{t-1}) &=& 
 	 [\Sigma^{-1} \otimes V_{t}] \vv(\wh{\bm\theta}_{t} - \wh{\bm\theta}_{t-1}) \\ &=& [\Sigma^{-1} \otimes a_t] X_t - [\Sigma^{-1} \otimes (a_t a_t^\T)] \vv{(\wh{\bm\theta}_{t-1})} \\
 	&=& [\Sigma^{-1} \otimes a_t] X_t - \lp [\Sigma^{-1} \otimes a_t][I_d \otimes a_t^\T]\rp \vv{(\wh{\bm\theta}_{t-1})} \\
 	&=& [\Sigma^{-1} \otimes a_t] (X_t - (I_d \otimes a_t^\T) \vv{(\wh{\bm\theta}_{t-1})}) . 
  \end{eqnarray*}
  We recall that we have $X_t := (I_d \otimes a_t^\T) \vv(\bm\theta) + \veps_t $ from the linear model.  Thus, all put together, we recover the claimed result. 
 \end{proof}

  \infBounded*
 This result shows that the best response always exists in the unstructured setting as an $\inf$ over a compact subset of the alternative. Although its closed form is unknown in PSI, we show that it belongs to a ball centered at $\bm \mu'$ and whose radius depends also on $\bm \mu'$.  
  
 \begin{proof}
 The idea of the proof is to show that when there is a ball $\bB$ such that if an alternative parameter of $\bm \mu'$ does not belong to $\bB$, then there is a parameter in $\alt(S^\star(\bm \mu')) \cap \bB$ for which the transportation cost will be smaller. Let $\bm \mu' := (\mu'_1 \dots \mu'_K)^\T$, where $\mu'_i$ denotes the vector mean of arm $i$ and we let $\bm \lambda := (\lambda_1 \dots \lambda_K)^\T$ where $\lambda_i \in \bR^d$. To ease notation, we let $S = S^\star(\bm\mu')$.  Introducing  \begin{equation}
\label{eq:alt_decomp}
	\alt^-(S) := \bigcup_{i, j\in S^2: i\neq j}  W_{i,j}  \quad \text{and} \quad \alt^+(S) := \bigcup_{i \in S^c} V_i , 
\end{equation}
where $\Theta := \bR^{K \times d}$, we define  
$$ W_{i, j} := \lb \bm\lambda := (\lambda_1 \dots \lambda_K)^\T \in \Theta \mid  \lambda_i \leq \lambda_j \rb  \;\text{and} \; V_i :=  \lb \bm \lambda := (\lambda_1 \dots \lambda_K)^\T \in \Theta \mid  \exists\; i \in (S)^c: \forall\; j\in S, \lambda_i \nprec  \lambda_j \rb.$$ By simply expanding the expression, we have in the unstructured case,  
$$ D(\w, \bm \lambda; \bm \mu' ) := \| \vv(\bm \lambda - \bm \mu')\|^2_{\Sigma^{-1} \otimes \diag(\w)} = \sum_{i=1}^K \omega_i \| \lambda_i - \mu'_i\|^2_{\Sigma^{-1}}.$$
Let $(i,j)\in S^2$ be fixed and $\bm\lambda \in W_{i,j}$. Let $\alpha_{i,j} = \| \mu'_i - \mu'_j\|_{\Sigma^{-1}}$. If 
$\| \lambda_i - \mu'_i\|_{\Sigma^{-1}} > \alpha_{i,j}$ then define the instance $\bm{\wt \lambda}$ as 
$$ \wt \lambda_k := \begin{cases}
	\mu'_k & \text{if } k \notin \{ i, j\},\\
	\mu'_j & \text{else }, \\
\end{cases} $$
which satisfies $\wt \lambda_i \leq \wt \lambda_j$, so $\bm{\wt\lambda} \in W_{i,j}$, and  
\begin{eqnarray*}
	D(\w, \bm{\wt\lambda}; \bm \mu' ) &=& \omega_i\| \mu'_i - \mu'_j\|^2_{\Sigma^{-1}} + \sum_{k \notin \{i,j\}} \omega_k\| \lambda_k - \mu'_k\|^2_{\Sigma^{-1}},\\
	&=& \omega_i\| \mu'_i - \mu'_j\|^2_{\Sigma^{-1}} < \omega_i\| \lambda_i - \mu'_i\|^2_{\Sigma^{-1}}\\
	&<& D(\w, \bm \lambda, \bm \mu'),
\end{eqnarray*}
and further observe that $\max_{k \in [K]} \| \wt \lambda_k - \mu'_k \|_{\Sigma^{-1}}\leq \alpha_{i,j}$. We proceed similarly if $\| \lambda_j - \mu'_j\|_{\Sigma^{-1}} > \alpha_{i,j}$, by defining 
$$ \wt \lambda_k := \begin{cases}
	\mu'_k & \text{if } k \notin \{ i, j\},\\
	\mu'_i & \text{else }, \\
\end{cases} $$
and the same conclusion follows. So we have proved that there exists $\bm{\wt\lambda} \in W_{i,j}$ with $\max_{k \in [K]} \| \wt \lambda_k - \mu'_k \|_{\Sigma^{-1}}\leq \alpha_{i,j}$
and for which the transportation cost is not larger than that of $\bm \lambda$. We prove a similar property for $\alt^+(S)$. 

Now, fix $i \notin S$ and take $\bm \lambda \in V_i$. Let $b_i = \max_{k \in S} \| \mu'_i - \mu'_k\|_{\Sigma^{-1}}$. If $\| \lambda_i - \mu'_i\|_{\Sigma^{-1}} > b_i$ then it suffices to define 
$\bm{\wt\lambda}$ as 
$$\wt \lambda_p := \begin{cases}
	\mu'_p & \text{if } p \neq i,\\
	\mu'_{\wt i}  &\text{else },\end{cases}$$
for $\wt i \in S$, to ensure that $\bm{\wt\lambda} \in V_i$ and
\begin{eqnarray*}
D(\w, \bm{\wt\lambda}; \bm \mu' ) &=& \omega_i \| \mu'_{\wt i} - \mu'_i\|_{\Sigma^{-1}}^2 + \sum_{k \neq i }\omega_k\| \lambda_k - \mu'_k\|_{\Sigma^{-1}}^2 \\
&\leq& \omega_i b_i^2  \\&<&  \omega_i \| \lambda_i - \mu'_i\|_{\Sigma^{-1}}^2 \\
&<& D(\w, \bm{\lambda}; \bm \mu'),
\end{eqnarray*}
where the second line also uses the fact that $\wt i \in S$. Assume $\| \lambda_i - \mu'_i\|_{\Sigma^{-1}} < b_i$ and that $H_i := \{ k \in S : \| \lambda_k - \mu'_k\|_{\Sigma^{-1}} > 2b_i\}$ is non-empty. Let us define the instance $\bm{\wt\lambda}:$
$$\wt \lambda_k := \begin{cases}
	\lambda_i & \text{if } k \in H_i,\\
	\lambda_k & \text{if } k \in  (S\cup \{i\}) \backslash H_i,\\
	\mu'_k & \text{else.}
\end{cases}$$
Since $\bm \lambda \in V_i$, $\bm{\wt\lambda}$ as defined above satisfies $\wt \lambda_i \nprec \wt \lambda_k, k \in S$, that is $\bm{\wt\lambda} \in V_i$ and  
\begin{eqnarray*}
	D(\w, \bm{\wt\lambda}; \bm \mu' )  &= &\sum_{k \in (S\cup \{i\})} \omega_k \| \wt \lambda_k - \mu'_k\|_{\Sigma^{-1}}^2 \\
	&=& \sum_{k \in (S\cup \{i\}) \backslash H_i} \omega_k \| \lambda_k - \mu_k\|_{\Sigma^{-1}}^2 + \sum_{k \in H_i} \omega_k \| \lambda_i - \mu'_k\|_{\Sigma^{-1}}^2\\
	&<& \sum_{k \in (S\cup \{i\}) \backslash H_i} \omega_k \| \lambda_k - \mu'_k\|_{\Sigma^{-1}}^2 + \sum_{k \in H_i} \omega_k 4 b_i^2,
\end{eqnarray*}
which follows since $\| \lambda_i - \mu'_k\|_{\Sigma^{-1}} \leq \| \lambda_i - \mu'_i\|_{\Sigma^{-1}} + \| \mu'_i - \mu'_k\|_{\Sigma^{-1}} \leq 2b_i$. Recalling that for $k \in H_i$, 
$$ 4b_i < \| \lambda_k - \mu'_k\|_{\Sigma^{-1}}^2,$$it follows that 
$$ D(\w, \bm{\wt\lambda}; \bm \mu' ) < D(\w, \bm{\lambda}; \bm \mu' ).$$
Put together, we have proved that for all $\bm \lambda \in V_i$, there exists $\bm{\wt\lambda} \in V_i$ whose transportation cost is not larger than that of $\bm \lambda$ and which additionally satisfies: $\max_{k \in [K]}\| \wt \lambda_k - \mu'_k \|_{\Sigma^{-1}} \leq 2b_i$. 

To conclude, let us define $\epsilon := \max(2\max_{i\notin S} b_i, \max_{i,j \in S^2} \alpha_{i, j})$. Using what precedes, we have proved that 
$$ \inf_{\bm \lambda \in \alt(S) \cap \{ \bm \lambda : \max_{k}\| \lambda_k - \mu'_k\|_{\Sigma^{-1}}<\epsilon\}} D(\w, \bm{\lambda}; \bm \mu') \leq \inf_{\bm \lambda \in \alt(S)} D(\bm \w, \bm{\lambda}; \bm \mu'),$$
which proves the claimed result as we have $\alt(S) \cap \{ \bm \lambda : \max_{k}\| \lambda_k - \mu'_k\|_{\Sigma^{-1}}<\epsilon\} \subset \alt(S)$. 
\end{proof}


\section{IMPLEMENTATION DETAILS AND ADDITIONAL EXPERIMENTS}
\label{app:add_experiments}

\newcommand{\ccZ}[0]{\lvert \cZ \rvert}
After presenting the implementations details and computational cost in Appendix~\ref{app:ssec_implementation}, we display supplementary experiments in Appendix~\ref{app:ssec_supplementary}. The datesets are described in Appendix~\ref{app:ssec_datasets}. 

\subsection{Implementation Details and Complexity}
\label{app:ssec_implementation}
\subsubsection{Setup}
The algorithm is implemented mainly in standard \texttt{C++17}. Our complete code will be made available on an anonymous repository at \url{https://anonymous.4open.science/r/psips-17EF/README.md} with instructions to compile and run. We run our algorithms mainly on a personal laptop with 8GB RAM, 256GB SSD and an Apple M2 CPU. The code is compiled with the \texttt{GCC13}. In the experiments, we include the ``Oracle'' algorithm which compute the optimal weights of the underlying bandit instance. At each round, the oracle pulls an arm according to this optimal weights vector. Another algorithm included in our benchmark is the APE algorithm of \cite{kone2023adaptive} which is a LUCB-type algorithm for PSI. \elise{} denotes the tracking algorithm of \cite{crepon24a} for PSI. 

In the following paragraphs, we analyze the time and memory complexity of the different parts of our algorithms and we discuss its total time complexity. 

\subsubsection{Time and memory complexity}

\paragraph{Compute Pareto Set}
When $d=2$, it is known that the Pareto set can be computed in an average $\cO(\ccZ\log(\ccZ))$ time complexity and a worst-case $\cO(\ccZ\log(\ccZ) + \ccZ\lvert p \rvert)$ where $p$ is the size of the Pareto set. For a general dimension $d>3$, the algorithm of \cite{kung_ps_algo}  achieves a time complexity of $\cO(\ccZ\log(\ccZ)^{d-2})$. Their algorithm is based on divide and conquer and in the case $d=2$, it consists in sorting the arms along one coordinate, then by traversing the sorted items in decreasing order (for sorted dimension) each element is added to the set $S$ if it is note dominated by an element of $S$. The resulting set is then the Pareto set, and combining the initial sorting and the construction of the Pareto set, the worst case complexity will be $\cO(\log(\ccZ) + \ccZ p + p(1-p)/2)$, $p$ is the size of the Pareto set. The space complexity is $O(\ccZ)$.  This algorithm is described below. 

To check if two parameters have the same Pareto set, in the worst case, one can compute their Pareto sets, using the algorithms described earlier. However, it is possible (and more efficient) to check this condition without computing the Pareto sets of the parameters by simply checking from the definition of $\alt$.  We propose the  following algorithm to check whether a parameter belongs to an alternative. 

\paragraph{Update $\bm{\widehat{\theta}}_{t+1}$} To compute  $\bm{\widehat{\theta}}_{t+1}$ in the linear setting, we used 
Sherman-Morrison formula to avoid computing the inverse of $V_t$. Indeed as $V_t = V_{t-1} + a_t a_t^\T$ we have  by Sherman-Morrison formula $$V_t^{-1} = V_{t-1}^{-1} -  \frac{V_{t-1}^{-1}a_t ( V_{t-1}^{-1} a_t)^\T}{ 1+ \| a_t \|^2_{V_{t-1}^{-1}}}, $$ which can be computed in time $\cO(h^2)$. The computation of $\bm Z_{t+1}$ is done in time $\cO(h d)$, so that the total cost of computing the least-squares estimate 
is dominated by the matrix product $V_{t+1}^{-1} \bm Z_{t+1}$, which can be done naively in $\cO(h^2d)$. 

\paragraph{Generate Random Samples} To pull from $\Pi_t$ we generate samples from $\cN(0_{dh}, I_{dh})$ (whose average cost is  $\cO(dh)$) and we compute a Cholesky decomposition of the $\Sigma \otimes V_{t}^{-1}$ which we do by computing the Cholesky of $\Sigma$ and that of $V_{t}^{-1}$.  Note that in the unstructured setting, the Cholesky of $V_{t}^{-1}$ is trivial to compute. In summary, in the worst case, the total time complexity of this step is 
$\cO(\max(d^2 h^2, h^3))$ in the linear setting and $\cO(d^2h)$ in the unstructured setting.  

\paragraph{Miscellaneous} 
For the other parts of the algorithm, the cost of a step of AdaHedge~\citep{rooij_ada_hedge} is $\cO(K)$, same for the memory complexity. When $\Theta$ is defined as ball, the cost of the verification $\bm \lambda \in \Theta$ is simply the cost of for computing the norm of $\Theta$. For Euclidean norms, it will be $\cO(dh)$. The computation of $M(t, \delta)$ is $\cO(1)$, as we only need to compute $(q(\delta, i)_{i \in [K]})$ once at initialization. 

\paragraph{Summary} Combining the displays above, the time complexity of our algorithm at each step is  $$ \cO(\ccZ\log(\ccZ)^{\max (1, d-2)}) + \cO(\max(d^2 h^2, h^3)).$$
In the unstructured setting, it can be reduced to 
$$\cO(\max(K^2d, K\log(K)^{\max (1, d-2)})).$$
\subsection{Datasets}
\label{app:ssec_datasets}
\paragraph{COV-BOOST} 
This dataset has been publicly released in \cite{munro_safety_2021} based on a phase 2 booster trial for Covid19 vaccine. 
\cite{kone2023adaptive} further processed the data and extracted the parameters of a bandit instance with 20 arms and $d=3$. 
We use the values tabulated by the authors to simulate our algorithms, see Tables~\ref{tab:arith_mean} and~\ref{tab:pooled_var}.

\begin{table}[H]
\caption{Table of the empirical arithmetic mean of the log-transformed immune response for three immunogenicity indicators. Each acronym corresponds to a vaccine. There are two groups of arms corresponding to the first 2 doses: one with prime BNT/BNT (BNT as first and second dose) and the second with prime ChAd/ChAd (ChAd as first and second dose). Each row in the table gives the values of the 3 immune responses for an arm (i.e. a combination of three doses).}
\begin{center}
\begin{tabular}{r r c c c}
\toprule
\multirow{2}{*}{Dose 1/Dose 2} & \multirow{2}{*}{Dose 3 (booster)}  &\multicolumn{3}{c}{Immune response}\\

 &  & Anti-spike IgG &  $\text{NT}_{50}$ &  cellular response\\
 \midrule 
    \multirow{10}{*}{Prime BNT/BNT}&ChAd     &                      9.50 &                                       6.86 &                                 4.56 \\
&NVX      &                      9.29 &                                       6.64 &                                 4.04 \\
&NVX Half &                      9.05 &                                       6.41 &                                 3.56 \\
&BNT      &                     10.21 &                                       7.49 &                                 4.43 \\
&BNT Half &                     10.05 &                                       7.20 &                                 4.36 \\
&VLA      &                      8.34 &                                       5.67 &                                 3.51 \\
&VLA Half &                      8.22 &                                       5.46 &                                 3.64 \\
&Ad26     &                      9.75 &                                       7.27 &                                 4.71 \\
&m1273    &                     10.43 &                                       7.61 &                                 4.72 \\
&CVn      &                      8.94 &                                       6.19 &                                 3.84 \\
\midrule
 \multirow{10}{*}{Prime ChAd/ChAd}&ChAd     &                      7.81 &                                       5.26 &                                 3.97 \\
&NVX      &                      8.85 &                                       6.59 &                                 4.73 \\
&NVX Half &                      8.44 &                                       6.15 &                                 4.59 \\
&BNT      &                      9.93 &                                       7.39 &                                 4.75 \\
&BNT Half &                      8.71 &                                       7.20 &                                 4.91 \\
&VLA      &                      7.51 &                                       5.31 &                                 3.96 \\
&VLA Half &                      7.27 &                                       4.99 &                                 4.02 \\
&Ad26     &                      8.62 &                                       6.33 &                                 4.66 \\
&m1273    &                     10.35 &                                       7.77 &                                 5.00 \\
&CVn      &                      8.29 &                                       5.92 &                                 3.87 \\
\bottomrule
\end{tabular}
\end{center}
\label{tab:arith_mean}
\end{table}

\begin{table}[ht]
\caption{Pooled variance of each group.}
\begin{center}
\begin{tabular}{r c c c}
\toprule 
\multirow{2}{*}{} & \multicolumn{3}{c}{Immune response}\\
 & Anti-spike IgG &  $\text{NT}_{50}$ &  cellular response\\
 \midrule 
    Pooled sample variance&  0.70&0.83& 1.54\\
    \bottomrule 
\end{tabular}
\end{center}
\label{tab:pooled_var}
\end{table}

\paragraph{Network on Chip} We use the dataset studied by \cite{almer} and publicly released by \cite{zuluaga_active_2013} and available at \url{http://www.spiral.net/software/pal.html}. We apply further preprocessing, by normalizing the features to extract a linear instance  with parameter $\bm \theta^*$ given below. 
$$\bm \theta^* =  \begin{pmatrix}  -3.08665453 & -3.35487744\\  -3.66027623 & \phantom{-}0.19333635\\  -2.68963781 & -1.39779755\\  -7.90670356 & -4.44360318\\\end{pmatrix}.$$ In Figure~\ref{fig:noc} we plot the resulting PSI instance.

\begin{figure}
    \centering
    \includegraphics[width=0.45\linewidth]{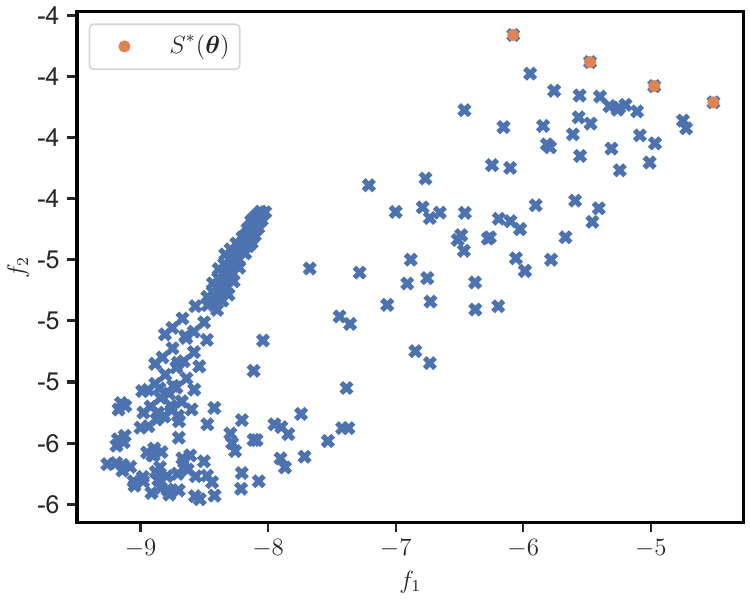}
    \caption{Means of each arm in the NoC instance with normalized features.}
\label{fig:noc}
\end{figure}

\paragraph{Reproducibility} We provide below the parameters used in the experiments reported in the main paper.  As explained in the experiment section, since their calibration relies on overly conservative union bounds, we used $M(t, \delta) = \frac1\delta \log(t/\delta)$ and $c(t, \delta) = 1 + \frac{\log(\log(t))}{\log(1/\delta)}$ for which we noticed a negligible empirical error for $\delta \leq 0.1$ in our experiments. For AdaHedge, we used the implementation of \cite{rooij_ada_hedge} translating the pseudo-code provided in their paper.  In the supplementary material, we provide our \texttt{C++} implementation together with the scripts to un each experiment. The complete scripts to reproduce our experiments will be available on an anonymous repository with the instructions to compile and run the algorithms.

\subsection{Supplementary Experiments}
\label{app:ssec_supplementary}

We provide additional experimental results. 
In particular, we conduct experiments to evaluate the performance of our algorithm in more complex scenarios. 
These experiments include instances in higher dimension $d$ or with larger number of arms $K$. 
Moreover, we also report the probability of error in the covid19 experiment.
Each experiment is repeated on $100$ independent runs and the quantity reported are averaged statistics. 
Due to its high runtime, we omit \elise{} of~\cite{crepon24a} from these experiments.

\begin{figure}[H]
  \centering
  \includegraphics[width=0.23\linewidth]{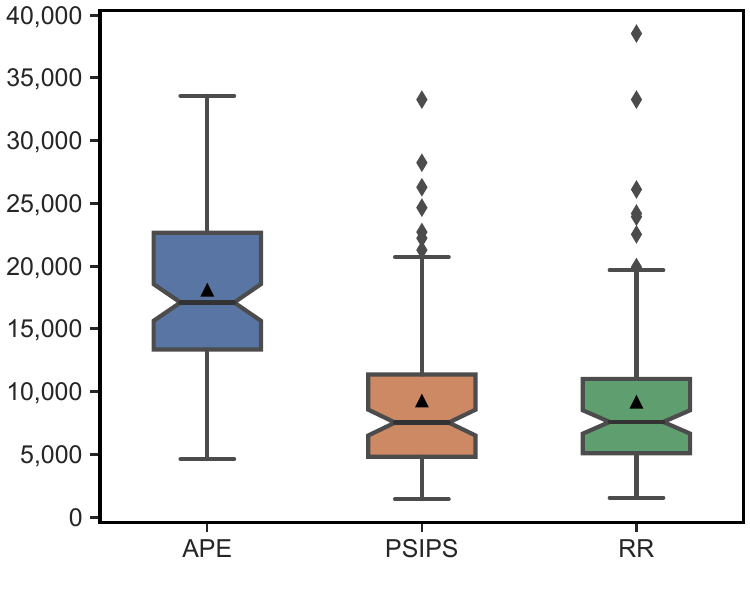}
  \includegraphics[width=0.23\linewidth]{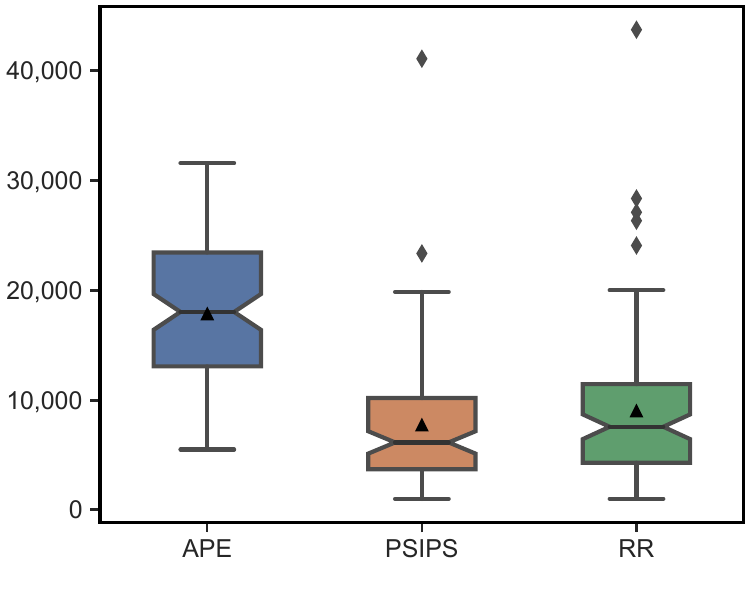}
  \includegraphics[width=0.23\linewidth]{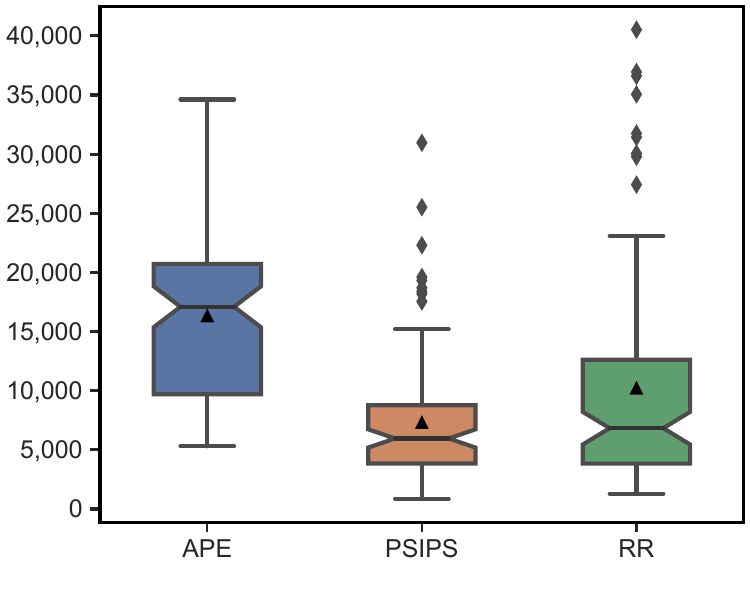}
 \includegraphics[width=0.23\linewidth]{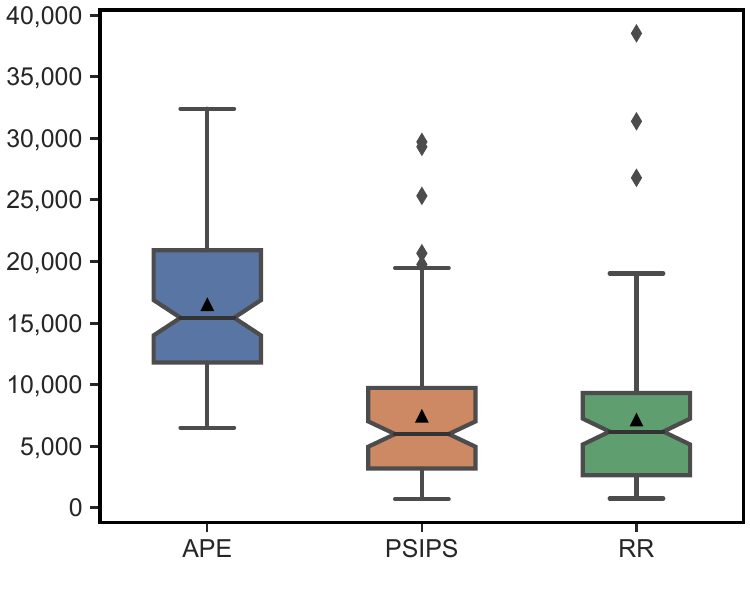}
  \captionof{figure}{Empirical stopping time on $10$-armed random Gaussian instances in dimension $d \in \{3, 4, 5, 6\}$ (left to right).}
  \label{fig:abayes}
\end{figure} 

\paragraph{Higher Dimension $d$} 
In this experiment, we evaluate our algorithm in the unstructured setting on $100$ random Gaussian instances with higher dimension $d \in \{3, 4, 5, 6\}$, $K=10$ and $\Sigma = I_d / 2$. 
Each arm is drawn  from $\cU([-1, 1]^d)$ and to have reasonable runtime, we reject instances whose complexity $H(\bm \theta)$ is larger than $500$. 
The average Pareto set size was respectively $4.42$, $6.78$, $8.57$, $9.21$ in dimensions $d \in \{3, 4, 5, 6\}$.  
Using $\delta = 0.01$, we observe a negligle empirical error. 

In Figure~\ref{fig:abayes}, we see that \hyperlink{PSIPS}{PSIPS} has competitive empirical performance for higher dimension $d$.
It consistently outperforms APE and performs on par with uniform sampling. 
The good performance of uniform sampling in this experiment can be attributed to the large size of the Pareto set, which contains most of the arms.
We conjecture that $\omega^\star(\bm \theta)$ is close to the uniform allocation on those instances.

\begin{figure}[H]
  \centering
  \includegraphics[width=0.23\linewidth]{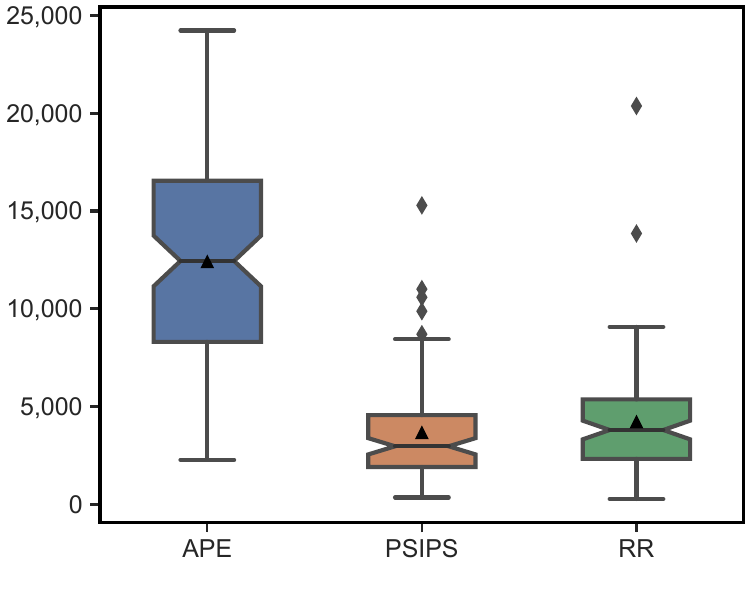}
  \includegraphics[width=0.23\linewidth]{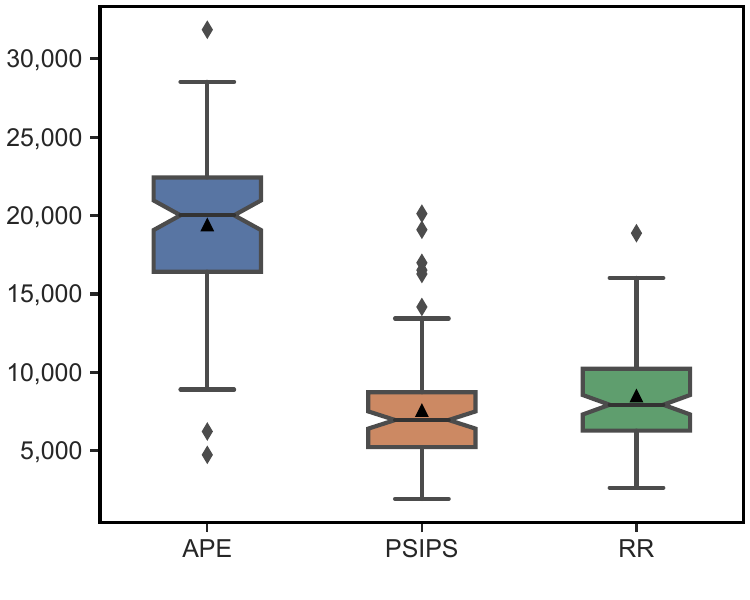}
  \includegraphics[width=0.23\linewidth]{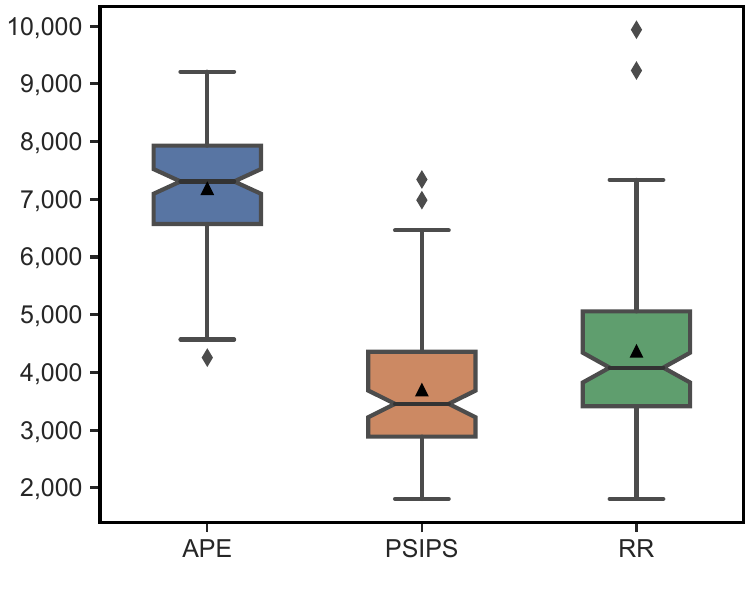}
 \includegraphics[width=0.23\linewidth]{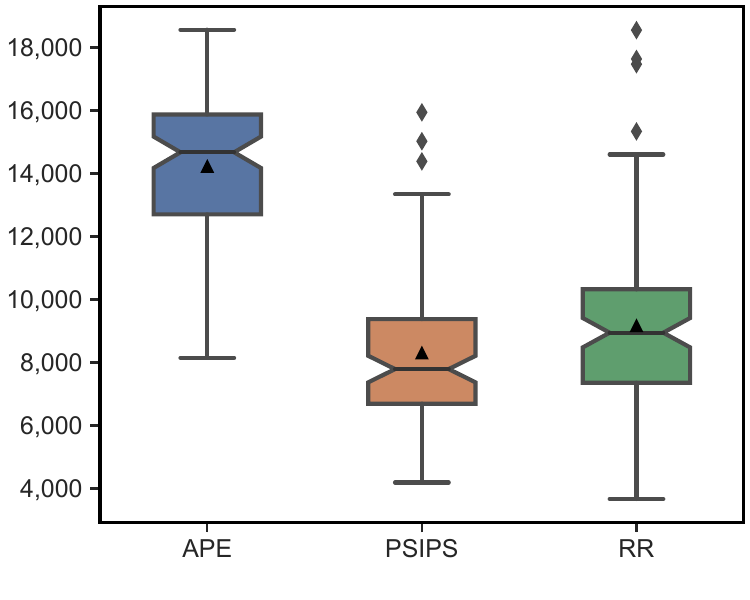}
  \captionof{figure}{Empirical stopping time on $2$-dimensional random Bernoulli instances with $K \in \{10, 15, 25, 40\}$ arms (left to right).}
  \label{fig:alargeK}
\end{figure}

\paragraph{Large Number of Arms $K$} 
We benchmark the algorithms on instances with a larger number of arms. 
We report the results on $100$ random Bernoulli instances with $K \in \{10,15,25,40\}$ arms in dimension $d=2$. 
The means of the instances are drawn from $\cU([0.2, 0.9]^{K \times d})$ and we use $\Sigma = I_2/4$, and we reject instance whose complexity $H(\bm \theta)$ is larger than $500$. 
The average Pareto set size was respectively $1.73$, $1.64$, $1.37$, $1.28$ for the number of arms $K \in \{10,15,25,40\}$.
Using $\delta = 0.01$, we observe a negligle empirical error.

Figure~\ref{fig:alargeK} shows that \hyperlink{PSIPS}{PSIPS} has competitive empirical performance for higher dimension $d$.
It significantly outperforms APE and slightly outperforms uniform sampling. 
The worsening of uniform sampling's performance stems from the smaller size of the Pareto set, which almost contains only one arm.

\begin{figure}
    \centering
    \includegraphics[width=0.5\linewidth]{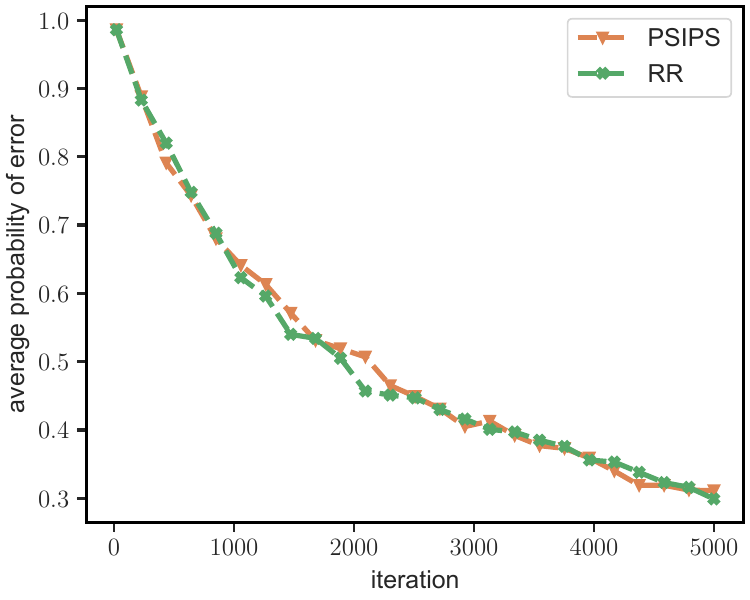}
    \caption{Average probability of mis-identification of the Pareto set in the covid19 experiment.}
\label{fig:PoE}
\end{figure}

\paragraph{Probability of Error} 
Both the recommendation rule and the sampling rule of \hyperlink{PSIPS}{PSIPS} are $\delta$-independent and defined as any time $t$.
Therefore, by disabling the stopping rule, it is possible to run \hyperlink{PSIPS}{PSIPS} for $T= 5000$ time steps and report the empirical error $\indi{\wh S_t \ne S^\star}$ at any time $t \in [5000]$ and averaged over $1000$ runs.
Uniform sampling enjoy the same anytime property and has the same recommendation rule. Therefore, we use it as benchmark.

Figure~\ref{fig:PoE} reveals that the sampling rule from \hyperlink{PSIPS}{PSIPS} yields an empirical error which is close to the one achieved by uniform sampling. Even though \hyperlink{PSIPS}{PSIPS} has only theoretical guarantees in the fixed-confidence setting, it performs empirically well in the anytime setting.

\end{appendix}

\end{document}